\pdfoutput=1
\documentclass[11pt]{article}
\usepackage[letterpaper, left=1.2truein, right=1.2truein, top = 1.2truein,bottom = 1.2truein]{geometry}
\usepackage{filecontents}
\RequirePackage{amsmath, amsfonts, amssymb, amsthm, bm, graphicx, mathtools, enumerate,multirow}
\RequirePackage[numbers, sort]{natbib}
\RequirePackage[colorlinks,citecolor=blue,urlcolor=blue]{hyperref}
\RequirePackage{graphicx}%
\usepackage{mathtools}
\usepackage{enumitem}
\usepackage[ruled, vlined, lined, commentsnumbered,linesnumbered]{algorithm2e}
\usepackage{bbm}
\usepackage{xr}
\usepackage{natbib}
\usepackage{titling}

\newtheorem{theorem}{Theorem}
\newtheorem{lemma}{Lemma}

\newtheorem{corollary}{Corollary}
\theoremstyle{remark}

\newtheorem{assumption}{Assumption}

\newcommand{\supp}[1]{{Section #1 of the Supplementary Material}}
\newcommand{\Rnum}[1]{\uppercase\expandafter{\romannumeral #1\relax}}

\newcommand{\proposed}[1]{{\textsc{ProjBalance}#1}}
\newcommand{\vl}[1]{{\textsc{VL}#1}}
\newcommand{\ql}[1]{{\textsc{SAVE}#1}}
\newcommand{\revise}[1]{{\textcolor{black}{#1}}}

\DeclareMathOperator*{\argmin}{argmin}
\DeclarePairedDelimiter\ceil{\lceil}{\rceil}
\DeclarePairedDelimiter\floor{\lfloor}{\rfloor}

\def\pr{\mbox{Pr}} %

\def\G{\mathbb{G}}

\newcommand{\tp}{\intercal}

\newcommand{\bigO}{\ensuremath{\mathop{}\mathopen{}\mathcal{O}\mathopen{}}}
\newcommand{\smallO}{ \scalebox{0.7}{$\mathcal{O}$}}
\newcommand{\bigOp}{\bigO_\mathrm{p}}
\newcommand{\smallOp}{\smallO_\mathrm{p}}

\newcommand{\Ascr}{{\mathcal{A}}}
\newcommand{\Hscr}{{\mathcal{H}}}
\newcommand{\Gscr}{{\mathcal{G}}}

\newcommand{\Fscr}{{\mathcal{F}}}

\newcommand{\kernel}{{\kappa}}

\newcommand{\blam}{\bm \lambda}

\newcommand{\Total}{\textcolor{black}{nT}}
\newcommand{\total}{\textcolor{black}{n,T}}
\newcommand{\fltwo}{L_2}
\newcommand{\vltwo}{\ell_2}

\newcounter{ucnt}
\newcommand{\newu}{
	\refstepcounter{ucnt}
	\ensuremath{C_{\theucnt}}
}
\newcommand{\oldu}[1]{\ensuremath{C_{\ref*{#1}}}}

\newcounter{lcnt}
\newcommand{\newl}{
	\refstepcounter{lcnt}
	\ensuremath{c_{\thelcnt}}
}
\newcommand{\oldl}[1]{\ensuremath{c_{\ref*{#1}}}}

\makeatletter

\newcommand{\ltxlabel}[1]{\ltx@label{#1}}

\makeatother

\newcommand{\epc}{\hspace{1pc}}
\def\g{g_*^\pi}

\def\gn{\hat g^\pi}

\def\EE{\mathbb{E}}

\def\G{\boldsymbol{G}}
\def\F{\mathcal{F}}

\def\R{\mathbb{R}}

\def\PN{\mathbb{P}_N}

\def\leqconst{\lesssim}
\def\geqconst{\gtrsim}

\newcommand{\pushright}[1]{\ifmeasuring@#1\else\omit\hfill$\displaystyle#1$\fi\ignorespaces}

\newcommand{\norm}[1]{\|#1\|}

\newcommand{\abs}[1]{|#1|}
\newcommand{\dkh}[1]{\{#1\}}
\newcommand{\bigdkh}[1]{\big\{#1 \big\}}
\newcommand{\Bigdkh}[1]{\Big\{#1\Big\}}

\def\EE{\mathbb{E}}

\def\R{\mathbb{R}}

\def\argmin{\operatorname{argmin}}

\newcommand{\samfixed}[1]{}

\def\Q{\mathcal{Q}}
\def\given{\, | \,}

\renewcommand{\G}{\mathcal{G}}

\def\argmin{\text{argmin}}

\def\ds1{{\mathrm{1 \hspace{-2.6pt} I}}}

\def\calA{{\cal A}}
\def\calB{{\cal B}}

\def\calD{{\cal D}}

\def\calG{{\cal G}}

\def\calM{{\cal M}}

\def\calQ{{\cal Q}}

\def\calS{{\cal S}}

\def\calV{{\cal V}}

\def\floor#1{\lfloor #1 \rfloor}

\begin{document}

\title{Projected State-action Balancing Weights for Offline Reinforcement Learning}
\author{Jiayi Wang$^1$ \and Zhengling Qi$^2$ \and Raymond K.W. Wong$^1$  }
\date{%
    $^1$Texas A\&M University\\%
    $^2$George Washington University\\[2ex]%
}

\maketitle

\begin{abstract}
	Offline policy evaluation (OPE) is considered a fundamental and challenging problem in reinforcement learning (RL). This paper focuses on the value estimation of a target policy based on pre-collected data generated from a possibly different policy, under the framework of infinite-horizon Markov decision processes. Motivated by the recently developed marginal importance sampling method in RL and the covariate balancing idea in causal inference, we propose a novel estimator with approximately projected state-action balancing weights for the policy value estimation. We obtain the convergence rate of these weights, and show that the proposed value estimator is semi-parametric efficient under technical conditions. In terms of asymptotics, our results scale with both the number of trajectories and the number of decision points at each trajectory. As such, consistency can still be achieved with a limited number of subjects when the number of decision points diverges. In addition,
	\revise{we develop a necessary and sufficient condition for establishing the well-posedness of Bellman operator in the off-policy setting, which characterizes the difficulty of OPE and may be of independent interest.}
	Numerical experiments demonstrate the promising performance of our proposed estimator.
\end{abstract}

\noindent%

\section{Introduction}

In reinforcement learning (RL), off-policy evaluation (OPE) refers to the problem of estimating some notion of rewards (e.g., the value defined in \eqref{def: integrated value fun}) of a target policy based on historical data collected from a potentially different policy.
There is a recent surge of interest in OPE among the statistics and machine learning communities.
On the one hand, OPE serves as the foundation for many RL methods. On the other hand, OPE is of great importance in some high-stakes domains where deploying a new policy can be very costly or risky, such as in medical applications.

This work is partly motivated by mobile health (mHealth) studies. Due to the rapid development of mobile devices and sensing technology, mHealth studies have recently emerged as a promising way to promote the healthy behaviors of patients \cite{liao2018just}.
Mobile devices are used to monitor their health conditions in real-time and deliver just-in-time interventions to individuals \citep{nahum2016just}.
With access to a rich amount of longitudinal data pre-collected by wearable devices, researchers are often interested in evaluating a policy (intervention), potentially different from the one behind the data collection process.
For example, the OhioT1DM dataset \citep{marcolino2018impact} was collected to improve the health and wellbeing of people with type 1 diabetes.
The data consist of 8 weeks' health information of six subjects, based on some unknown policy. For each subject, real-time information such as insulin doses, glucose levels, and self-reported times of meals and exercises was collected. It is often of great interest to evaluate the efficacy of insulin pump therapy.
One notable challenge of mHealth data
is a limited number of subjects, combined with a usually large number of decision points for each subject.
For instance, the OhioT1DM dataset has six subjects with a few thousand decision points per subject.
In statistics, there is a rich literature on estimating the treatment effect or optimal policy from complex longitudinal data \citep{robins1994estimation,robins2000marginal,murphy2001marginal,murphy2003optimal,laber2014dynamic,zhao2015new,wang2018quantile,shi2018high,kosorok2019precision}. However, these methods are mainly designed for studies with very few decision points and often require the number of subjects to grow in order to achieve accurate estimation.

To address the above challenge in mHealth studies, we adopt the framework of the Markov decision process (MDP) \citep[][]{puterman1994markov} and consider an infinite-horizon time-homogeneous setting. This framework is particularly suitable for mHealth studies, and its efficacy has been well demonstrated in the recent literature \citep{luckett2019estimating,liao2020off, shi2020statistical,liao2020batch,hu2021personalized}.

In this paper, we focus on developing a new model-free approach for OPE. The existing model-free OPE methods under the infinite-horizon discounted setting can be mainly divided into three categories. The methods in the first category \citep[e.g.,][]{bertsekas1995dynamic,sutton2018reinforcement,le2019batch,luckett2019estimating, shi2020statistical} directly estimate the state(-action) value function (see Section \ref{sec: preliminary}), based on the Bellman equation (see \eqref{eq: Bellman equation for Q}).
The second category is motivated by importance weights or the so-called marginal importance sampling \citep[e.g.,][]{liu2018breaking, nachum2019dualdice,uehara2020minimax}. 
These approaches utilize the probability ratio function (see \eqref{eqn:true_weight}) to adjust for the distributional mismatch due to the difference between the target policy and the behavior policy (i.e., the one in the data collection process).
The last category of OPE methods combines methods in the first category and the second category to obtain the so-called ``double robust'' estimators \citep[e.g.,][]{ kallus2019efficiently, tang2019doubly,shi2021deeply}.
{Apart from the model-free methods, we note in passing that
there exists a rich class of literature on
	model-based methods which mainly rely on directly modeling the dynamics
	(i.e., the transition mechanism and the reward function) \citep[e.g.][]{paduraru2007planning, chua2018deep, janner2019trust}.}

This paper focuses on the second category of model-free OPE methods, which is based on the probability ratio function.
These methods do not depend on the actual form of the reward function and thus can be used flexibly to evaluate the target policy on different reward functions, which is appealing in practice.
Some core theoretical questions associated with
the use of the (estimated) probability ratio function
directly relate to
the fundamental challenges of offline RL \citep{levine2020offline}, and therefore have recently attracted much interest from the machine learning community \citep{liu2018breaking,nachum2019dualdice, zhang2020gendice,uehara2020minimax, Uehara-Imaizumi-Jiang21}.
Last but not least, 
the probability ratio function and related estimators can improve the accuracy and stability of  offline RL, which has been demonstrated in \citep[][]{nachum2019dualdice,kallus2020statistically}.
Despite some recent progress towards using the (estimated) probability ratio function to perform OPE,
the corresponding development and (theoretical) understanding are still incomplete.

Motivated by the covariate balancing methods \citep[e.g.,][]{Athey-Imbens-Wager18,Wong-Chan18, wang2020minimal, Kallus20} that have been recently studied in the average treatment effect (ATE) estimation, we propose a novel OPE estimator via projected state-action balancing weights,
under the framework of the time-homogeneous MDP and provide a comprehensive theoretical analysis. Specifically, we characterize the special challenge in developing a weighted estimator for the OPE problem in terms of an ``expanded dimension'' phenomenon, which significantly complicates the adoption of the balancing idea in OPE, as compared with the ATE estimation.
 See the detailed discussion in Section \ref{sec:mismatch}.
 Roughly speaking, a direct modification of an ATE estimator \citep[e.g.,][]{wang2020minimal} for the purpose of OPE leads to estimated weights that depend on not only the current state-action pairs, but also the state variable in the next decision points, which is inconsistent with the definition of the true probability ratio function.
 To tackle this issue, we propose an approximate projection step to rule out this ``expanded dimension" issue.  With the help of this additional projection, we are able to show the convergence rate of the proposed weights. (or the estimated ratio function).

As for theoretical contributions, we analyze the convergence rates of the approximate projection step,
the projected state-action balancing weights, as well as the estimated value of a given policy.
All these convergence rates are characterized in the scaling of both the sample size (\(n\), the number of observed trajectories) and the number of observed decision points (\(T\)) per trajectory.    The scaling with respect to \(T\) is particularly important when $n$ (the number of subjects) is limited, which is common in mHealth datasets such as the aforementioned OhioT1DM dataset (see also Section \ref{sec:real}).
For instance, under our setup, the estimated value of the policy is still consistent in the asymptotic setting where $n$ is bounded, but $T\rightarrow\infty$.
In the course of analyzing the proposed method,
we also obtain a uniform convergence rate (with respect to \(n\) and \(T\)) for a non-parametric regression based on exponentially \(\beta\)-mixing sequences, which may be of independent interest for other applications.
Furthermore, under some appropriate technical assumptions (including notably a non-diminishing minimum eigenvalue condition), we show that the proposed weighted estimator is asymptotically normal with a \(\sqrt{nT}\) rate of convergence, and achieves the efficiency bound, which aligns with other types of estimators \citep{kallus2019efficiently,liao2020off}.
Besides, our theoretical results do not require that the underlying data are independent or generated from stationary sequences, {although these assumptions are widely} used in the existing literature \cite[e.g.,][]{farahmand2016regularized,kallus2019efficiently,nachum2019dualdice,tang2019doubly} for {general OPE methods}. 
Without imposing these restrictive assumptions can significantly increase the applicability of the proposed method.

As another theoretical contribution, we make the first attempt to analyze the difficulty of non-parametric OPE under infinite-horizon discounted settings.
Some strong assumptions are often imposed in the existing literature to establish desirable theoretical properties for corresponding OPE estimators.
For instance, \cite{shi2020statistical} assumed that the minimal eigenvalue of some second moment is strictly bounded away from $0$ as the number of basis functions grows.
\cite{uehara2021finite} adopts the ``completeness'' condition
to study the convergence of their OPE estimators.
\revise{In their work, they show that these assumptions can be satisfied when 
the discount factor $0 \leq \gamma < 1$ in the policy value \eqref{def: integrated value fun} is small enough (with additional boundedness conditions on the average visitation probability).
However, in practice,
the discount factor is often preferred to be set close to 1 \citep{komorowski2018artificial}. As such, these sufficient conditions are of limited use in practice.
In this paper, we provide a necessary and sufficient condition for lower bounding the minimal eigenvalue of the operator \(I- \gamma \mathcal{P}^\pi\) (see the definition in Section \ref{sec:mineigen}), with respect to the data generating process, which characterizes the well-posedness of their $Q$-function estimations in \cite{shi2020statistical} and \cite{uehara2021finite}. With the help of this characterization, we can further show that the minimal eigenvalue is strictly bounded away from zero under some mild sufficient conditions without any restrictions on \(\pi\) or \(\gamma\), which may be of independent interest. }

The rest of the paper is organized as follows. Section \ref{sec:basic} presents the basic framework for OPE and some existing OPE methods.  In Section \ref{sec:method} we introduce the state-action balancing weights and the proposed estimator. Theoretical results of our estimated weight function and the asymptotic properties of the proposed estimator are developed in Section \ref{sec:theory}. A detailed discussion regarding the lower bound of the minimum eigenvalue is presented in Section \ref{sec:mineigen}. 
Lastly, a simulation study and a real data application are presented in Sections \ref{sec:sim} and \ref{sec:real} respectively.

\section{Offline Policy Evaluation in Infinite-horizon Markov Decision Processes}
\label{sec:basic}
In this section, we review the framework of discrete-time homogeneous MDPs and the related OPE methods. Specifically,
the framework of MDP and some necessary notations
are presented in Section \ref{sec: preliminary},
while 
three major classes of model-free OPE methods under the infinite-horizon discounted setting are reviewed in Section \ref{sec: review}.
\subsection{Preliminary and Notations}\label{sec: preliminary}
 Consider a trajectory $\left\{S_t, A_t, R_t\right\}_{t \geq 0}$, where $(S_t, A_t, R_t)$ denotes the triplet of the state, action and immediate reward, observed at the decision point $t$. Let $\calS$ and $\calA$ be the state and action spaces, respectively. We assume $\calA$ is finite and consider the following two assumptions, which are commonly imposed in infinite-horizon OPE problems.
\begin{assumption}[Markovian assumption with stationary transitions]\label{ass: Markovian}
	There exists a transition kernel $P$ such that for every $t\geq 0$, $ a \in \calA$, $s \in \calS$ and any set $F \in \calB(\calS)$,
	$$
	\Pr(S_{t+1}  \in F \given A_t = a, S_t = s, \left\{S_j, A_j, R_j\right\}_{0 \leq j < t}) = P(S_{t+1}  \in F \given A_t = a, S_t = s),
	$$
	where $\calB(\calS)$ is the family of Borel subsets of $\calS$.
\end{assumption}
For the notational convenience, we assume the transition kernel has a density $p$. Assumption \ref{ass: Markovian} requires that given the current state-action pair, future states are independent of past observations. Note that this assumption can be tested using the observed data \citep[e.g.,][]{shi2020does}. If one believes the trajectory satisfies some higher-order Markovian properties, one can instead construct a state variable by aggregating all the original state variables over the corresponding number of decision points.
We refer the interested readers to \cite{shi2020does} for an example.
\begin{assumption}\label{ass: reward}
	There exists a reward function $r$ defined on $\calS \times \calA$ such that\\
	$
	\EE\left[R_t \given A_t = a, S_t = s, \left\{S_j, A_j, R_j\right\}_{0 \leq j < t}\right] = r(s, a),
	$
	for every $t \geq 0$, $s \in \calS$ and $a \in \calA$. In addition, $R_t$ is uniformly bounded, i.e., there exists a constant $R_{\max}$ such that $\abs{R_t} \leq R_{\max}$ for every $t \geq0$.
\end{assumption}
Assumption \ref{ass: reward} states that the current reward is conditionally mean independent of the history given the current state-action pair. One can also regard $R_t$ as a part of $S_{t+1}$ if needed. The uniform boundedness assumption on rewards is introduced for simplifying the technical derivation and can be relaxed. In practice, the time-stationarity of the transition density $p$ and the reward function $r$ can be warranted by incorporating time-associated covariates.  Under Assumptions \ref{ass: Markovian} and \ref{ass: reward}, the tuple $\calM = \langle\calS, \calA, r, p\rangle$ forms a discrete-time homogeneous MDP.

A policy is defined as a way of choosing actions at all decision points. In this work, we use the value criterion (defined in \eqref{def: integrated value fun} below) to evaluate policies and thus focus on the time-invariant Markovian policy (also called a stationary policy) $\pi$, which is a function mapping from the state space $\calS$ to a probability mass function over the action space $\calA$. More specifically, $\pi(a \given s)$ is the probability of choosing an action $a$ given a state $s$. The sufficiency of considering only the stationary policy is explained in Section 6.2 of \cite{puterman1994markov}.

In the offline RL setting, one of the fundamental tasks is to estimate a target stationary policy's (expected) value function based on the pre-collected (batch) data. Given a stationary policy $\pi$, the value function is defined as
\begin{align}\label{def: value fun}
	V^\pi(s) = \EE^\pi\left[\sum_{t = 0}^{\infty}\gamma^tR_t \given S_0 = s\right],
\end{align}
where $\EE^\pi$ denotes the expectation with respect to the distribution whose actions are generated by $\pi$, and $0 \leq  \gamma < 1$ refers to the discounted factor controlling the trade-off between the future rewards and the immediate rewards.
The value function $V^\pi(s)$, which is always finite due to Assumption \ref{ass: reward}, represents the discounted sum of rewards under the target policy given the initial state $s$. Our goal is to estimate the \textit{policy value}, i.e., the expectation of the value function, defined as
\begin{align}\label{def: integrated value fun}
	\calV(\pi) = (1- \gamma) \int_{s\in \mathcal{S}} V^\pi(s) \mathbb G(ds),
\end{align}
using the pre-collected training data,
where $\mathbb G$ denotes a reference distribution over $\calS$. In RL literature, $\mathbb G$ is typically assumed \textit{known}. 

Now suppose we are given pre-collected training data $\calD_n$ consisting of $n$ independent and identically distributed finite-horizon trajectories, denoted by
$$
\calD_n = \left\{ \left\{\left(S_{i, t}, A_{i, t}, R_{i, t}, S_{i, t+1}\right)\right\}_{0 \leq t < {T_i}} \right\}_{1 \leq i \leq n},
$$
where {$T_i$} denotes the termination time for \(i\)-th trajectory. {For simplicity, we assume the same number of time points are observed for all trajectories, i.e. \(T_i = T\) for \(i = 1, \dots, n\).  This assumption can be relaxed as long as \(T_i\), \(i=1,\dots,n\) are of the same order.} We also make
the following assumption on the data generating mechanism.
\begin{assumption}\label{ass: DGP}
	The training data $\calD_n$ is generated by a fixed stationary policy $b$ with an initial distribution $\mathbb G_0$.
\end{assumption}
Under Assumptions \ref{ass: Markovian} and \ref{ass: DGP}, $\left\{S_t, A_t\right\}_{t\geq 0}$ forms a discrete time time-homogeneous Markov chain. In the literature, $b$ is usually called the behavior policy, which may not be known.
For convenience, we denote $\EE$ by $\EE^b$. Next we define notations for several important probability distributions.
 For any $t \geq 0$, we define $p_t^\pi(s, a)$ as the marginal density of $(S_t, A_t)$ at $(s, a) \in \calS \times \calA$ under the target policy $\pi$ and reference distribution $\mathbb{G}$. In particular, when $t = 0$, $p_t^\pi(s, a) = \mathbb{G}(s)\pi(a \mid s)$. Similarly, we can define $p_t^b$ over $\calS \times \calA$ under the behavior policy, where $p^b_0(s, a) = \mathbb{G}_0(s)b(a \mid s)$. With $p_t^\pi$, the discounted visitation probability density is defined as
\begin{align}\label{def: discounted visitation measure}
	d^{\pi}(s, a) = (1- \gamma)\sum_{t = 0}^\infty \gamma^t p^\pi_t(s, a),
\end{align}
which is assumed to be well-defined.

\subsection{Existing Off-policy Evaluation Methods}\label{sec: review}
Most existing model-based OPE methods for the above setting can be grouped into three categories. The first category is direct methods,
which directly estimate the state-action value function
\citep[e.g.,][]{luckett2019estimating,shi2020statistical}
defined as
\begin{align}\label{def: Q fun}
 Q^\pi(s, a) = \EE^\pi\left[\sum_{t = 0}^{\infty}\gamma^tR_t \given S_0 = s, A_0 = a\right],
\end{align}
also known as the Q-function. As we can see from \eqref{def: value fun} and \eqref{def: integrated value fun},
\begin{equation}\label{eq: value function}
\calV(\pi)= (1- \gamma)\int_{s\in \mathcal{S}}\sum_{a \in \calA} \pi(a \mid s) Q^\pi(s, a)\mathbb G(ds).
\end{equation}
It is well-known that $Q^\pi$ satisfy the following Bellman equation
\begin{align}\label{eq: Bellman equation for Q}
	Q^\pi(s, a) = \EE\left[R_t + \gamma \sum_{a \in \calA} \pi(a' \given S_{t+1})Q^\pi(S_{t+1}, a')  \given S_t = s, A_t = a\right],
\end{align}
 for $t\geq 0$, $s \in \calS $ and $a \in \calA$.
Clearly \eqref{eq: Bellman equation for Q} forms a conditional moment restriction, based on which $Q^\pi$ can be estimated by many methods such as generalized method of moments \citep{hansen1982large} and the nonparametric instrumental variable regression \citep{newey2003instrumental,chen2007large}. 
The second category of {OPE} methods is motivated by the idea of marginal importance sampling
\citep[e.g.,][]{liu2018breaking,nachum2019dualdice}.
Notice that, with \eqref{def: discounted visitation measure},
one can rewrite $\calV(\pi)$ as
\begin{align}
	\label{eqn:Vpi}
	\calV(\pi) = {\EE\left[\frac{1}{T}\sum_{t = 0}^{T-1}\frac{d^{\pi}(S_t, A_t)}{\frac{1}{T}\sum_{v = 0}^{T-1} p_v^b(S_v, A_v)}R_{t}\right]} = \EE\left[\frac{1}{T}\sum_{t = 0}^{T-1}\omega^{\pi}(S_t, A_t)R_{t}\right],
\end{align}
as long as $d^\pi$ is absolutely continuous with respect to $\bar{p}_T^b = \frac{1}{T}\sum_{v = 0}^{T-1}p_v^b$. We call 
\begin{align}
	\label{eqn:true_weight}
	\omega^\pi(s,a) := \frac{d^{\pi}(s, a)}{\bar{p}_T^b(s, a)},
\end{align}
the probability ratio function, which is used to adjust for the mismatch between the behavior policy $b$ and the target policy $\pi$.
 Based on this relationship, one can obtain an estimator of $\calV(\pi)$ via estimating $\omega^\pi$. By the so-called backward Bellman equation of  $d^{\pi}$, for any $(s, a) \in \calS \times \calA$, one can show that
\begin{align}\label{eq: fwd bellman d}
d^{\pi}(s, a) = (1-\gamma)\mathbb G(s)\pi(a\mid s) + \gamma \int_{s'\in \mathcal{S}} \sum_{a'\in \calA}d^{\pi}(s', a')p(s \mid s', a')\pi(a \mid s)ds'.
\end{align}
This implies that
\begin{multline}\label{eq: Estimating Equation for ratio}
\EE\left[\frac{1}{T}\sum_{t = 0}^{T-1}\omega^{\pi}(S_t, A_t)\left(f(S_t, A_t) - \gamma \sum_{a' \in \calA}\pi(a'\mid S_{t+1})f(S_{t+1}, a')\right)\right] =\\ (1-\gamma)\EE_{S_0 \sim \mathbb G}\left[ \sum_{a\in \calA} \pi(a\mid S_0)f(a,S_0) \right].
\end{multline}
See the detailed derivation for obtaining \eqref{eq: Estimating Equation for ratio} in 
\supp{\ref{sec:proof_sec2}}
Recall that $\mathbb G$ is known. Based on \eqref{eq: Estimating Equation for ratio}, several methods \citep[e.g.,][]{nachum2019dualdice, uehara2020minimax} can be leveraged to estimate $\omega^\pi$. {In Section \ref{sec:other}, we provide more discussion on other weighted estimators as compared with the proposed estimator.}

The last category of methods combines direct and marginal importance sampling methods to construct a doubly robust estimator, which is also motivated by the following efficient influence function \cite[e.g.,][]{kallus2019efficiently}
\begin{multline}\label{eq: EIF}
\frac{1}{T}\sum_{t=0}^{T-1}\omega^{\pi}(S_t, A_t)\left(R_t + \gamma \sum_{a'\in \calA}\pi(a' \mid S_{t+1})Q^\pi(S_{t+1}, a') -Q^\pi(S_t, A_t)\right)\\
+  (1-\gamma)\EE_{S_0 \sim \mathbb G}\left[\sum_{a_0\in \calA}\pi(a_0 \mid S_0)Q^\pi(S_0, a_0)\right]  - \calV(\pi),
\end{multline}
where $Q^\pi$ and $\omega^\pi$ are nuisance functions.

\section{Projected State-action Balancing Estimator}
\label{sec:method}

In this section, we introduce the proposed weighted estimator for $\calV(\pi)$. Since our estimator is motivated by covariate balancing weights in the literature of ATE estimation, we discuss their connection in Section \ref{sec: covariate balance}. In Section \ref{sec:mismatch}, we show the difficulty of directly applying the covariate balancing idea in the aforementioned policy evaluation problem due to an ``expanded-dimension'' issue. We address this difficulty and propose a projected state-action balancing estimator for $\calV(\pi)$ in Section \ref{sec:projected}.

\subsection{State-action Balancing Weights}\label{sec: covariate balance}
Consider a general form of weighted estimators:
\[
\frac{1}{nT}\sum_{i=1}^n \sum_{t=0}^{T-1} \omega_{i,t} R_{i,t},
\]
where $\omega_{i,t}$'s are some weights constructed from the training data $\mathcal{D}_n$.
Due to \eqref{eqn:Vpi},
a reasonable strategy to derive such weights is to first estimate the probability ratio function $\omega^\pi$, and then evaluate it at the observed state-action pairs $\{(S_{i,t}, A_{i,t})\}$ in $\calD_n$.
This is analogous to the inverse probability weights commonly adopted in the ATE estimation.
However, this strategy often produces an unstable weighted estimator due to small weights \citep{kang2007demystifying}.
Instead, there is a recent surge of interest to directly obtain weights that achieve empirical covariate balance
\citep[e.g.,][]{Imai-Ratkovic14, Chan-Yam-Zhang15, Wong-Chan18, wang2020minimal}.
These weights usually produce more stable weighted estimators with superior finite-sample performances for the ATE estimation.

Inspired by the covariate balancing,
a natural idea is to choose the weights
$\{\omega_{i,t}\}$
that ensure the (approximate) validity of the empirical counterpart of \eqref{eq: Estimating Equation for ratio}, i.e.,
\begin{multline}
\label{eqn: Empirical Estimating Equation for ratio}
\frac{1}{nT}\sum_{i=1}^n \sum_{t=0}^{T-1} \omega^\pi_{i,t} \left(f(S_{i,t}, A_{i,t}) - \gamma \sum_{a' \in \calA}\pi(a'\mid S_{i,t+1})f(S_{i,t+1}, a')\right) =\\
(1-\gamma)\EE_{S_0 \sim \mathbb G}\left[ \sum_{a\in \calA} \pi(a \mid S_0)f(a,S_0) \right],
\end{multline}
over $f\in\mathrm{span}\{B_1,\dots, B_K\}$
where $B_1,\dots, B_K$ are $K$ pre-specified functions defined on $\mathcal{S}\times \mathcal{A}$.
The equality \eqref{eqn: Empirical Estimating Equation for ratio}
can be viewed as a form of the state-action balance, in contrast to the covariate balance in the ATE estimation.
The space $\mbox{span}\{B_1,\dots, B_K\}$ can be viewed as a finite-dimensional approximation of the function space in which the balance should be enforced.
In theory, \(K\) is expected to increase with \(n\) and \(T\).
Changing the balancing criterion in \cite{wang2020minimal},
one can obtain a form of state-action balancing weights
via
the following mathematical programming:
\begin{subequations}\label{eqn:optim1}
	\begin{align}
	& \underset{\{{\omega}_{i,t}\}_{1\leq i \leq n, 0 \leq t \leq T-1}}{\mbox{minimize}} \quad \frac{1}{nT}\sum_{i=1}^n\sum_{t=0}^{T-1} h({\omega}_{i,t}) \label{eqn:obj1}\\
	& \quad \mbox{subject to}  \left| \frac{1}{nT}\sum_{i=1}^n \sum_{t=0}^{T-1} {\omega}_{i,t} \left\{ B_k(S_{i,t}, A_{i,t}) - \gamma \sum_{a' \in \calA}\pi(a'\mid S_{i,t+1})B_k(S_{i,t+1}, a')\right\} 
	\right. \nonumber\\
	& \qquad \quad \left.- (1-\gamma) \EE_{S_0\sim \mathbb{G}} \left[ \sum_{a'\in \mathcal{A}}\pi(a'\mid S_0) B_k(S_0,a') \right] \right| \leq \delta_k, \quad
	\mbox{for} \ k = 1,2,\dots, K, \label{eqn:cons1}
	\end{align}  
\end{subequations}
where the tuning parameters \(\delta_k\ge 0\) controls the imbalance with respect to \(B_k\). When \(\delta_k = 0\), the weights achieve the exact balance \eqref{eqn: Empirical Estimating Equation for ratio} over $B_k$. In practice, the exact balance can be hard to achieve especially when \(K\) is large. Allowing \(\delta_k > 0\) leads to approximate balance \citep{Athey-Imbens-Wager18,Wong-Chan18}, which introduces flexibility.
In addition, one can also constrain the weights to be non-negative.
Since non-negativity constraints are not necessary for consistent estimation (as shown in our theoretical analysis), we will not enforce them throughout this paper.
Common choices of $\{B_k\}$ are constructed based on tensor products of
one-dimensional basis functions.
Examples of one-dimensional basis functions include spline basis \citep{de1976splines} (for a continuous dimension) and indicator functions of levels (for a categorical dimension).
The objective function \eqref{eqn:obj1} is introduced to control the magnitude of the weights.
Here \(h\) is chosen as a non-negative, strictly convex and continuously differentiable function. Examples include \(h(x) = (x-1)^2\) and \(h(x) = (x-1)\log(x)\). In the following, we discuss the issue with the weights defined by \eqref{eqn:optim1},
and explain the challenge of directly applying this covariate balancing idea in the OPE problem.

Define  \(\epsilon_{i,t} = R_{i,t} - r(S_{i,t}, A_{i,t})\) and write the solution of \eqref{eqn:optim1} as \(\tilde\omega^\pi_{i,t}\).
Naturally, we can obtain a weighted estimator as \(\tilde\calV(\pi) = (nT)^{-1}\sum_{i=1}^n \sum_{t=0}^{T-1} \tilde\omega_{i,t}^\pi R_{i,t}\). 
For any function \(B\) defined on \(\mathcal{S}\times \mathcal{A}\), let
\begin{align}
	\label{eqn:gstar}
	g_*^\pi(S_t,A_t; B):=\EE\left\{ \sum_{a' \in \mathcal{A}}\pi(a' \mid S_{t+1}) B(S_{t+1}, a') \mid S_{t}, A_{t}\right\}.
\end{align}
The difference between \(\tilde{\calV}(\pi)\) and \(\calV(\pi)\) yields the following decomposition:
\begin{align}
&\tilde{\calV}(\pi) - \calV(\pi) \nonumber\\
	= &\frac{1}{nT}\sum_{i=1}^{n}\sum_{t=0}^{T-1}\tilde {\omega}^\pi_{i,t}\left[  Q^\pi(S_{i,t}, A_{i,t}) - \gamma g_*^\pi(S_{i,t}, A_{i,t}; Q^\pi)  \right] \label{eqn:bal_exp1}\\
	& \qquad - (1-\gamma)\EE_{S_0 \sim \mathbb{G}}\left\{ \sum_{a' \in \mathcal{A}}\pi(a'\mid S_0)Q^\pi(S_0, a') \right\} 
    + \frac{1}{nT}\sum_{i=1}^{n}\sum_{t=0}^{T-1}\tilde {\omega}^\pi_{i,t}\epsilon_{i,t} \label{eqn:bal_exp2}\\
	= &\frac{1}{nT}\sum_{i=1}^{n}\sum_{t=0}^{T-1}\tilde {\omega}^\pi_{i,t}\left[  Q^\pi(S_{i,t}, A_{i,t}) - \gamma\sum_{a' \in \mathcal{A}}\pi(a' \mid S_{i, t+1}) Q^\pi(S_{i, t+1}, a') \right]    \label{eqn:Q_control1}\\
    &\qquad - (1-\gamma)\EE_{S_0 \sim \mathbb{G}}\left\{ \sum_{a' \in \mathcal{A}}\pi(a'\mid S_0)Q^\pi(S_0, a') \right\}  \label{eqn:Q_control2}\\
    &  +\frac{1}{nT}\sum_{i=1}^{n}\sum_{t=0}^{T-1}\tilde {\omega}^\pi_{i,t}\epsilon_{i,t} \label{eqn:error_term} \\
    & + \frac{\gamma}{nT}\sum_{i=1}^{n}\sum_{t=0}^{T-1}  \tilde \omega_{i,t}^\pi \left[\sum_{a' \in \mathcal{A}}\pi(a' \mid S_{i, t+1}) Q^\pi(S_{i, t+1}, a') -  g_*^\pi(S_{i,t}, A_{i,t}; Q^\pi)\right],\label{eqn:error_term1}
\end{align}
where the first equality is given by \eqref{eq: Bellman equation for Q} and the representation of $\calV(\pi)$ in \eqref{eq: value function}.  Clearly, \eqref{eqn:Q_control1}-\eqref{eqn:Q_control2} can be controlled via the balancing constraint \eqref{eqn:cons1} by carefully controlling \(\{\delta_k\}\) and selecting \(\{B_k\}\) so that \(Q^\pi\) can be well approximated. For \eqref{eqn:error_term}, by assuming that \(\{\epsilon_{i,t}\}\) are independent noises of the trajectories \(\{S_{i,t}, A_{i,t}\}_{i=1,
\dots n; t = 0,\dots,T-1}\), it may not be hard to obtain an upper bound of order \(\sqrt{nT}\) as long as the magnitude of \(\tilde \omega_{i,t}^\pi\) is properly controlled. 
However, it remains unclear how to control \eqref{eqn:error_term1} due to the complex dependence between \(\tilde \omega_{i,t}^\pi\) and $\calD_n$.
Indeed, we observe an ``expanded-dimension'' issue due to the balancing constraints \eqref{eqn:cons1}, which will be explained in details in Section \ref{sec:mismatch}.
This also motivates the development of the novel projected balancing constraints
in Section \ref{sec:projected}.

\subsection{Expanded Dimension}
\label{sec:mismatch}
First, we obtain the dual form of \eqref{eqn:optim1}, which provides an important characterization for the solution of \eqref{eqn:optim1}. Define \(\bm \Psi_K(s,a,s') = [B_k(s, a) - \gamma  \sum_{a'\in \mathcal{A}} \pi(a'\mid s') B_k(s',a')]_{k=1}^K\in \mathbb{R}^{K}\),  \(\bm l_K = [\EE_{s\sim \mathbb{G}} \{(1-\gamma) \sum_{a\in \mathcal{A}}\pi(a\mid s) B_k(s,a)\} ]_{k=1,\dots, K} \in \mathbb{R}^{K}\), and \(\bm \delta_K = [\delta_1, \dots, \delta_K]^{\tp} \in \mathbb{R}^{K}.\)

\begin{theorem}\label{thm:dual1}
    The dual of \eqref{eqn:optim1} is equivalent to the following unconstrained optimization problem:
    \begin{align}
        \label{eqn:dual1}
     \min_{\bm \lambda \in \mathbb{R}^K} \qquad \frac{1}{nT}\sum_{j=1}^n\sum_{t=0}^{T-1} \rho({\bm \Psi}_K(S_{i,t}, A_{i,t}, S_{i,t+1})^\tp\bm \lambda)
     - \bm \lambda^\tp \bm l_K + |\bm \lambda|^\tp \bm \delta_K,
    \end{align}
    where $\rho(t) = t(h')^{-1}(t) - h{(h')^{-1}(t)} $.
    The primal solution to \eqref{eqn:optim1} is given as
    \begin{align}
        \tilde{w}^\pi_{i,t} = \rho'({\bm \Psi}_K(S_{i,t}, A_{i,t}, S_{i,t+1})^\tp\bm \lambda^{\star}),\label{eqn:solution_weights1},
    \end{align}
for every $1 \leq i \leq n$ and $0 \leq t \leq T-1$,
    where $\bm \lambda^\star $ is the solution to \eqref{eqn:dual1}.
    \end{theorem}
The proof of Theorem \ref{thm:dual1} is similar to that of Theorem \ref{thm:dual} below, and can be found in \supp{\ref{sec:proof_dual}}.
Now the expanded-dimension issue can be easily seen via the following example. Suppose that there are two triplets \((S_{i_1,t_1}, A_{i_1,t_1}, S_{i_1,t_1+1})\) and \((S_{i_2,t_2}, A_{i_2,t_2}, S_{i_2,t_2+1})\) such that \(S_{i_1,t_1} = S_{i_2,t_2}\), \(A_{i_1,t_1} = A_{i_2,t_2}\) and \(S_{i_1,t_1+1} \neq S_{i_2,t_2+1}\).  As the true  probability ratio function \(\omega^\pi\) is a function of the current state and action variables, we must have \(\omega^\pi_{i_1,t_1} = \omega^\pi_{i_2,t_2}\). 
However,  the solution form  \eqref{eqn:solution_weights1} in Theorem \ref{thm:dual1} does not lead to \(\tilde \omega^\pi_{i_1,t_1} = \tilde \omega^\pi_{i_2,t_2}\) in general, which violates  our knowledge of \(\omega^\pi\).

One may hypothesize that the expanded-dimension issue is a finite-sample property, and the variation of the estimated weights due to the next state may
diminish asymptotically under some reasonable conditions.
To gain more insight, we
show that the solution form \eqref{eqn:solution_weights1} indeed
induces an implicit restriction on the modeling of the true weight function under finite-state and finite-action settings.
Therefore, unless one is willing to make further non-trivial assumptions on the weight function, the hypothesis cannot be true in general.

Notice that
\begin{equation}
    \bm \Psi_K(s,a,s')^\intercal \bm\lambda = f(s, a) - \gamma  \sum_{a'\in \mathcal{A}} \pi(a'\mid s') f(s',a'),
\end{equation}
where $f(s,a) = \sum^K_{k=1} \lambda_k B_k(s,a)$.
To avoid dealing with the approximation error, we focus on an even more general class
$\rho'(\mathcal{G})$ of functions on $\mathcal{S}\times \mathcal{A}\times \mathcal{S}$,
where
\begin{equation*}
\mathcal{G} := \left\{f(s,a)-\gamma\sum_{a'\in\mathcal{A}} \pi(a' \mid s') f(s', a'): \mbox{$f$ is any real-valued function defined on $\mathcal{S}\times \mathcal{A}$}\right\}.
\end{equation*}
Recall that $\rho'$ is the first derivative of $\rho$ defined in Theorem \ref{thm:dual1}.
Assume that $\tilde{\omega}^\pi(s,a,s') \equiv \omega^\pi(s,a)$.
We would like to know if $\rho'(\mathcal{G})$ can model
$\tilde{\omega}^\pi$ well.
Suppose $\tilde\omega^{\pi}\in\rho'(\mathcal{G}):=
\{\rho'(g(\cdot)): g\in\mathcal{G}\}$.
As $\tilde{\omega}^{\pi}(s,a,s')$ is constant with respect to $s'$,
we have $\tilde{\omega}^\pi \in \rho'(\mathcal{G}')$ where
\begin{align*}
\mathcal{G}'&:=\left\{g\in\mathcal{G}: g(s,a,s') =g(s,a,s'')\ \forall s,s',s''\in\mathcal{S}, a\in\mathcal{A}\right\}
\end{align*}
characterizes the subclass with reduced input dimensions.
A key question is whether $\mathcal{G}'$
restricts the class of possible weight functions modeled by $\rho'(\mathcal{G})$ and, as a result, induces some implicit form of restriction.
To see this, we focus on the settings
where $|\mathcal{S}|=p_S$ and $|\mathcal{A}|=p_A$ are finite.
In Lemma \ref{lem:expand_dim_example} below, we show that the dimension of $\mathcal{G}'$ is $p_Ap_S-p_S+1$, which is strictly less than $p_Sp_A-1$ as long as $p_S>2$.
Note that, due to the natural constraint that $\EE[T^{-1}\sum_{t=0}^{T-1}\omega^\pi(S_t, A_t)] = 1$, a general weight function should have $p_Ap_S-1$ free parameters.
As $\rho'$ is invertible, Lemma \ref{lem:expand_dim_example} suggests 
{a possible implicit restriction and that the solution obtained from \eqref{eqn:optim1} may not be a consistent estimator for \(\omega^\pi\).}
\begin{lemma}\label{lem:expand_dim_example}
Suppose $|\mathcal{S}|=p_S$ and $|\mathcal{A}|=p_A$ are both finite.
Then $\dim(\mathcal{G}')= p_A p_S-p_S+1$.
\end{lemma}
	\begin{proof}[Proof sketch]
		{For any function \(g = f(s,a) - \gamma \sum_{a'\in\mathcal{A}} \pi(a'|s') f(s',a') \in \mathcal{G}'\), there exists a constant \(c_g \in \mathbb{R}\) such that \(\sum_{a'\in\mathcal{A}} \pi(a'|s') f(s',a') \equiv c_g\) for any \(s' \in \mathcal{S}\). Since \(|\mathcal{S}| = p_S\), \(\sum_{a'\in \mathcal{A}}\pi(a'|s')=1\) and \(\pi(a'\mid s')\ge 0\), this yields \(p_S\) linearly independent constraints on \(f\). Together with the parameter \(c_g\), we can show that the dimension of \(\mathcal{G}'\) is \(p_Sp_A - p_S + 1\). The detailed proof can be found in \supp{\ref{sec:proof_dual}}.}
	\end{proof}

\subsection{Projected State-action Balancing Weights}
\label{sec:projected}
To overcome the expanded-dimension issue, we propose an approximate projection step, which is applied to
\(\sum_{a'\in \mathcal{A}} \pi(a'\mid S_{i,t+1}) B_k(S_{i,t+1}, a')\), \(k = 1, \dots, K\), to rule out the involvement of \(S_{i,t+1}\).
To explain the idea, we again focus on the decomposition of \(\tilde \calV(\pi) - \calV(\pi)\).
From \eqref{eqn:bal_exp1} and \eqref{eqn:bal_exp2},
we would like to choose weights that ideally control
\begin{align*}
	\left| \frac{1}{nT}\sum_{i=1}^n \sum_{t=0}^{T-1} {\omega}_{i,t} \left[ B_k(S_{i,t}, A_{i,t}) - \gamma g_*^\pi(S_{i,t}, A_{i,t}; B_k)\right] 
		  - (1-\gamma) \EE_{S_0\sim \mathbb{G}} \left[ \sum_{a'\in \mathcal{A}}\pi(a'\mid S_0) B_k(S_0,a') \right] \right|, \nonumber
\end{align*}
for every \(k = 1,\dots,K\).
However, in practice,  \(g_*^\pi(S_{i,t}, A_{i,t}; B_k)\) (i.e., \(\EE \{ \sum_{a' \in \calA}\pi(a'\mid S_{i,t+1})B_k(S_{i,t+1}, a')\mid S_{i,t}, A_{i,t}\}\)) is unknown to us.
As explained in Section \ref{sec:mismatch},
the idea of replacing it with the empirical counterpart \(\sum_{a' \in \calA}\pi(a'\mid S_{i,t+1})B_k(S_{i,t+1}, a')\)
results in a non-trivial expanded-dimension issue.
Instead, we propose to estimate the projection term $ g_*^\pi(S_t,A_t; B_k)$ via a more involved optimization problem:
\begin{align}
   \label{eqn:approx_g}
   \hat g^\pi(\cdot, \cdot; B_k)  = \argmin_{g \in \Gscr} \frac{1}{nT} \sum_{i=1}^n \sum_{t=0}^{T-1} \left\{ \sum_{a'} \pi(a'\mid S_{i,t+1}) B_k(S_{i, t+1}, a') - g(S_{i,t}, A_{i,t}) \right\}^2 + \mu J^2_{\mathcal{G}}(g), 
\end{align}
where \(\Gscr \) is a pre-specified function space that contains \(g_*^\pi\),
\(\mu\ge 0\) is a tuning parameter and \(J_{\mathcal{G}}(\cdot)\) is a regularization functional.
In this work we focus on the kernel ridge regression, where
 \(\Gscr\) is a subset of a reproducing kernel Hilbert space (RKHS) and \(J_\mathcal{G}\) is taken as the squared RKHS norm; see Assumption \ref{assum:RKHS}(c) in 
 \supp{\ref{sec:thm_projection}}
  In Theorem \ref{lem:approx_g} of Section \ref{sec:theory}, we establish the finite-sample error bound  of \(\hat g^\pi(\cdot, \cdot; B)\) (in scaling of both $n$ and $T$) that holds \textit{uniformly} for
	different $B$ (in a reasonably large class $\mathcal{Q}$ specified later).
  This provides a solid theoretical guarantee for
	replacing $g^\pi_*$ by $\hat{g}^\pi$
	in the construction of the weights.

With the approximate projection \eqref{eqn:approx_g}, we propose to estimate the weights by solving the following optimization problem:
\begin{subequations}\label{eqn:optim2}
    \begin{align}
       & \underset{\{{\omega}_{i,t}\}_{1\leq i \leq n, 0 \leq t \leq T-1}}{\mbox{minimize}}  \quad \frac{1}{nT}\sum_{i=1}^n\sum_{t=0}^{T-1} h({\omega}_{i,t}) \label{eqn:obj2}\\
       & \quad \mbox{subject to} \quad \left| \frac{1}{nT}\sum_{i=1}^n \sum_{t=0}^{T-1} {\omega}_{i,t} \left\{ B_k(S_{i,t}, A_{i,t}) - \gamma \hat g^\pi(S_{i,t}, A_{i,t}; B_k)\right\} 
       \right. \nonumber\\
      &\qquad \quad \left.- (1-\gamma) \EE_{S_0\sim \mathbb{G}} \left[ \sum_{a'\in \mathcal{A}}\pi(a'\mid S_0) B_k(S_0,a') \right] \right| \leq \delta_k, \quad
       \mbox{for} \quad k = 1,2,\dots, K. \label{eqn:cons}
    \end{align}  
\end{subequations}
The resulting solution $\{\hat{w}_{i,t}^\pi\}_{1\leq i \leq n, 0 \leq t < T}$ are the proposed weights.
As such, the proposed estimator for $\calV(\pi)$ is
\begin{equation}\label{eqn: final estimator}
\hat \calV(\pi) = \frac{1}{nT}\sum_{i=1}^{n}\sum_{t=0}^{T-1} \hat \omega^\pi_{i,t}R_{i, t}.
\end{equation}

Similar to Theorem \ref{thm:dual1}, we derive the dual form of \eqref{eqn:optim2}, which is shown in Theorem \ref{thm:dual} below. For the notational simplicity, we introduce the following notations: \(L_k(s,a) = B_k(s, a) - \gamma  g^{\pi}_*(s, a; B_k)\), \(\hat{L}_k(s,a) = B_k(s, a) - \gamma  \hat{g}^{\pi}(s, a; B_k)\),  \({\bm B}_K(s,a) = [{B}_k(s,a)]_{k=1}^K\in \mathbb{R}^{K}\), \({\bm L}_K(s,a) = [{L}_k(s,a)]_{k=1}^K\in \mathbb{R}^{K}\) and \(\hat{\bm L}_K(s,a) = [\hat{L}_k(s,a)]_{k=1}^K\in \mathbb{R}^{K}\).
\begin{theorem}\label{thm:dual}
    The dual of \eqref{eqn:optim2} is equivalent to the following unconstrained optimization problem: 
    \begin{align}
        \label{eqn:dual}
     \min_{\bm \lambda \in \mathbb{R}^K} \qquad \frac{1}{nT}\sum_{i=1}^n\sum_{t=0}^{T-1} \rho(\hat{\bm L}_K(S_{i,t}, A_{i,t})^\tp\bm \lambda) - \bm \lambda^\tp \bm l_K + |\bm \lambda|^\tp \bm \delta_K.
    \end{align}
    And the primal solution to \eqref{eqn:optim2} is given by, for every $1 \leq i \leq n$ and $0 \leq t \leq T-1$,
    \begin{align}
            \label{eqn:solution_weights2}
        \hat{\omega}^\pi_{i,t} = \rho'(\hat{\bm L}_K(S_{i,t}, A_{i,t})^\tp\bm \lambda^{+}),
    \end{align}
    where $\rho(t) = t(h')^{-1}(t) - h{(h')^{-1}(t)} $ and   $\bm \lambda^+ $ is the solution to \eqref{eqn:dual}.
    \end{theorem}
We sometimes write $\hat{\omega}^\pi_{i,t}  = \hat{\omega}^\pi_{i,t}(S_{i, t}, A_{i, t})$. The proof can be found in \supp{\ref{sec:proof_dual}}. 
As seen from the representation \eqref{eqn:solution_weights2},
\(\hat w^\pi_{i,t}\) do not suffer from the expanded-dimension issue that we see in \eqref{eqn:dual1}.
Besides, \eqref{eqn:dual} can be regarded as an $M$-estimation of $\blam$ with a weighted $\ell_1$-norm regularization.
Since the estimated weights are parametrized in $\blam$ via \eqref{eqn:solution_weights2}, Theorem  \ref{thm:dual} reveals a connection between  \(\hat \omega^\pi_{i,t}\) and the shrinkage estimation of the probability ratio function \(\omega^\pi\).  To see this, we consider the objective function \(T^{-1}\sum_{t=0}^{T-1} \rho({\bm L}_K(S_{t}, A_{t})^\tp\bm \lambda) - \bm \lambda^\tp \bm l_K\). By Lemma 1 in \cite{uehara2020minimax}, the expectation of this loss function is minimized when \(\bm \lambda\) satisfies \(\omega^\pi (s, a) = \rho'(\bm L_K(s, a)^\tp\bm \lambda)\), for every $s\in\mathcal{S}$, $a\in\mathcal{A}$.
  In Theorem \ref{thm:weights} of Section \ref{sec:theory}, we show the convergence rate of the proposed weights to the true weights in the scaling of both $n$ and $T$.

Next, we discuss the computation of the weights \(\hat \omega_{i,t}\) in practice. The projection step \eqref{eqn:approx_g} is a pre-computation step
for the optimization \eqref{eqn:optim2}.  In other words,   we only need to estimate \(\hat g^{\pi}(S_{i,t}, A_{i,t}; B_k)\), \(k = 1,\dots, K\), once and there is no need to recompute them within the weights estimation \eqref{eqn:optim2}. The optimization \eqref{eqn:optim2} is a standard convex optimization problem with linear constraints.
We outline the proposed algorithm of the weights in Algorithm \ref{algo:computation} of the Supplementary Material.

\subsection{Other weighting methods}
\label{sec:other}
Apart from the proposed projection method,
a sensible alternative to avoid the expanded dimension is to
directly
specify a class $\mathcal{F}_w$ of functions over $\mathcal{S}\times \mathcal{A}$
as a model of the weight function (i.e., weights are evaluations of this function), which naturally results in the following form of estimators:
\begin{align}\label{eqn:balancing_1}
    \hat f^\pi_w&= \argmin_{f_1 \in \mathcal{F}_w} \sup_{f_2 \in \mathcal{F}_Q} \left[ \frac{1}{nT}\sum_{i=1}^{n}\sum_{t = 0}^{T-1}f_1(S_{i,t}, A_{i,t})\left\{f_2(S_{i,t}, A_{i,t}) - \right. \right. \nonumber\\ 
	& \left. \left. \gamma \sum_{a'}\pi(a'\mid S_{i,t+1})f_2(S_{i,t+1}, a')\right\} 
     - (1-\gamma)\EE_{S_0 \sim \mathbb G}\left\{ \sum_{a\in \calA} \pi(a \mid S_0)f_2(a,S_0) \right\} - \nu_1 \mathcal{P}_1(f_2) \right],
\end{align}
where \(\mathcal{F}_Q\) is a space that models \(Q^\pi\),   \(\mathcal{P}_1\) is a regularization functional and \(\nu_1 \ge 0\) is a tuning parameter.
Similar approaches have been adopted by a few recent works such as \citep[][]{liu2018breaking,uehara2020minimax, uehara2021finite}, 
although they are not directly motivated by the above expanded-dimension issue.
Since these works assume $T=1$ (without considering the dependence in the trajectory), we restrict our results to this setting for comparisons, with a remark that our estimator is analyzed in a more general setting of $T$ (e.g., either bounded, or diverging to infinity).
In order to use the above estimator, the choice of the function classes $\mathcal{F}_w$ and $\mathcal{F}_Q$ seems difficult in terms of computation of the weights.
For nonparametric modeling,
a natural idea is to take \(\mathcal{F}_Q\) as a RKHS, since this often leads to a finite-dimensional optimization in regression problems via a representer theorem \citep{Wahba90}.
However, the term \(\EE_{S_0 \sim \mathbb G} \{ \sum_{a\in \calA} \pi(a \mid S_0)f(a,S_0) \} \) could make the representer theorem inapplicable
 when \(S\) is continuous,
 and so \eqref{eqn:balancing_1} becomes an impractical infinite-dimensional optimization.
Another way is to take a finite dimensional space to approximate \(\mathcal{F}_Q\) and \(\mathcal{F}_w\). The dimensions of the approximation spaces would need to increase with \(n\) so as to avoid an asymptotic bias.
Moreover, the existing results on the convergence rate of such estimator to the policy value is not optimal (i.e., slower than $\sqrt{n}$ or $\sqrt{nT}$).
See Corollary 11 in \cite{uehara2021finite}.
As for our weighted estimator, an optimal convergence rate with statistical efficiency can be achieved without an additional nuisance parameter estimation,
even in a more general dependent setting when $T>1$.
See Theorem \ref{thm:asymp} in Section \ref{sec:theory}.

\section{Theoretical Results}
\label{sec:theory}

In this section, we study the theoretical properties of the approximate projection \(\hat g^\pi\) in \eqref{eqn:approx_g}, balancing weights \(\hat \omega^\pi\) in \eqref{eqn:optim2} and the final weighted estimator \(\hat \calV(\pi)\) in Section \ref{sec:projected}.
Specifically, in Section \ref{sec:thm_projection}, we derive the finite-sample error bound for the approximate projection.
In Section \ref{sec:weights}, we study the convergence rate of the proposed balancing weights.
Finally, we show that the proposed weighted estimator is statistically efficient under additional conditions specified in Section \ref{sec:efficiency}. In Section \ref{sec:mineigen}, we study the difficulty of the offline RL in a conservative manner.
To start with, we introduce some notations.  
We define the squared empirical norm as \(\|f(\cdot,\cdot )\|^2_{\total} = (nT)^{-1}\sum_{i=1}^n\sum_{t=0}^{T-1} f^2(S_{i,t}, A_{i,t})\). The notation $\varpi(\total) \leqconst \theta(\total)$ (resp.~$\varpi(\total) \geqconst \theta(\total)$) means that there exists a sufficiently large constant (resp. small) constant $c_1>0$ (resp.~$c_2>0$) such that $\varpi(\total) \geq c_1 \theta(\total)$ (resp. $\varpi(\total) \leq c_2 \theta(\total)$) for some sequences $\theta(\total)$ and $\varpi(\total)$.  Also, we denote by \(\mathcal{N}(\epsilon, \Fscr, \|\cdot\|)\) the $\epsilon$-covering number of \(\Fscr\) with respect to some metric \(\|\cdot\|\).
We take \(d\) as the dimension of a state vector \(S\). 
\subsection{Non-parametric Regressions with Exponentially $\beta$-mixing Sequences}
\label{sec:thm_projection}
Recall that the constraint \eqref{eqn:cons} is merely a surrogate of the desired state-action balancing condition, due to the use of $\hat{g}^\pi(\cdot,\cdot, B_k)$, $k=1,\dots, K$.
To bound the surrogate error on the estimations of the weights and the policy value,
we study the uniform convergence of the approximate projection.
More generally, in Theorem \ref{lem:approx_g} below, we obtain the error bound of \(\hat{g}^\pi(\cdot,\cdot;B)\) \textit{uniformly} over $B \in \calQ$, where $\mathcal{Q}$ is a class of functions defined on $\calS \times \calA$ of interest.
The bound scales with \textit{both} $n$ and $T$.
Later when we adopt Theorem \ref{lem:approx_g}, $\mathcal{Q}$ will be taken as a subset of the linear span of $\{B_1,\dots,B_K\}$. See Corollary \ref{cor:KRR} for more details.
Theorem \ref{lem:approx_g} requires the following assumption. Let \(\|\cdot\|_{\mathrm{TV}}\) denote the total variation norm,

\begin{assumption}
	The following conditions hold.
	\label{assum:approx_g}
	\begin{enumerate}
		\item   
	The Markov chain \(\{S_{t}, A_{t}\}_{t\ge 0}\) has a unique stationary distribution \(\mathbb{G}^*\) with density \(p^*\), and is geometrically ergodic, i.e., there exists a function \(\phi(s,a)\) and constant \(\kappa\in (0,1)\) such that,
		 for any $s\in\mathcal{S}$ and $a\in\mathcal{A}$,
		 \[
			 \left\|\mathbb{G}^b_t(\cdot \mid (s,a)) - \mathbb{G}^*(\cdot)\right\|_{\mathrm{TV}} \leq \phi(s,a) \kappa^t, \qquad\forall t \ge 0,
			\] 
		where \(\mathbb{G}_t^b(\cdot \mid (s,a))\) is the behavior policy induced conditioinal distribution of $(S_t, A_t)$ given $S_0=s$ and $A_0 = a$.
		Also, there exists a constant $\newu\ltxlabel{C_s}>0$ such that \(\int \phi(s,a) d\mathbb{G}_0(s,a) \leq \oldu{C_s}\), where we recall that \(\mathbb{G}_0\) is the initial distribution in the batch data.
		\item 
		The function class \(\mathcal{Q}\) 
		satisfies that \(\|B\|_\infty \leq Q_{\max}\) for all \(B \in \mathcal{Q}\).
		\item The function class $\mathcal{G}$ in \eqref{eqn:approx_g} satisfies that 
		\(\|g\|_\infty \leq G_{\max}\) for all $g\in\mathcal{G}$,
		and \(g^\pi_*(\cdot,\cdot; B) \in \mathcal{G}\) for all \(B \in \mathcal{Q}\).
		\item The regularization functional \(J_{\mathcal{G}}\) in \eqref{eqn:approx_g} 
		is a pseudo norm. 
		   Also, \(J_{\mathcal{Q}}\) (chosen in Theorem \ref{lem:approx_g})
		   is a pseudo norm. 
		Let \(\mathcal{Q}_M = \{B: B\in \mathcal{Q}, J_{\mathcal{Q}}(B)\leq M\}\) and \(\mathcal{G}_M = \{g: g\in \mathcal{G}, J_{\mathcal{G}}(g)\leq M\}\). There exist constants  \(\newu\ltxlabel{entropy}>0\) and \(\alpha \in (0,1)\), such that for any \(\epsilon, M >0\),
		\[\max \left\{ \log \mathcal{N}(\epsilon, \mathcal{G}_M, \|\cdot\|_\infty), \log \mathcal{N}(\epsilon, \mathcal{Q}_M, \|\cdot\|_\infty) \right\} \leq \oldu{entropy} (M/\epsilon)^{2\alpha}.\] 
	\end{enumerate}
\end{assumption}

Assumption \ref{assum:approx_g}(a) basically assumes the stationary distribution exists and the mixing rate is exponentially fast. This allows the convergence rate to scale well with \(T\).  By a truncation argument (see the proof of Theorem \ref{lem:approx_g}), we can show that our estimation is almost equivalent to the nonparametric regression based on stationary and exponentially $\beta$-mixing sequences.  It is known that if a Markov chain is aperiodic and satisfies some drift condition in terms of a well-behaved non-negative measurable function, then it is geometrically ergodic, see \cite{chan1989note} and Chapter 15 of \cite{meyn1995markov} for a detailed characterization. 
The boundedness assumptions of $\mathcal{Q}$ and $\mathcal{G}$ in Assumptions \ref{assum:approx_g}(b) and \ref{assum:approx_g}(c) are used to simplify the proof and can  be relaxed by a careful truncation argument.  Similar assumptions are also adopted in \cite{liao2020off}. Assumption \ref{assum:approx_g}(d) specifies the complexity of the spaces. These entropy assumptions are satisfied for common functional classes, such as RKHS and Sobolev spaces \citep{hearst1998support,geer2000empirical,gyorfi2006distribution}.  Take the Sobolev spaces as an example, \(\alpha = d/(2q)\), where \(q\) is the number of continuous derivatives possessed by the functions in the corresponding space.
With Assumption \ref{assum:approx_g}, we obtain the following uniform finite-sample error bound.

\begin{theorem}
	\label{lem:approx_g}
	Suppose Assumption \ref{assum:approx_g} holds. For any $0 < \tau \leq 1/3$ and all sufficiently large $\Total$ {(i.e., either \(n\) or \(T\) is large enough)},   with probability at least \(1-2\delta - 1/(\Total)\), the following inequality holds for all \(B \in \mathcal{Q}\):
	\begin{multline}\label{eqn:approx_g_bound}
{\EE \left\{ \frac{1}{T}\sum_{t=0}^{T-1} \left[ \hat g^\pi(S_t, A_t; B)- \g (S_t, A_t; B) \right]^2  \right\}}
	 +   \norm{\hat g^\pi(\cdot, \cdot; B) - \g (\cdot, \cdot; B)  }_{\total}^2 + \mu J_{\mathcal{G}}^2 \{\hat g^\pi(\cdot,\cdot; B)\}\\
	 \lesssim \mu\{1 + J_{\mathcal{G}}^2(g^\pi_*(\cdot,\cdot; B)) + J_{\mathcal{Q}}^2(B)\} + \frac{1}{\Total \mu^{\frac{\alpha}{1- \tau(2+\alpha)}} }  + \frac{[\log(\max\{1/\delta, \Total\})]^{1/\tau}}{\Total},
	\end{multline}
	where the leading constant in the above inequality depends on \(Q_{\max}\), \(G_{\max}\), \(\kappa\), and \(\oldu{C_s}\).

	Further suppose \(J_{\mathcal{G}}^2(g^\pi_*(\cdot,\cdot; B))\) and \(J_{\mathcal{Q}}^2(B)\) are bounded. If we take \(\delta = (\Total)^{-1}\),  \(\tau = (1+\alpha)\log(\log (\Total))/(\alpha\log (\Total))\) and \(\mu \asymp  (\Total)^{-1/(1+\alpha)} (\log (\Total)) ^{(2+\alpha)/(1+\alpha)}\) 
	, then for a sufficiently large $\Total$, the following inequality holds for all \(B \in \mathcal{Q}\) with the probability at least \(1-1/(\Total)\):
	$$\EE \{\hat g^\pi (\cdot, \cdot; B) - \g (\cdot, \cdot; B)  \}^2 +  \norm{\hat g^\pi(\cdot, \cdot; B) - \g (\cdot, \cdot; B)  }_{\total}^2 \lesssim (\Total)^{-\frac{1}{1+\alpha}}(\log (\Total))^{\frac{2+\alpha}{1+\alpha}}.$$
\end{theorem}
Next, we adapt Theorem \ref{lem:approx_g} to our case. 
{More specifically, we consider the setting when 
	\(\Gscr\) in Assumption \ref{assum:approx_g}(c) is  an RKHS.} 
Due to space limitations, we list the corresponding assumption in the Supplementary Material (See Assumption \ref{assum:RKHS} 
in \supp{\ref{sec:thm_projection}}
Then we have the following corollary. 
\begin{corollary}
	\label{cor:KRR}
	Suppose Assumption \ref{assum:RKHS}  
	(in the Supplementary Material) hold, and \(\hat g^\pi(\cdot, \cdot; B_k)\), $k=1,\dots,K$, are defined by \eqref{eqn:approx_g} with the same tuning parameter $\mu$.  
	If \(\mu \asymp  (\Total)^{-1/(1+\alpha)} (\log (\Total)) ^{(2+\alpha)/(1+\alpha)}\), then
	for any \({\bm \upsilon = (\upsilon_k)_{k=1}^K }\in \mathbb{R}^K\) such that  \(\sum_{k=1}^K \upsilon_k B_k \in \mathcal{Q}\), $J_{\mathcal{Q}}(\sum_{k=1}^K \upsilon_k B_k) < \infty$ 
	and  \( g^\pi (\cdot, \cdot;\sum_{k=1}^K \upsilon_k B_k) \in \mathcal{G}\) (where $K$ is allowed to grow with \(\Total\)) and a sufficiently large $\Total$, the following inequality holds with the probability at least \(1-1/(\Total)\):
	\begin{multline*}
	\EE \left\{ \sum_{k=1}^K {\upsilon_k} \hat g^\pi (\cdot, \cdot;  B_k) - \sum_{k=1}^K {\upsilon_k} \g (\cdot, \cdot;  B_k) \right\}^2  + \left\|\sum_{k=1}^K {\upsilon_k} \hat g^\pi(\cdot, \cdot; B_k) -\sum_{k=1}^K {\upsilon_k} \g (\cdot, \cdot;B_k) \right\|_{\total}^2 \\
	\lesssim  (\Total)^{-\frac{1}{1+\alpha}}(\log (\Total))^{\frac{2+\alpha}{1+\alpha}}.
	\end{multline*}
\end{corollary}

The proofs of Theorem \ref{lem:approx_g} and Corollary \ref{cor:KRR} can be found in 
\supp{\ref{sec:proof_approx_g}}.
From Theorem \ref{lem:approx_g}, by carefully choosing \(\mu\) and \(\tau\), we can achieve the optimal nonparametric convergence rate that holds uniformly for all \(B \in \mathcal{Q}\) up to a logarithmic factor, compared to the i.i.d. setting. More importantly, Theorem \ref{lem:approx_g} does not require  the initial distribution to be the stationary distribution, which can be unrealistic in practice but is often required in most existing results such as \cite{farahmand2012regularized}. Thus our result is broadly applicable. 
Accordingly, Corollary \ref{cor:KRR} provides a tight bound (in the scaling of both $n$ and $T$) on the proposed approximate projection step \eqref{eqn:approx_g}, which leads to an accurate estimation of the target state-action balancing condition for the construction of the proposed weights. 

\subsection{Convergence Rates of Balancing Weights}
\label{sec:weights}

With Theorem \ref{lem:approx_g} and Corollary \ref{cor:KRR}, we now derive the convergence of the proposed weights $\hat \omega^\pi$.
Define
\[\zeta_{\total} = \left( (\Total)^{-\frac{1}{1+\alpha}}(\log (\Total))^{\frac{2+\alpha}{1+\alpha}} \right)^{1/2}.\]
For any square and symmetric matrix $\bm{A}$, $\lambda_{\max}(\bm{A})$ and $\lambda_{\min}(\bm{A})$ represents the maximum and minimum eigenvalues respectively. And we use \(\|\cdot\|_2\) 	to denote the Euclidean norm of a vector. 
We will need the following assumption.
\begin{assumption}
	\label{assum:weights}
	The following conditions hold.
	\begin{enumerate}
		\item  There exist \(r_1 >1/2\) and \(\blam^*\in \mathbb{R}^K\) such that the true weight function 
		\(\omega^\pi(s,a)\) satisfies 
		\(\sup_{s,a} |\omega^\pi(s,a) - \rho'\{\bm L_K(s,a)^\tp \blam^*\}| \leq \newu\ltxlabel{weightsbound} K^{-r_1}\), where \(\oldu{weightsbound}>0\) is a constant. Also, there exists a constant \(\newu\ltxlabel{weightsC1} > 0\) such that \( \sup_{s,a}|\omega^\pi(s,a)| \leq  \oldu{weightsC1}\). 
		\item  The second derivative of $\rho$ defined in Theorem \ref{thm:dual}, i.e., 
		\(\rho''\), is a positive and  continuous function and  \( \rho'' \ge \newu\ltxlabel{weightssecond}\) for some constant \(\oldu{weightssecond}>0\). 
		\item There exist constants \(\newu\ltxlabel{weightsC2}>0\) and \(\newu\ltxlabel{weightsC4}>0\) such that
		\(
		\sup_{s,a}\|\bm B_K(s,a)\|_2 \leq \oldu{weightsC2} K^{1/2},\)
		and for every $T\geq 0$,
		{
		\[\left\{  
			\begin{array}{l}
					\max\left\{	\lambda_{\max}\left\{ \EE\left[\frac{1}{T}\sum_{t=0}^{T-1}\bm L_K(S_t,A_t) \bm L_K(S_t,A_t)^\tp\right] \right\}, \EE_{(S,A) \sim \mathbb{G}^*} \left[ \frac{\bar{p}^b_T(S,A)}{p^*(S,A)} \right] ^2\right\}	\leq \oldu{weightsC4}\\
			\lambda_{\max}\left\{ \EE_{(S,A) \sim \mathbb{G^*}}\left[\bm L_K(S,A) \bm L_K(S,A)^\tp\right] \right\} \leq \oldu{weightsC4}.
			\end{array}
		\right.
			\]}
		\item 
		There exists a quantity \(\psi(K) > 0\) depending on \(K\) such that
		\[
			\lambda_{\min}\left\{\EE\left[\frac{1}{T}\sum_{t=0}^{T-1}\bm L_K(S_t,A_t) \bm L_K(S_t,A_t)^\tp\right]\right\}\ge \psi(K).
		\]
		\item 
		\(\sum_{k=1}^K \blam^\ast_{k} B_k \in \mathcal{Q}\)
		, \(g^\pi_*(\cdot,\cdot; \sum_{k=1}^K \blam^\ast_{k} B_k) \in \mathcal{G}\), where \(\mathcal{Q}\) and \(\mathcal{G}\) are function classes considered in Corollary \ref{cor:KRR} (also see 
		\supp{\ref{sec:proof_approx_g}}). 
		In addition, the same tuning parameter \(\mu \asymp (\Total)^{-1/(1+\alpha)} (\log (\Total)) ^{(2+\alpha)/(1+\alpha)}\) is adopted when estimating \( g^\pi_* (\cdot, \cdot;  B_k)\), \(k = 1,\dots,K\). 
		For any vector \(\bm a\) that satisfies \(\|\sum_{k=1}^K \alpha_k B_k\|_\infty \leq Q_{\max} \), we have  \(\sum_{k=1}^K \alpha_k B_k \in \mathcal{Q}\) and \(g(\cdot,\cdot; \sum_{k=1}^K \alpha_k B_k) \in \mathcal{G}\).\\
		\(K  = \smallO(\zeta_{\total}^{-2})\).
		\item  $\|\bm \delta_K\|_2 = \bigO\left[ \{\psi(K)\}^{-1}\left(  {\sqrt{K}\log (\Total)}/{\sqrt {\Total}} + K^{-r_1}+ \sqrt{K}\zeta_{\total}\right) \right]$.
	\end{enumerate}
\end{assumption}
 Assumption \ref{assum:weights}(a) specifies the requirement on the uniform approximation error of the true weight function \(\omega^\pi\).
 The boundedness of \(\omega^\pi\) can be guaranteed if the average visitation density $\bar p_T^b$ is bounded away from 0 and the policy-induced discounted visitation probability $d^\pi$ is bounded above.
 This overlapping condition is commonly assumed in the literature of the ATE  and RL \citep[e.g.][]{Wong-Chan18,uehara2020minimax,uehara2021finite}.
 For Assumption \ref{assum:weights}(b) on $\rho''$,  by the relationship between $\rho$ and $h$ detailed in Theorem \ref{thm:dual}, many convex choices of \(h\) adopted in \eqref{eqn:optim2} will result in a $\rho$ that satisfies this condition,
   e.g., \(h(x) = (x-1)^2\).
 For Assumption \ref{assum:weights}(c), the uniform bound for \(\|\bm B_K\|_2\) is a mild technical condition on the basis functions \(B_k\), \(k = 1,\dots,K\). It is satisfied by many classes of basis functions, including the regression spline, the trigonometric polynomial, and the wavelet basis \citep{newey1997convergence, horowitz2004nonparametric, chen2007large, belloni2015some}. See the same assumption in \cite{fan2016improving} (Condition 6 of their Assumption 4.1). 
 As for the largest eigenvalue condition, we verify that such condition holds for common bases such as tensor-product B-spline basis and tensor-product wavelet basis in Lemma \ref{lem:eigenvalues} 
of the Supplementary Material.  
 As for Assumption \ref{assum:weights}(d), \(\psi(K)\) is allowed to depend on \(K\) in our analysis as opposed to the existing study such as \cite{shi2020statistical}. 
 \revise{While we allow \(\psi(K)\) to diminish as \(K\) grows, we later show that \(\psi(K)\) can be  strictly bounded below if we are willing to make  further relatively minor assumptions. See detailed discussion in Section \ref{sec:mineigen}.}
 Assumption \ref{assum:weights}(e) is a mild condition 
 requiring the best approximation $\sum_{k=1}^K \blam^\ast_{k} B_k$ in
	  \(\mathcal{Q}\), and its projection in \(\mathcal{G}\). 
Now we can show the convergence of the proposed balancing weights.

\begin{theorem}
	\label{thm:weights}
	Suppose Assumptions  \ref{assum:RKHS} 
	(in the Supplementary Material) and \ref{assum:weights} hold. If we further assume that \(\{\psi(K)\}^{-1}\sqrt{K}\zeta_{\total} = \smallO(1) \) as $\Total \rightarrow +\infty$,  then 
		\begin{multline*}
			\max\left\{ \left\|\hat{\omega}^\pi -\omega^\pi\right\|_{\total}, 
			{\EE \left\{ \frac{1}{T}\sum_{t=0}^{T-1} \left[ \hat \omega^\pi(S_t, A_t)- \omega^\pi (S_t, A_t) \right]^2  \right\}^{1/2}}
			\right\}= \\
				\bigOp\left[ \{\psi(K)\}^{-1}\left(  \frac{\sqrt{K}\log (\Total)}{\sqrt {\Total}} + K^{-r_1}+ \sqrt{K}\zeta_{\total}\right) \right] = \smallOp(1).
			\end{multline*}

\end{theorem}
The proof of Theorem \ref{thm:weights} can be found in 
\supp{\ref{sec:proof_weights}}.  
Theorem \ref{thm:weights} gives the convergence rate of the proposed balancing weights in terms of both the empirical  $\fltwo$
 and the $\fltwo$ metrics. Note that the rate implies that as long as either $n$ or $T$ goes to infinity,
the proposed weight estimates converge to the true ratio functions. To the best of our knowledge, this is the first consistent result of the ratio function in the scaling of both $n$ and $T$ under the discounted infinite-horizon setting.

\subsection{Estimation Error and Statistical Efficiency}
\label{sec:efficiency}
Based on Theorems \ref{lem:approx_g} and \ref{thm:weights},
we can derive an error bound for \(\hat \calV(\pi)\). 

	\begin{assumption}
		\label{assum:additional}
		Suppose the following conditions hold.
		\begin{enumerate}
			\item  There exists a \(\bm \beta \in \mathbb{R}^K\)  such that
			\(
				\sup_{s,a} |\Delta_Q(s,a)| = \sup_{s,a}|Q^\pi(s,a) - \bm B_K(s,a)^\tp \bm \beta| \leq \newu\ltxlabel{Qconst1} K^{-r_2}
				\)
				 for some constants \(r_2 > 1/2\) and \(\oldu{Qconst1}>0\).
				 Also, \(\sum_{k=1}^K  \bm\beta_kB_k \in \mathcal{Q}\),  \(g^*(\cdot, \cdot; \sum_{k=1}^K  \bm\beta_kB_k) \in \mathcal{G}\), where \(\mathcal{Q}\) and \(\mathcal{G}\) are function classes considered in Corollary \ref{cor:KRR}. (Also see \supp{\ref{sec:proof_approx_g}}). 
			\item The errors \(\epsilon_{i,t} := R_{i, t} - \EE(R_{i,t} \mid S_{i,t}, A_{i,t})\), \(i = 1,\dots, n, t = 0, \dots, T-1\), are independent mean-zero subgaussian random variables. 
		\end{enumerate}
	\end{assumption}

	Assumption \ref{assum:additional}(a) is a regularity condition for \(Q^\pi\). 
	It is satisfied by letting \(r_2 = q_Q/d\), where \(q_{Q}\) is the maximum number of continuous derivatives for \(Q^\pi\) with respect to \(S\) among all action levels, and choosing $B_K$ as basis functions such as splines and power series if \(Q^\pi\) is defined over a compact domain.
	Assumption \ref{assum:additional}(b) is a mild condition for the error of the reward. 
	{In fact, this assumption can be relaxed to allow errors themselves to be dependent. 
	Note that the estimated weight function \(\hat \omega^\pi\) only depends on \(\{S_{i,t}, A_{i,t}: i = 1,\dots,n, t = 0,\dots, T-1\}\), which is  independent of \(\{\epsilon_{i,t}: i = 1,\dots,n, t = 0,\dots, T-1 \}\).
	The proof of the convergence of \(\hat {\mathcal{V}}(\pi)\) stated in the following theorems is based on a conditioning argument that separates the effects of the weights and the errors.
	Standard techniques that deal with weighted averages of dependent random variables can be adopted to extend the current results (e.g., Theorem \ref{thm:upper_bound}) to those dependent settings.
	For instance, If there are some weak autocorrelations among $\epsilon_{i,t}$, we are still able to obtain results in Theorems \ref{thm:upper_bound} and \ref{thm:asymp}. 
	} 
\begin{theorem}
	\label{thm:upper_bound}
	Suppose Assumptions  \ref{assum:RKHS} 
	(in the Supplementary Material), \ref{assum:weights} and \ref{assum:additional} and hold. Furthermore, assume \(\{\psi(K)\}^{-1}\sqrt{K}\zeta_{\total} = \smallO(1) \). 
	Then we have 
	\begin{multline*}
	\left| \hat \calV(\pi) - \calV(\pi)\right| = \bigOp\left\{ {\|\bm \beta \circ \bm \delta_{K}\|_1}  
	 \right.\\
	\left. +  \{\psi(K)\}^{-1}\left( K^{-r_1} \zeta_{\total} + K^{1/2-r_2}\zeta_{\total} + K^{-r_1-r_2} + \sqrt{K}\zeta^2_{\total}  \right)+ \zeta_{\total}\right\} \nonumber,
	\end{multline*}
where $\circ $ refers to the element-wise product between two vectors.
\end{theorem}

Theorem \ref{thm:upper_bound} can be proved by the similar arguments in the proof of Theorem \ref{thm:asymp}.
To obtain the sharp convergence rate, based on this theorem,
one need to tune $K$ accordingly.
As such, the eigenvalue bound $\psi(K)$ specified in Assumption \ref{assum:weights}(d) becomes crucial. 
\revise{Indeed, we can show that \(\psi(K)\) is strictly lower bounded by a positive constant independent of \(K\) under some mild conditions. We defer the detailed discussion about characterizing \(\psi(K)\) in Section \ref{sec:mineigen}. 
In the following, we establish the bounds when $\psi(K)$ is bounded below. }

\begin{theorem}
	\label{thm:asymp}
	Suppose Assumptions  \ref{assum:RKHS} 
	(in the Supplementary Material), \ref{assum:weights}(a)--(g) and \ref{assum:additional} hold. Also, assume that \(\psi(K) \ge \newu\ltxlabel{weightsC3}\) for some constant \(\oldu{weightsC3}>0\). %
	\begin{enumerate}[label = (\roman*)]
		\item If \(
{\|\bm \beta \circ \bm \delta_{K}\|_1} = \bigOp(\zeta_{\total})\) and \(K^{-1} = \bigOp(\zeta_{\total}^{1/(r_1 + r_2)})\) ,  we have 
        \begin{align}
			\label{eqn:upper_bound2}
		\left| \hat \calV(\pi) - \calV(\pi)\right|= \bigOp(\zeta_{\total}).
		\end{align}
		In particular, we can take \(K \asymp \xi_{\total}^{-a}\) for any $a$ such that \(1/(r_1 + r_2) < a< 2 \).
		\item Assume \(
	{\|\bm \beta\circ \bm \delta_{K}\|_1}= \bigOp((\Total)^{-1/2})\), \(r_1 \ge 1\), \(r_2 >1\), \(K = \smallO((\Total)^{-1} \xi_{\total}^{-4})\),  \(K^{-1} = \smallO( (\Total)^{-1/(2r_1)}\zeta_{\total}^{-1/r_1})\),   \(K^{-1} = \smallO( (\Total)^{-1/(2r_2 - 1)}\zeta_{\total}^{-1/(r_2-1/2)})\), \(K^{-r_1-r_2} = \smallOp((\Total)^{-1/2})\), and \(\omega^\pi \in \mathcal{G}\).
		Take \[\sigma^2 = \frac{1}{T} \sum_{t=0}^{T-1}\EE \left\{\omega^\pi(S_{t},A_{t})\left(R_t + \gamma\sum_{a' \in \calA}\pi(a' \mid S_{t+1}) Q^\pi(S_{t+1}, a') - Q^\pi(S_t, A_t)\right)\right\}^2.\]
		Then, as either \(n\rightarrow \infty\) or \(T \rightarrow \infty\), we have
		\begin{align}
			\label{eqn:sigma2}
			\frac{\sqrt{nT}}{\sigma} \left\{ \hat \calV(\pi) - \calV(\pi)\right\}  \xrightarrow[]{d} N(0,1). 
		\end{align}
		In particular, if \(r_1 = 1\), then we can take \(K \asymp (\Total)^{a}\) for any \(a\) such that \(1/(6r_2 - 3) < a < 1/3\) so that the above results hold. 
	\end{enumerate}

\end{theorem}
The proof of this theorem can be found in 
\supp{\ref{sec:proof_asymp}}.
Note that Theorem \ref{thm:asymp} requires \(\psi(K)\) to be lower bounded by a positive constant, which can be satisfied under the conditions in \revise{Assumption \ref{assum:mineign}.} 
The constraints for \(\bm \delta_K\) in Theorem \ref{thm:asymp}(i) and (ii) are stronger than that in Assumption \ref{assum:weights}, which lead to the final desired \(\zeta_{\total}\)-consistency and \(\sqrt{\Total}\)-consistency of our weighted estimator in cases (i) and (ii) respectively. 
Compared with case (i), the constraints for  \(K\) in case (ii) are more restrictive so that the bias is asymptotically negligible and thus \(\hat \calV(\pi)\) is $\sqrt{\Total}$-consistent. When \(r_1 \ge 1\) and \(r_2 >1\), the existence of \(K\) is guaranteed. The additional assumption of \(\omega^\pi\), i.e., $\omega^\pi \in \calG$, is a mild assumption, which allows the bias term to diminish asymptotically.

If $T$ is fixed, one can show that $\sigma^2/T$ is indeed the semi-parametric efficiency bound in the standard i.i.d. setting. When both $n$ and $T$ are allowed to go to infinity,  $\sigma^2$ becomes
$
\EE^* \{\omega^\pi(S,A)(R_t + \gamma\sum_{a' \in \calA}\pi(a' \mid S') Q^\pi(S', a') - Q^\pi(S, A))\}^2,
$
where $\EE^*$
denotes the expectation with respect to the stationary measure induced by the behavior policy, i.e., $\mathbb{G}^*$. As shown in \cite{kallus2019efficiently}, $\sigma^2$ is also the statistical efficiency bound in the notion of \cite{komunjer2010semiparametric}. Note that our results do not require that the data come from stationary distribution, which is however needed in \cite{kallus2019efficiently}. 

{Finally, we remark that to the best of our knowledge, only two prior works establish the convergence rates of the policy value estimators under some non-parametric models in the scaling of both \(n\) and \(T\). 
One is \cite{shi2020statistical}, 
which directly estimates \(Q\)-function.
The underlying analysis does not require the stationary assumption for the data generating process.
However, they did not show that their estimator can achieve the statistical efficiency. In addition, their conditions require that the initial distribution of the data is bounded away from zero, which we do not require.
The other one is \cite{kallus2019efficiently}.
This work shows the convergence of their estimator when \(T \rightarrow \infty\). In their theoretical results, stationary assumption is needed. In order to obtain the efficiency, this work requires the adoption of cross-fitting (sample splitting) for the Q-function and ratio (weight) function estimations. However, it is unclear how to perform an efficient cross-fitting due to the existence of temporal dependence among the batch data.}

\subsection{\revise{Lower boundedness for the minimal eigenvalue}}
\label{sec:mineigen}

{In this section, we discuss \(\psi(K)\) defined in Assumption \ref{assum:weights}(d) in details. For the notational simplicity, we take \(\overline \pr\) and \(\bar \EE\) as the probability and expectation  with respect to the average visitation distribution \(\bar{p}_T^b\).
We also write  \[d^\pi(s',a'\mid s, a) = \sum_{t=0}^\infty \gamma^t p_t^\pi(s',a'\mid S_0=s, A_0 = a)\]  as the conditional discounted visitation probability.  Define the operator \(\mathcal{P}^\pi : \mathcal{S} \times \mathcal{A} \rightarrow \mathcal{S} \times \mathcal{A}\) by 
\[(\mathcal{P}^\pi f)(s,a) = \bar{\EE} \left[ \sum_{a'\in \mathcal{A}} f(S',a') \pi(a'\mid S')  \mid S = s, A= a  \right]\]
for any function \(f: \mathcal{S} \times \mathcal{A} \rightarrow \mathbb{R} \), and denote \(I: \R^{\mathcal{S} \times \mathcal{A}} \rightarrow \R^{\mathcal{S} \times \mathcal{A}}\) as the identity operator.
}

{In Section C.1 of \cite{shi2020statistical}, a sufficient condition for the lower boundedness of \(\psi(K)\) is provided. They argue that, under some boundedness conditions on the average visitation probability, the minimal eigenvalue is lower bounded by a constant independent of \(K\) as long as $\gamma$ is small enough.
(In fact, we show that this sufficient condition can be further relaxed; 
see Corollary \ref{cor:mineigen} for details.) However, in practice, we may not know the distance between the target policy and the behavior policy in advance, or be able to choose a reasonably small $\gamma$ that reflects the desired emphasis of the long-term rewards.
More importantly, choosing $\gamma$ close to $1$ is often preferred in many applications as we discussed before.
Despite its importance, the theoretical property of
non-parametric OPE under this setting is largely uncharted territory
in the current state of the literature.
As a result, it is important to understand the behavior of $\psi(K)$ for \textit{any} $0 \leq \gamma < 1$ and target policy $\pi$.

In the following, we focus on the general operator \(I - \gamma \mathcal{P}^\pi\) and study the (squared) minimal eigenvalue of it:
\begin{align}
	\label{eqn:Upsilon}
	\Upsilon : = \inf_{\{f:\bar{\EE} f^2(S,A) \ge 1\}}	\left[{\bar{\EE}\left\{ (I-\gamma\mathcal{P}^\pi)f(S,A)  \right\}^2}\right].
\end{align}
To see the relationship between \(\Upsilon\) and \(\psi(K)\), we can take \(f(s,a) = \bm B_K^\tp(s,a)\bm \alpha\) for some \(\bm \alpha \in \mathbb{R}^K\) in \eqref{eqn:Upsilon}. Then we have
\begin{align*}
	\psi(K) & = \inf_{\| \bm \alpha\|_2 = 1}	\left[{\bar{\EE}\left\{ (I-\gamma\mathcal{P}^\pi)\bm B_K^\tp (S,A) \bm \alpha  \right\}^2}\right] \\
	& = \inf_{\| \bm \alpha\|_2 = 1} \frac{\bar{\EE}\left\{ (I-\gamma\mathcal{P}^\pi)\bm B_K^\tp (S,A) \bm \alpha  \right\}^2}{\bar{\EE}\left\{ \bm B_K^\tp (S,A)\bm \alpha \right\}^2} {\bar{\EE}\left\{ \bm B_K^\tp (S,A)\bm \alpha \right\}^2}. \\
	& \ge \Upsilon \lambda_{\min} \left\{ \bar{\EE} \bm B_K (S,A)  \bm B^\tp_K (S,A)\right\}.
\end{align*}
In Theorem \ref{thm:mineign}, we propose a necessary and sufficient condition for bounding the minimal eigenvalue of the operator \(I - \gamma \mathcal{P}^{\pi}\) that works for any \(\pi\) and \(\gamma\). When combined with the standard eigenvalue condition on the basis $\bm{B}_K$, it allows a comprehensive characterization of the lower boundedness of $\psi(K)$.}

\begin{theorem}
	\label{thm:mineign}
	\(\Upsilon\) 
	is lower bounded by a positive constant if and only if 
	\[\chi := \sup_{\{f: \bar{\EE} f^2(S,A) \le 1\}} {\bar{\EE} \left[ \EE_{(S',A') \sim d^\pi(\cdot,\cdot \mid S, A)}\left\{ f(S',A') \right\} \right]^2} < + \infty. \]
	In this case, $\Upsilon = (1-\gamma)^2/\chi$.

\end{theorem}

Theorem \ref{thm:mineign} allows us to avoid directly analyzing the minimal eigenvalue of \(I-\mathcal{P}^\pi\), but to study the upper bound of \(\chi\) instead, which is much easier  to deal with.  We can view \(\chi\) as a criterion to examine the difficulty of OPE problems. When \(\chi\) gets larger, the OPE problem becomes more difficult. The value of \(\chi\) depends on many components, including the target policy \(\pi\), the behavior policy \(\pi_b\), the discount factor \(\gamma\) and the horizon of the observed data \(T\).  Next, we focus on the term \(\chi\) and provide some sufficient conditions for it to be upper bounded.  
Take 
\[
	{\bar{d}}^\pi (s',a') =\bar{\EE}\left[  d^\pi(s', a' \mid S, A) \right]. 
\]

\begin{assumption}
	\label{assum:mineign}
	Suppose the following conditions hold.
	\begin{enumerate}
		\item The average visitation probability (density) \(\bar{p}_{T}^b(s,a)\) 
		is lower bounded by a constant \(p_{\min}>0\) and upper bounded by a constant \(p_{\max,1}\).
		\item The transition probability under the target policy \( q(s',a'\mid s,a):=\pi(a'\mid s') p(s'\mid s,a)\)  is upper bounded by some constant \(p_{\max,2}\) for all \((s',a') \in \mathcal{S} \times \mathcal{A}\).
	\end{enumerate}
\end{assumption}

\begin{corollary}
	\label{cor:mineigen}
	If there exists a constant \(\oldu{cnt:denbound}\) such that 
	\begin{align}
		\label{eqn:denbound}
		\varrho : = \sup_{(s,a)\in \mathcal{S}\times \mathcal{A}} \frac{\bar{d}^\pi(s,a)}{\bar{p}_T^b(s,a)} \leq \newu\ltxlabel{cnt:denbound},
	\end{align}
	 then  we have 
	 \[\chi\leq \oldu{cnt:denbound}\quad \mbox{and} \quad \Upsilon \ge \frac{(1-\gamma)^2}{{\oldu{cnt:denbound}}}.\]
	In addition, 
		if we  assume   Assumption \ref{assum:mineign}, then we have
		\[\chi \leq \frac{p_{\max}}{p_{\min}}, \qquad \Upsilon \geq {(1-\gamma)^2} \frac{p_{\min}}{p_{\max}},\] 
	   where \(p_{\max} = \max\{p_{\max,1}, p_{\max,2}\}\).
Finally, \(\psi(K) \gtrsim {(1-\gamma)^2} \frac{p_{\min}}{p_{\max}}  \) under an additional condition that 
\begin{align}
	\label{eqn:basismineigen}
	\lambda_{\min}\left\{ \bar{\EE} \left[ \bm B_K(S,A) \bm B_K(S,A)^\tp)  \right]\right\} \gtrsim 1. 	
\end{align}

\end{corollary}

In Corollary \ref{cor:mineigen}, \eqref{eqn:denbound} provides a sufficient condition for \(\chi\) to be upper bounded (or equivalently,  for \(\Upsilon\) to be lower bounded) that works for general settings without any restriction on \(\gamma\) or the target policy \(\pi\). 	
The second part of Corollary \ref{cor:mineigen} provides a specific characterization for \(\varrho\) under Assumption \ref{assum:mineign}, in which we  require   the average visitation probability \(\bar{p}_T^b\) to be lower and upper bounded. This coverage assumption is very common in RL literature \citep[e.g.][]{precup2000eligibility,kallus2019efficiently, shi2020statistical}. 
Compared to \cite{shi2020statistical}, we provide an explicit bound with respect to \(p_{\min}\) and \(p_{\max}\). More importantly,  we do not impose any further assumptions on the target policy \(\pi\) or the discount factor \(\gamma\),  in order to show the lower boundedness of  \(\psi(K)\). In other words, we show that such boundedness holds uniformly for any target policy \(\pi\) and \(\gamma\). We note that there is a parallel work regarding the well-poseness minimax optimal rates of nonparametric Q-function estimation in OPE \citep{chen2022well}. They provide a \textit{sufficient} condition for the lower boundedness of $\Upsilon$.
Their result is similar to the one specified in Corollary \ref{cor:mineigen},
and require similar boundedness conditions as in Assumption \ref{assum:mineign}.
See Assumption 4(a) and Theorem 1 in \cite{chen2022well} for more details.
Compared to their bound:  \[\Upsilon \gtrsim \frac{ (1-\gamma)^2p_{\min}}{p_{\max}\left( 1 + \gamma^2\frac{p_{\max}}{p_{\min}} \right)},\] 
Corollary \ref{cor:mineigen} provides a sharper dependence  with respect to $\gamma$, $p_{\min}$ and $p_{\max}$.
Apart from Corollary \ref{cor:mineigen}, we note that our general result in Theorem \ref{thm:mineign}
does not require any boundedness condition on visitation and transition probability as in Assumption \ref{assum:mineign} and \citep{chen2022well},
and, more importantly, provides a \textit{necessary and sufficient} condition for the boundedness of $\Upsilon$.

Finally, we remark that the additional assumption \eqref{eqn:basismineigen} is satisfied for common bases such as tensor-product B-spline basis and tensor-product wavelet basis. See Lemma \ref{lem:eigenvalues} of the Supplementary Material.

\section{Simulation Study}
\label{sec:sim}
We conduct a simulation study to investigate the finite-sample performance of the proposed estimator.
We adopt the similar simulation settings as in \cite{luckett2019estimating}, \cite{liao2020off} and \cite{shi2020statistical}.
Specifically, the data generative model is given as follows.
The state variables are two-dimensional, i.e., \(S_{i,t} = (S^{(1)}_{i,t}, S^{(2)}_{i,t})\) for $0 \leq t \leq T$ and $1 \leq i \leq n$,
while the action is binary, i.e., \(\calA =  \{0,1\}\).
The initial state follows the standard bivariate normal distribution.
The transition dynamics are given by \(S^{(1)}_{i,t+1} =(3/4) (2A_{i,t}-1)S^{(1)}_{i,t} + \epsilon^{(1)}_{i,t} \), and \(S^{(2)}_{i,t+1} =(3/4) (1-2A_{i,t})S^{(2)}_{i,t} + \epsilon^{(2)}_{i,t} \), where \(\epsilon^{(1)}_{i,t}\) and \(\epsilon^{(2)}_{i,t}\) are independent normal random variables with mean 0 and variance 0.25. The behavior policy independently follows a Bernoulli distribution with mean \(1/2\). The immediate reward \(R_{i,t}\) is defined as \(R_{i,t} = 2S^{(1)}_{i,t+1} + S^{(2)}_{i,t+1} - (1/4)(2A_{i,t}-1) \).
 We use the initial state distribution as the reference distribution \(\mathbb{G}\) and set \(\gamma\) to $0.9$. 
We evaluate the following four different target policies.
\begin{enumerate}
	\item \(\pi_1(a\mid s) = 1,\, a\in\mathcal{A}, s\in\mathcal{S}.\)
   This is the ``always-treat'' policy used in the simulation study of \cite{liao2020off}, where the chosen action is always 1, and does not depend on the state variable.
	\item
	\[
		\pi_2(a\mid s) = 	
		\begin{cases}
			1 & \text{if $s^{(1)}\leq 0$ and  $s^{(2)}\leq 0$,} \\
			0 & \text{otherwise,}
		  \end{cases},
      \quad a\in\mathcal{A}, s=(s^{(1)}, s^{(2)})\in\mathcal{S}.
	\]
	This policy is a discontinuous function with respect to the state variable. The same type of policy is used in the simulation study of \cite{shi2020statistical}. 
	\item \(\pi_3(a\mid s) = \exp\{-(s^{(1)} + s^{(2)})\}, \, a\in\mathcal{A}, s=(s^{(1)}, s^{(2)})\in\mathcal{S}.\)
	This policy is smooth with respect to the state variable.
	\item \(\pi_4(a\mid s) = 0.5, , a\in\mathcal{A}, s\in\mathcal{S}\). 
	This policy is the same as the behavior one. Note that the observed data only contain finite horizon of decision points, i.e., $T<\infty$,
  while the target here is the policy value under \textit{infinite} horizon.
\end{enumerate}
For each target policy, we consider four different combinations of $n$ and $T$, i.e., \( (n,T) = (40, 25),  (80, 25)\), \((40, 50)\) and \((80, 50)\). The true policy values \(\calV(\pi_k)\), \(k = 1,\dots,4\), were computed approximately by the Monte Carlo method.
Specifically, for every \(\pi_k\), we simulate \(\tilde{n} = 10^5\) independent trajectories of length \(\tilde T = 5000\), with initial states drawn from \(\mathbb{G}\).
Then we approximate \(\calV(\pi_k)\) by \( (1-\gamma){\tilde{n}}^{-1} \sum_{t=0}^{\tilde{T}-1} \gamma^t R_{i,t}\).

Due to space limitations, the implementation details of the proposed method (\proposed{}) are reported in 
\supp{\ref{sec:sim_implement}}. 
For comparison,  { we include the following estimators: (1) \vl{}: the estimator  from \cite{luckett2019estimating}; (2) \ql{}: the estimator from  \cite{shi2020statistical}; (3) \textsc{FQE}: the fitted Q-evaluation estimator developed in \citep*{le2019batch}  where regression problems are solved using random forest models; (4) \textsc{IS}: importance sampling estimator from \citep*{precup2000eligibility}; (5)  \textsc{ MINIMAX}: minimax weight learning method from \citep*{uehara2020minimax}; (6) \textsc{DR}:  double reinforcement learning method considered in \citep*{kallus2019double}. }  
As suggested by \cite{luckett2019estimating}, we implemented \vl{} with Gaussian basis functions since they offer the highest flexibility.
For \ql{}, we use the same basis function \(\bm B_K\) as in the proposed weighted estimator to estimate \(Q^\pi\). Note that, to compute the confidence interval for the proposed weighted estimator, we use the estimate of \(Q^\pi\) from \ql{}. Construction of our confidence interval can be found in Section 
\supp{\ref{sec:sim_implement}}. 
{The implementation details of \textsc{FQE}, \textsc{IS}, \textsc{MINIMAX} and \textsc{DR} can be found in \cite{shi2021deeply}.}

Table \ref{tab:40_50} shows the mean squared error (MSE) and median squared error (MeSE) of the above estimators, as well as the empirical coverage probabilities (ECP) and average length (AL) of their \(95\%\) confidence intervals, over 500 simulated data sets when \((n,T) = (40, 50)\). 
The results for \((n,T) = (40,25)\), \((n,T) = (80,25)\) and \((n,T) = (80,50)\) can be found in 
\supp{\ref{sec:sim_more}}.
 Overall, the proposed estimator (\proposed{}) shows competitive  performances. { Other weighted estimators such as \textsc{IS} and \textsc{MINIMAX} suffer from instability issues and produce inferior results compared with other methods. Also, \textsc{DR} is influenced by the unstable weights from IS and does not perform well in these settings.}

Specifically, for Scenario (a), \vl{} in general has the smallest MSE and MeSE. The performances of \proposed{} and \vl{} are close, while that of \ql{} {and \textsc{FQE}} are worse. For Scenario (b), \ql{} is the best when the sample size is relatively large, while the performance of \proposed{} is close to that of \ql{}. However, when \(n\) and \(T\) are relatively small, \ql{} produces some extreme estimates, as seen from the notable difference between MSE and MeSE. See Tables \ref{tab:40_50} and \ref{tab:40_25} 
 in the Supplementary Material. In contrast, the results of \proposed{} remain stable for small \(n\) and \(T\), and have comparable performance to the best estimator \textsc{FQE} in these settings. As for Scenario (c), \proposed{} and \vl{} perform similarly, and are better than \ql{} {and \textsc{FQE}}. \proposed{} always has the smallest MSE and MeSE when the target policy is \(\pi_4\).
As for ECP and AL, it seems that the confidence intervals of \proposed{} tend to have lower coverage than the target level 95\%, especially when the target policy is \(\pi_2\). 
We hypothesize that this under-coverage phenomenon is due to the regularization of weights in \eqref{eqn:obj2}, which affects the variance estimation in \eqref{eqn:sigma2}. 
Since this is beyond the scope of the paper, we leave it for future study. In practice, we recommend multiplying a factor to the length of our confidence interval (CI) in order to relieve this under-coverage issue. In this simulation study, we choose a constant factor \(1.2\) to obtain adjusted intervals. The ECPs and ALs for adjusted confidence intervals are provided in all tables. In general, our method performs robustly and satisfactorily in terms of the coverage and average length of the confidence interval.
{In the Supplementary Material, we also evaluate the performance of the above methods 
on the Cartpole environment, which  can be obtained from OpenAI Gym \citep{brockman2016openai} and has been frequently considered in the computer science literature.  
Overall, the performance of \proposed{} is appealing, See Section \ref{sec:sim_cartpole} of the Supplementary Material for more details. }

      \begin{table}[h]
        \centering
        \caption{Simulation results for four different target policies when  $n = 40$ and  $T = 50$. The MSE values (with standard errors in parentheses), MeSE values, ECP, and AL (with standard errors in parentheses) are provided for three estimators. For \proposed{}, ECPs and ALs for the adjusted confidence interval are provided after the slashes in the ECP and AL columns.} 
        \label{tab:40_50}
		\resizebox{\textwidth}{!}{%
        \begin{tabular}{c |c |cccc}
          \hline
        Target & methods  & MSE ($\times 1000$) & MeSE ($\times 1000$)  & ECP & AL ($\times 100$)  \\ 
        \hline
        $\pi_1$ & \proposed{} & 8.76 (0.527) & 3.75  & 0.96 / 0.99 & 38.93 (0.258) / 46.71 (0.310) \\ 
   & \vl{} & 7.71 (0.461) & 3.57 & 0.96 & 36.25 (0.159) \\ 
   & \ql{} & 10.2 (0.614) & 4.32  & 0.96 & 42.28 ( 0.932) \\ 
   & {\textsc{FQE}} & 10.99 (0.744) & 4.40  &\_  & \_ \\ 
   & {\textsc{minimax}} & 34.53 (3.184) & 14.04  &\_  & \_ \\ 
   & {\textsc{DR}} & 151.87 (16.670) & 22.84 &\_  & \_ \\ 
   & {\textsc{IS}} & 1040.89 (270.752) & 53.94  &\_  & \_ \\ 
        \hline
        $\pi_2$ & \proposed{} & 4.29 (0.356) & 1.75  & 0.91 / 0.94 & 23.57 (2.130) / 28.28 (2.556) \\ 
   & \vl{} & 6.61 (0.369) & 3.50 & 0.84 & 24.18 (0.142) \\ 
   & \ql{} & 36.90 (33.300) & 1.52 & 0.94 & 439.90 (415.277) \\ 
   & {\textsc{FQE}} & 3.93 (0.242) & 1.87  &\_  & \_ \\ 
   & {\textsc{minimax}} & 120.92 (1.276) & 118.98  &\_  & \_ \\ 
   & {\textsc{DR}} & 92.33 (6.212) & 16.283  &\_  & \_ \\ 
   & {\textsc{IS}} & 1111.542 (100.934) & 327.248  &\_  & \_ \\ 
   \hline
   $\pi_3$ & \proposed{} & 2.41 (0.153) & 1.06 & 0.93 / 0.97 & 18.47 (0.125) / 22.17 (0.150) \\ 
   & \vl{} &  2.45 (0.162) & 1.14 & 0.98 & 23.10 (0.158) \\ 
   & \ql{} & 2.72 (0.173) & 1.18  & 0.94 & 21.43 (1.717) \\ 
   & {\textsc{FQE}} & 4.66 (0.307) & 2.30  &\_  & \_ \\ 
   & {\textsc{minimax}} & 37.11 (1.415) & 31.21 &\_  & \_ \\ 
   & {\textsc{DR}} &49.20 (34.761) & 2.36  &\_  & \_ \\ 
   & {\textsc{IS}} & 38.12 (4.839) & 13.97  &\_  & \_ \\ 
   \hline
   $\pi_4$ & \proposed{} & 0.67 (0.044) & 0.32  & 0.93 / 0.97 & 9.70 (0.021) / 11.64 (0.025) \\ 
   & \vl{} & 0.77 (0.053) & 0.34 & 1.00 & 17.00 (0.031) \\ 
   & \ql{} &0.72 (0.047) & 0.35 & 0.93 & 10.03 (0.013) \\ 
   & {\textsc{FQE}} & 1.70 (0.103) & 0.81  &\_  & \_ \\ 
   & {\textsc{minimax}} & 1.62 (0.104) & 0.73  &\_  & \_ \\ 
   & {\textsc{DR}} &1.35 (0.087) & 0.61  &\_  & \_ \\ 
   & {\textsc{IS}} & 4.78 (0.290) & 2.27  &\_  & \_ \\ 
   \hline 
        \end{tabular}}
        \end{table}

            \section{Real Data Application}
            \label{sec:real}
            In this section we apply { the  methods mentioned in Section  \ref{sec:sim}  that provide confidence intervals (\proposed{}, \vl{} and \ql{})} to the OhioT1DM dataset \citep*{marling2020ohiot1dm} obtained from the Ohio University.
			This dataset contains approximately eight weeks' records of CGM blood glucose levels, insulin doses and self-reported life-event data for each of six subjects with type 1 diabetes. Following \cite{shi2020statistical}, we divide the trajectory of every subject into segments of three-hour spans and constructed the state variable \(S_{i,t} = (S^{(1)}_{i,t},S^{(2)}_{i,t},S^{(3)}_{i,t})\)
            as follows.
            First \(S^{(1)}_{i,t}\) is the average CGM glucose level over the three-hour interval \([t-1,t)\).
            Next, \(S^{(2)}_{i,t}\) is constructed based on the \(i\)-th-subject's self-reported time and the corresponding carbohydrate estimate for the meal. More specifically,
           \(
                S_{i,t}^{(2)} = \sum_{j=1}^J \mathrm{CE}_j \gamma_c^{36(t_j-t+1)},\nonumber
            \)
            where \(\mathrm{CE}_1,\mathrm{CE}_2, \dots, \mathrm{CE}_J\) are carbohydrate estimates for the \(i\)-th-subject's meals at times \(t_1, t_2, \dots, t_J \in [t-1,t)\), and \(\gamma_c\) is a 5-minute decay rate.
            Here we set \(\gamma_c = 0.5\).
            Last, \(S_{i,t}^{(3)}\) is the average of the basal rate during the three-hour interval.
             The action variable \(A_{i,t}\) is binarized  according to the amount of insulin injected. In particular, we set \(A_{i,t} = 1\) when the total amount of insulin delivered to the \(i\)-th subject is larger than one unit during the three-hour interval. Otherwise, we set \(A_{i,t} = 0\).
            In the data, the time series of glucose levels and the life-event data are not perfectly overlapped within the same time frame. So we remove several boundary points of time series to ensure that all state variables have records in the same time frame.  According to the Index of Glycemic Control (IGC) \citep{rodbard2009interpretation}, the immediate reward \(R_{i,t}\) is constructed as 
            \[
                R_{i,t}= 
            \begin{cases}
                -\frac{1}{30}(80- S^{(1)}_{i,t+1})^2,& \text{if } S^{(1)}_{i,t+1} < 80;\\
                0,              & \text{if } 80 \leq S^{(1)}_{i,t+1} < 140;\\
                -\frac{1}{30}(S^{(1)}_{i,t+1} - 140)^{1.35}, &\text{otherwise }.
            \end{cases}
            \]

            Similar to the simulation study, we set the discount factor \(\gamma = 0.9\). In addition, we study six reference distributions \(\mathbb{G}\), where each of them is taken as the point mass at the initial state of a subject. We evaluate the aforementioned three estimators via two policies --- the always-treat policy \(\pi_1(a \mid s) \equiv 1\) and the never-treat policy \(\pi_0(a \mid s) \equiv 0\) since
            the estimated optimal policy based on \cite{shi2020statistical} is very close to the always-treat policy.
           A similar discovery has also been observed in \cite{luckett2019estimating}.
            Therefore, we expect to see that \(\calV(\pi_1) > \calV(\pi_0)\).

            \begin{figure}[h]
                \caption{Offline policy evaluation (OPE) estimates as well as the associated 95\% confidence intervals on the OhioT1DM dataset under always-treat and never-treat policies. The six sub-figures displays these quantities when the reference distributions are the initial state variable of six patients respectively. 
                The adjusted confidence intervals for \proposed{} are also included.} 
                \includegraphics[width =14.5cm]{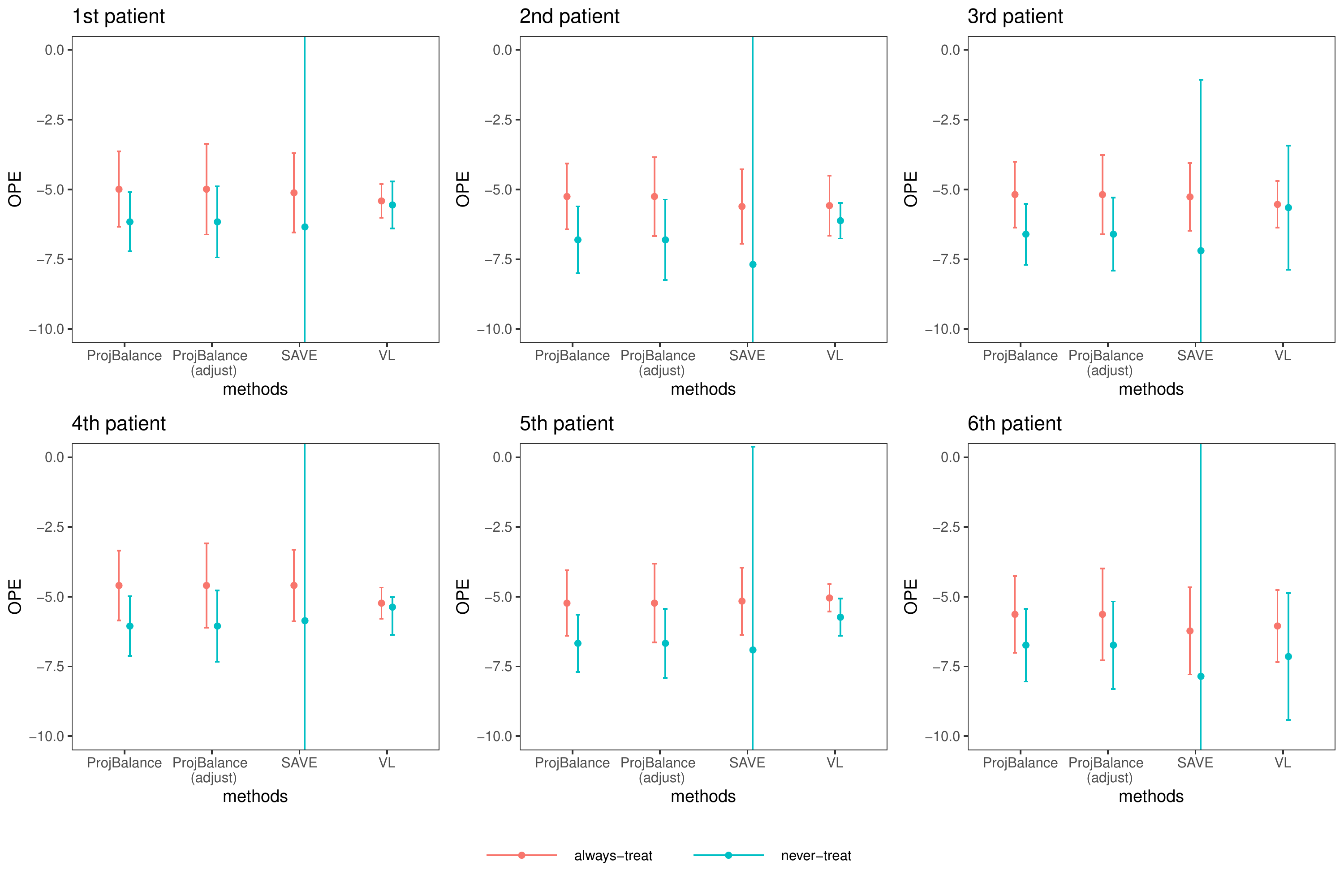}
                \label{fig:ohio}
            \end{figure}

            Figure \ref{fig:ohio} shows the point estimates and associated confidence intervals of \(\calV(\pi_1)\) and \(\calV(\pi_0)\) for \proposed{}, \vl{} and \ql{}.
            To adjust for the under-coverage issue we observe in the simulation study, we also calculate the adjusted confidence interval for \proposed{}.
			In most scenarios, it can be observed that estimated values under the always-treat policy are larger than that of the never-treat policy, which aligns with our expectation. Note that the results of \vl{} do not show a distinctive difference between the two policies compared with other methods.
            For \ql{}, due to the invertibility issue (the matrix \(\hat \Sigma_{\pi}\) in \cite{shi2020statistical} is numerically low-rank 
			under the never-treat policy), the confidence interval for the never-treat policy is surprisingly large. Even though the confidence intervals of our proposed estimators rely on the Q-function estimated by \ql{} (where the point estimate of the Q-function does not involve inverting the problematic matrix), the confidence interval does not inherit the instability issue.
   
\section{Acknowledgements}
Wong's research was partially supported by the National Science Foundation (DMS-1711952 and CCF-1934904).
Portions of this research were conducted with the advanced computing resources provided by Texas A\&M High Performance Research Computing.

\clearpage

\setcounter{page}{1}
\title{Supplement to ``Projected State-action Balancing Weights for Offline Reinforcement Learning''}
\maketitle

\renewcommand{\thesection}{S\arabic{section}}
\renewcommand{\theequation}{S\arabic{equation}}
\renewcommand{\thefigure}{S\arabic{figure}}
\renewcommand{\thetable}{S\arabic{table}}
\renewcommand{\thelemma}{S\arabic{lemma}}
\renewcommand{\thecorollary}{S\arabic{corollary}}
\renewcommand{\thetheorem}{S\arabic{theorem}}
\renewcommand{\theassumption}{\arabic{theorem}'}
\renewcommand{\theproposition}{S\arabic{proposition}}
\renewcommand{\thedefinition}{S\arabic{definition}}
\setcounter{section}{0}

\section{Additional Implementation Details and Numerical Results}
	
	\subsection{Algorithm}
	We present the algorithm outline for the proposed estimator in Algorithm \ref{algo:computation}.
	\begin{algorithm}
		\SetAlgoLined
		\SetKwInOut{init}{Initialization}
		\KwIn{Observed trajectories \(\calD_n = \{\{S_{i,t}, A_{i,t}, R_{i,t}\}_{t=0}^{T-1} \cup S_{i,T}\}_{i=1}^n\); number of basis functions \(K\);  a grid of tuning parameters \(\mu_{l}\), \(l = 1,\dots, L\); function \(h\); basis function  \(B_{k}\), \(k = 1,\dots, K\); reference distribution: $\mathbb{G}$.
		}
		Calculate \(\EE_{S\sim \mathbb{G}} \{\sum_{a\in \mathcal{A}}\pi(a\mid S) B_k(S,a)\}\), \(k = 1,\dots, K\).
		
		Divide $\calD_n$ into the training data $\calD_1$ and validation data $\calD_2$.
		\BlankLine
		\For {$l = 1,\dots, L$}{
			\For{$k=0,1,\dots,K$}
			{
				Compute the solution of \eqref{eqn:approx_g} (i.e., \(\hat g^\pi_l(\cdot, \cdot, B_k)\)) by  taking \(\mu = \mu_{l}\) on $\calD_1$.
			}
			Calculate the validation error \(\sum_{k=1}^K  \sum_{(i,t) \in \calD_2}\{\hat g^\pi_l(S_{i,t}, A_{i,t}, B_k) - \sum_{a\in \mathcal{A}} \pi(a'\mid S_{i,t+1}) B_k(S_{i,t+1}, a')\}^2 / \{(nT)^{-1} \sum_{i=1}^n \sum_{t=1}^T B^2_k(S_{i,t}, A_{i,t})\}\)
		}
		
		Select  \(l^*\) that achieves the smallest validation error among all tuning parameters.
		
		Take \(\hat g^\pi(S_{i,t}, A_{i,t}, B_k)\) as \(\hat g^\pi_{l^*}(S_{i,t}, A_{i,t}, B_k)\) and optimize \eqref{eqn:optim1} in the main text with the smallest feasible \(\bm \delta_K = \delta \bm J\), where \(\bm J = [1,\dots, 1]^\tp\). Obtain the solution \(\hat \omega^\pi_{i,t}\), \(i = 1, \dots, n\) and $t = 0, \cdots, (T-1)$.
		
		\KwOut{$ \hat \calV^\pi(\pi)  = \frac{1}{nT}\sum_{i=1}^n \sum_{t=0}^{T-1} \hat \omega^\pi_{i,t} R_{i,t}$}
		\caption{Outlines for computing \(\hat \calV(\pi)\).}
		\label{algo:computation}
	\end{algorithm}

	\subsection{Implementation  of the Proposed Estimator}
	\label{sec:sim_implement}
	In this subsection, we discuss the implementation details in the simulation study. For the proposed estimator (\proposed{}), the basis functions were constructed as follows. 
	{For settings described in Section \ref{sec:sim}, we adopted the tensor product basis functions for two dimensional state variables.}
	First, we set \(B_{1,l}(S,A) = b_{1,l_1}(S^{(1)})b_{2,l_2}(S^{(2)})\mathbbm{1}_{A =1 }\), \(B_{2,l}(S,A) = b_{1,l_1}(S^{(1)})b_{2,l_2}(S^{(2)})\). Here \(b_{1,l_1}\) and  \(b_{2,l_2}\) were one-dimensional cubic B-spline sets where internal knots were placed at equally spaced sample quantiles for state variables. To avoid extrapolation of the basis function, we put three repeated boundary knots. Then we chose \(\bm B_{K} (S,A) =[B_{k}(S,A)]_{k=1}^K = [B_{1,l}(S,A) , B_{2,l}(S,A)]_{l=1}^{K/2}\). In our numerical experiments, we set \(K = 2\max\{(nT)^{1/3}, 4^2\}\), where \(4^2\) was determined by the fact that there should be at least 4 one-dimensional basis for constructing cubic B-spline.
	{Note that the number of tensor products of spline bases grows exponentially as the dimension of the state space increases. We recommend reducing the interaction order of the tensor products or using kernel bases in the higher dimensional settings, while the number of basis functions \(K\) can be still chosen as stated before.  In the simulation study for Cartpole datasets in Section \ref{sec:sim_more}, we did not include any interaction of tensor products among four state variables, for the sake of simplicity. }
	For the approximate projection step \eqref{eqn:approx_g}, we used a kernel ridge regression with Gaussian kernel,
	where the bandwidth parameter in Gaussian kernel was chosen by the median heuristic (See Section 2.2 in \cite{garreau2017large}).%
	We adopted 5-fold cross-validation to tune the parameter \(\mu\).
	Note that we only needed one \(\mu\) for \(K\) projections, so we took an average of the ``standardized" validation errors from the projection steps for all basis functions, where the validation errors were standardized by empirical norm of the corresponding basis function, as our validation criterion. 
	For the weight estimation \eqref{eqn:optim2}, 
	we chose \(h(x) = (x-1)^2\). As \(\mathbb{G}\) is known to us, we used Monte Carlo method to obtain \(\EE_{S\sim \mathbb{G}}\{\sum_{a'\in \mathcal{A}} \pi(a'\mid S)B_k(S, a')\}\). To reduce the amount of tuning,
	we took \(\bm \delta_K = \delta \bm J\), where \(\delta >0\) and \(J = [1,\dots,1]\in \mathbb{R}^K\), and selected the minimal \(\delta\) that had a valid solution to \eqref{eqn:optim2}.

	In the following, we provide the details for constructing the confidence interval for \proposed{}.
	Take \(\hat {Q}^\pi\) as the estimator for \(Q^\pi\) obtained by \ql{} in \cite{shi2020statistical} and let \[\hat \sigma_{i,t}= \hat \omega_{i,t}^\pi \left( R_{i,t} + \gamma \sum_{a'\in \mathcal{A}} \pi(a'\mid S_{i,t+1}) \hat Q^\pi(S_{i,t+1}, a') - \hat Q^\pi(S_{i,t}, A_{i,t})\right).\]
	Calculate \[\hat \sigma = \left\{ \frac{1}{nT}\sum_{i=1}^n \sum_{t=0}^{T-1} \hat\sigma_{i,t}^2 \right\}^{1/2}.\]
	Then  the confidence interval is given by 
	\[\left[\hat \calV(\pi) - z_{\frac{\alpha}{2}}\frac{\hat \sigma}{\sqrt{nT}},\   \hat \calV(\pi) +  z_{\frac{\alpha}{2}}\frac{\hat \sigma}{\sqrt{nT}} \right],\]
	where \(\hat \calV(\pi) = (nT)^{-1}\sum_{i=1}^n\sum_{t=0}^{T-1}\hat \omega^\pi_{i,t} R_{i,t}\)
	and \(z_{\frac{\alpha}{2}}\) is $(1-\alpha/2)$-quantile of the standard normal distribution.

	\subsection{Additional Simulation Results}	
	\label{sec:sim_more}
	Tables \ref{tab:40_25}, \ref{tab:80_25} and \ref{tab:80_50} show the simulation results under the  target policies and data generation process described in Section \ref{sec:sim}, where \((n,T) = (40, 25)\), \((n,T) = (80,25)\) and \((n,T) = (80, 50)\) respectively.
	
	\begin{table}[h]
		\centering
		\caption{Simulation results for four different target policies when  $n = 40$ and  $T = 25$.  See details in Table \ref{tab:40_50}.}
		\label{tab:40_25}
		\resizebox{\textwidth}{!}{%
		\begin{tabular}{c |c |cccc}
			\hline
			Target & methods  & MSE ($\times 1000$) & MeSE ($\times 1000$) & ECP & AL ($\times 100$) \\ 
			\hline
			$\pi_1$ & \proposed{} & 17.32 (1.064) & 8.17 & 0.96 / 0.99 & 66.30 (6.145) / 79.56 (7.374) \\ 
			& \vl{} & 15.43 (0.942) & 7.08 & 0.97 & 54.15 (0.472) \\ 
			& \ql{} & 272.36 (178.185) & 11.06 & 0.96 &  637.96 (337.450) \\
			& {\textsc{FQE}} & 23.58 (1.431) & 10.74  &\_  & \_ \\ 
			   & {\textsc{minimax}} & 46.27 (3.544) & 16.75  &\_  & \_ \\ 
			   & {\textsc{DR}} & 855.00 (178.345) & 82.56   &\_  & \_ \\ 
			   & {\textsc{IS}} & 945.92 (234.336) & 53.75  &\_  & \_ \\  
			\hline
			$\pi_2$ & \proposed{} & 10.83 (0.874) & 3.17 & 0.87 / 0.91 &35.26 (4.682) / 42.31 (5.619) \\ 
			& \vl{} &  12.33 (0.721) & 5.65  & 0.88 &  36.63 (0.364)  \\ 
			& \ql{} & 310.41 (301.139) & 3.25 & 0.94 & 918.00 (874.044)\\ 
			& {\textsc{FQE}} & 23.58 (1.431) & 10.74  &\_  & \_ \\ 
			   & {\textsc{minimax}} & 41.52 (1.973) & 28.31  &\_  & \_ \\ 
			   & {\textsc{DR}} & 60.20 (5.623) & 16.30   &\_  & \_ \\ 
			   & {\textsc{IS}} & 2973.88 (1803.384) & 70.509  &\_  & \_ \\  
			\hline
			$\pi_3$ & \proposed{} & 4.62 (0.299) & 1.899 & 0.92 / 0.95 &  27.67 (0.984) / 33.20 (1.181) \\ 
			& \vl{} &  4.50 (0.305) & 1.89 & 0.98 & 33.04 (0.239) \\ 
			& \ql{} & 15.87 (8.981) & 2.35 & 0.94 & 669.62 (636.178)  \\ 
			& {\textsc{FQE}} & 8.687 (0.604) & 3.69  &\_  & \_ \\ 
			   & {\textsc{minimax}} & 45.51 (1.710) & 39.61  &\_  & \_ \\ 
			   & {\textsc{DR}} & 68.18 (21.377) & 3.98  &\_  & \_ \\ 
			   & {\textsc{IS}} & 34.02 (4.214) & 12.45  &\_  & \_ \\  
			\hline
			$\pi_4$ & \proposed{} & 1.25 (0.087) & 0.49& 0.94 / 0.98 & 13.61 (0.074) / 16.33 (0.089)  \\ 
			& \vl{} &  1.63 (0.111) & 0.66 & 0.99 & 24.47 (0.070) \\ 
			& \ql{} & 1.40 (0.096) & 0.53 & 0.95 & 14.68 (0.119) \\ 
			& {\textsc{FQE}} & 3.46 (0.224) & 1.65  &\_  & \_ \\ 
			   & {\textsc{minimax}} & 3.26 (0.202) & 1.50  &\_  & \_ \\ 
			   & {\textsc{DR}} & 2.48 (0.157) & 1.15   &\_  & \_ \\ 
			   & {\textsc{IS}} & 4.67 (0.281) & 2.28  &\_  & \_ \\  
			\hline
		\end{tabular}}
	\end{table}

	\begin{table}[h]
		\centering
		\caption{Simulation results for four different target policies when  $n = 80$ and  $T = 25$. See details in Table \ref{tab:40_50}.}
		\label{tab:80_25}
		\resizebox{\textwidth}{!}{%
		\begin{tabular}{c |c |cccc}
			\hline
			Target& methods  & MSE (\(\times 1000\)) & MeSE (\(\times 1000\)) & ECP & AL ($\times 100$) \\ 
			\hline
			$\pi_1$ & \proposed{} & 10.13 (0.661) & 4.61 & 0.95 / 0.97 &38.95 (0.175) / 46.74 (0.211)\\ 
			& \vl{} & 8.84 (0.562) & 3.66 & 0.95 & 38.38 (0.178) \\ 
			& \ql{} & 11.55 (0.766) & 4.90 & 0.94 & 40.58 (0.183)  \\ 
			& {\textsc{FQE}} & 3.46 (0.224) & 1.65  &\_  & \_ \\ 
			   & {\textsc{minimax}} & 38.08 (3.330) & 17.80  &\_  & \_ \\ 
			   & {\textsc{DR}} & 181.51 (24.850) & 28.07   &\_  & \_ \\ 
			   & {\textsc{IS}} & 4.67 (0.281) & 2.28  &\_  & \_ \\  
			\hline
			$\pi_2$ & \proposed{} &  3.88 (0.279) & 1.76 & 0.92 / 0.97 &  21.53 (0.120) / 25.83 (0.144) \\ 
			& \vl{} & 5.06 (0.293) & 2.46 & 0.94 &25.18 (0.145) \\ 
			& \ql{} &3.42 (0.199) & 1.77 & 0.96 & 23.55 (0.512) \\ 
			& {\textsc{FQE}} & 3.88 (0.214) & 2.04  &\_  & \_ \\ 
			   & {\textsc{minimax}} & 36.48 (1.556) & 26.28  &\_  & \_ \\ 
			   & {\textsc{DR}} & 28.51 (2.915) & 7.08  &\_  & \_ \\ 
			   & {\textsc{IS}} & 2951.59 (2613.528) & 62.73  &\_  & \_ \\  
			\hline
			$\pi_3$ & \proposed{} &  2.33 (0.134) & 1.09 & 0.94 / 0.98 & 18.36 (0.1) / 22.03 (0.120))\\ 
			& \vl{} & 2.33 (0.137) & 1.18 & 0.98 & 23.75 (0.122) \\ 
			& \ql{} & 2.74 (0.161) & 1.26 &  0.94 & 19.56 (0.177) \\ 
			& {\textsc{FQE}} & 4.96 (0.322) & 2.10  &\_  & \_ \\ 
			   & {\textsc{minimax}} & 32.12 (1.150) & 25.85  &\_  & \_ \\ 
			   & {\textsc{DR}} & 15.64 (5.164) & 2.35   &\_  & \_ \\ 
			   & {\textsc{IS}} & 22.39 (4.263) &7.05  &\_  & \_ \\  
			\hline
			$\pi_4$ & \proposed{} &0.75 (0.042) & 0.41 & 0.92 / 0.98 &  9.81 (0.018) / 11.77 (0.022)\\ 
			& \vl{} & 0.88 (0.056) & 0.37 & 1 & 17.32 (0.033) \\ 
			& \ql{} &  0.79 (0.045) & 0.43  & 0.93 &  10.19 (0.107) \\ 
			& {\textsc{FQE}} & 1.66 (0.713) & 0.66  &\_  & \_ \\ 
			   & {\textsc{minimax}} & 1.46 (0.084) & 0.73  &\_  & \_ \\ 
			   & {\textsc{DR}} & 1.25 (0.076) & 0.62   &\_  & \_ \\ 
			   & {\textsc{IS}} & 2.27 (0.144) & 1.10  &\_  & \_ \\  
			\hline
		\end{tabular}}
	\end{table}

	\begin{table}[h]
		\centering
		\caption{Simulation results for four different target policies when  $n = 80$ and  $T = 50$. See details in Table \ref{tab:40_50}.}
		\label{tab:80_50}
		\resizebox{\textwidth}{!}{%
		\begin{tabular}{c |c |cccc}
			\hline
			Target & methods  &  MSE (\(\times 1000\)) & MeSE (\(\times 1000\))  & ECP & AL ($\times 100$) \\ 
			\hline
			$\pi_1$ & \proposed{} & 5.43 (0.353) & 2.31  & 0.92 / 0.97 &  27.2(0.0679) / 32.58 (0.081) \\ 
			& \vl{} &  5.09 (0.318) & 2.01  & 0.94 & 26.80 (0.075)   \\ 
			& \ql{} & 5.93 (0.382) & 2.65 & 0.92 & 28.09 (0.117)  \\ 
			& {\textsc{FQE}} & 6.21 (0.402) & 2.57  &\_  & \_ \\ 
			   & {\textsc{minimax}} & 27.27 (2.030) & 12.17  &\_  & \_ \\ 
			   & {\textsc{DR}} & 67.29 (11.554) & 11.03   &\_  & \_ \\ 
			   & {\textsc{IS}} & 192585.20 (189045.400) & 44.53  &\_  & \_ \\  
			\hline
			$\pi_2$ & \proposed{} &  2.49 (0.204) & 1.09 & 0.87 / 0.93 &  15.05 (0.045)  / 18.06 (0.054) \\ 
			& \vl{} & 2.92 (0.162) & 1.58 & 0.89 & 16.93 (0.056)  \\ 
			& \ql{} &1.73 (0.107) & 0.72 & 0.94 & 15.78 (0.053) \\ 
			& {\textsc{FQE}} & 2.10 (0.140) & 0.91  &\_  & \_ \\ 
			& {\textsc{minimax}} & 33.13 (1.275) & 26.19  &\_  & \_ \\ 
			& {\textsc{DR}} & 11.98 (1.063) & 4.10   &\_  & \_ \\ 
			& {\textsc{IS}} & 2951.59 (2613.528) & 62.73  &\_  & \_ \\  
			\hline
			$\pi_3$ & \proposed{} & 1.19 (0.071) & 0.61 & 0.94 / 0.98 &  13.01 (0.048) / 15.61 (0.057) \\ 
			& \vl{} & 1.25 (0.071) & 0.73 & 0.98 & 16.42 (0.053)\\ 
			& \ql{} & 1.34 (0.079) & 0.64 & 0.94 &  13.97 (0.359)\\ 
			& {\textsc{FQE}} & 2.23 (0.174) & 0.89  &\_  & \_ \\ 
			& {\textsc{minimax}} & 32.78 (0.747) & 30.92  &\_  & \_ \\ 
			& {\textsc{DR}} & 3.45 (0.966) & 0.86   &\_  & \_ \\ 
			& {\textsc{IS}} & 29.67 (6.884) & 6.79  &\_  & \_ \\  
			\hline
			$\pi_4$ & \proposed{} & 0.39 (0.023) & 0.21  & 0.92 / 0.97 & 6.900 (0.008) / 8.28 (0.009)\\ 
			& \vl{} & 0.44 (0.026) & 0.24 & 1 & 12.12 (0.027) \\ 
			& \ql{} & 0.41 (0.024) & 0.22 & 0.92 & 7.06 (0.005) \\ 
			& {\textsc{FQE}} & 0.77 (0.046 & 0.38 &\_  & \_ \\ 
			& {\textsc{minimax}} & 0.77 (0.049) & 0.39  &\_  & \_ \\ 
			& {\textsc{DR}} & 0.62 (0.039) & 0.29   &\_  & \_ \\ 
			& {\textsc{IS}} & 2.26 (0.143) & 1.02  &\_  & \_ \\  
			\hline
		\end{tabular}}
	\end{table}

	\subsection{{Weighted Estimator without Projection and Augmented Projected Balancing Estimator}}
\label{sec:sim_proj}
	{In this section, we compare the performance of the weighted estimator (\textsc{balance}) where weights are obtained from \eqref{eqn:optim1}  with  our proposed estimator \proposed{} under the simulation settings described in Section \ref{sec:sim}. In particular, the same basis functions are adopted for two weighted estimators in the balancing step. Also, we include the augmented estimator (\textsc{aug}) based on our procedure into the comparison. The augmented estimator is constructed by 
	\begin{multline*}
		(1-\gamma)\EE_{S\sim \mathbb{G}}\left\{ \sum_{a\in \mathcal{A}}\pi(a\mid S)\hat Q^\pi(S, a) \right\}\\
		 + \frac{1}{nT}\sum_{i=1}^n \sum_{t=0}^{T-1} \hat\omega^\pi_{i,t} \left[ R_{i,t} -\hat Q^\pi (A_{i,t}, S_{i,t}) + \gamma \left\{ \sum_{a'\in \mathcal{A}} \pi(a'\mid S_{i,t+1})\hat Q^\pi(S_{i,t+1}, a') \right\}  \right],
	\end{multline*}
	where \(\hat Q^\pi\) is estimated by \ql{} and \(\hat \omega^\pi_{i,t}\), \(i = 1,\dots, n\), \(t = 0,\dots, T-1\) are obtained from \eqref{eqn:optim2}. 
	Table \ref{tab:proj} shows the performance of above three estimators when \(n = 40\) and \(T = 50\). As we can see,  \textsc{balance} only performs better when target policy is \(\pi_1\). Under other settings, \proposed{} performs much better than \textsc{balance}, especially when target policy is \(\pi_2\). As for \textsc{aug}, it has a very similar performance to \ql{} in all four settings. 
	}

	\begin{table}[h]
		\centering
		\caption{{Simulation results for \proposed{},  \textsc{balance} and \textsc{aug} under four different target policies when  $n = 40$ and  $T = 50$. See details in Table \ref{tab:40_50}}.}
		\label{tab:proj}
		\resizebox{\textwidth}{!}{%
		\begin{tabular}{c |c |cccc}
			\hline
			Target & methods  & MSE ($\times 1000$) & MeSE ($\times 1000$) & ECP & AL ($\times 100$) \\ 
			\hline
			$\pi_1$ & \proposed{} & 8.76 (0.527) & 3.75  & 0.96 / 0.99 & 38.93 (0.258) / 46.71 (0.310) \\ 
			& {\textsc{balance}} & 4.73 (0.298) & 2.36  & 0.99 / 1.00 & 38.63 (0.127) / 46.36 (0.152) \\ 
			& {\textsc{aug}} & 9.86 (0.585) & 4.31  & 0.95 / 0.99 & 38.81 (0.259) / 46.57 (0.310) \\ 
			\hline
			$\pi_2$ & \proposed{} & 4.29 (0.356) & 1.75  & 0.91 / 0.94 & 23.57 (2.130) / 28.28 (2.556) \\ 
			& {\textsc{balance}} & 117.23 (1.38) & 111.44  & 0.002 / 0.002 & 19.43 (0.583) / 23.33 (0.699) \\ 
			& {\textsc{aug}} & 36.66 (33.118) & 1.52  & 0.92 / 0.95 & 23.60 (2.166) / 28.32 (2.599) \\
			\hline
			$\pi_3$ & \proposed{} & 2.41 (0.153) & 1.06 & 0.93 / 0.97 & 18.47 (0.125) / 22.17 (0.150) \\ 
			& {\textsc{balance}} & 8.32 (0.32) & 6.79  & 0.45 / 0.57 & 14.95 (0.048) / 17.94 (0.058) \\
			& {\textsc{aug}} & 2.62 (0.164) & 1.16  & 0.92 / 0.97 & 18.50 (0.124) / 22.20 (0.148) \\
			\hline
			$\pi_4$ & \proposed{} & 0.67 (0.044) & 0.32  & 0.93 / 0.97 & 9.70 (0.021) / 11.64 (0.025) \\ 
			& {\textsc{balance}} & 0.84 (0.053) & 0.37  & 0.91 / 0.95 & 9.64 (0.010) / 11.56 (0.012) \\ 
			& {\textsc{aug}} & 0.73 (0.048) & 0.34  & 0.92 / 0.96 & 9.69 (0.023) / 11.63 (0.027) \\
			\hline
		\end{tabular}}
	\end{table}

	\subsection{{Simulation results for the Cartpole environment}}
	\label{sec:sim_cartpole}
{In this section, we compare the performance of various estimators using the CartPole datasets from the OpenAI Gym environment \citep{brockman2016openai}. The state variables in the CartPole environment are four dimensional, including cart position, cart velocity, pole angle and pole angular velocity. The action  is binary.  We made a slight modification to the Cartpole environment following \cite{uehara2020minimax} and \cite{shi2021deeply}. To summarize, we added a small Gaussian noise on the original deterministic transition dynamic and defined a new state-action-dependent reward function. See details in Section B2.1 in \cite{shi2021deeply}. Following \cite{uehara2020minimax} and \cite{shi2021deeply}, we first ran a deep-Q network to get a near-optimal policy as the target policy ($\pi_{opt}$). 
And we then applied ``epsilon-greedy'' adjustment using a factor $\epsilon$ to define our behavior policy, i.e.,
\begin{align*}
	\pi_b (a\mid s) =  (1-\epsilon) * \pi_{opt}(a\mid s) + \epsilon * 0.5,
\end{align*}
for every $(s, a) \in \calS \times \calA$.
For comparison, we implemented various methods such as fitted-Q evluation (\textsc{FQE}), \textsc{minimax} and \textsc{DR} using the ways they are implemented in \cite{shi2021deeply} and the public code from \url{https://github.com/RunzheStat/D2OPE}. We set the discount factor \(\gamma = 0.9\) in our evaluations. 
}

{In Tables \ref{tab:cartpole_0.1}, \ref{tab:cartpole_0.2} and \ref{tab:cartpole_0.3}, we show the results for \(200\) simulated dataset when setting \(\epsilon = 0.1\), \(\epsilon = 0.2\) and \(\epsilon = 0.3\) respectively. We fixed \(n = 50\) and \(T = 100\) for all three settings. As we can see,   \proposed{} has the smallest  or the second smallest MSE and MeSE among all the methods. The emprical coverages of \proposed{} are also better than other two methods (\ql{} and \vl{}). We did not report the confidence interval for the remaining methods because there are no related theoretical results.}

	\begin{table}[h]
		\centering
		\caption{{Simulation results for the Cartpole environment when  $n = 50$, $T = 100$ and $\epsilon = 0.1$. See detailed description in Table  \ref{tab:40_50}}}
		\label{tab:cartpole_0.1}
		\resizebox{\textwidth}{!}{%
		\begin{tabular}{c |c |cccc}
			\hline
			methods  & MSE ($\times 10^6$) & MeSE ($\times 10^6$) & ECP & AL ($\times 10^3$) \\ 
			\hline
			{\proposed{}} & 0.17 (0.011) & 0.13 & 0.72 / 0.86 & 0.93 (0.016) / 1.12 (0.02)\\  
			{\textsc{\ql{}}}  & 0.49 (0.022) & 0.43 & 0.12 & 0.92 (0.018) \\
			{\textsc{\vl{}}}  & 32.66 (1.262) & 29.57 & 0 & 1.84 (0.008) \\
			{\textsc{FQE}} &5.15 (0.033)&5.12& \_& \_\\
			{\textsc{minimax}} & 43.04 (1.122)&43.24 & \_& \_\\ 
			{\textsc{DR}}&  1.26 (0.042)&1.11& \_& \_ \\
			{\textsc{IS}} & 5581.23 (491.491)&2757.71& \_& \_\\
			\hline
		\end{tabular}}
	\end{table}

	\begin{table}[h]
		\centering
		\caption{{Simulation results for the Cartpole environment when  $n = 50$, $T = 100$ and $\epsilon = 0.2$.  See detailed description in Table \ref{tab:40_50}}}
		\label{tab:cartpole_0.2}
		\resizebox{\textwidth}{!}{%
		\begin{tabular}{c |c |cccc}
			\hline
			methods  & MSE ($\times 10^6$) & MeSE ($\times 10^6$) & ECP & AL ($\times 10^3$) \\ 
			\hline
			{\proposed{}} & 0.10 (0.009) & 0.04 & 0.84 / 0.91 & 0.85 (0.014) / 1.02 (0.016)\\  
			{\textsc{\ql{}}}  & 0.33 (0.021) & 0.24 & 0.38 & 0.84 (0.014) \\
			{\textsc{\vl{}}}  & 23.06 (0.547) & 22.5 & 0 & 1.21 (0.003) \\
			{\textsc{FQE}} & 4.44 (0.043)&4.38& \_& \_\\
			{\textsc{minimax}} & 35.28 (1.163)&33.69 & \_& \_\\ 
			{\textsc{DR}}&  0.76 (0.049)&0.64& \_& \_ \\
			{\textsc{IS}} & 20541.94 (2356.361)&9295.65& \_& \_\\
			\hline
		\end{tabular}}
	\end{table}

	\begin{table}[h]
		\centering
		\caption{{Simulation results for the Cartpole environment when  $n = 50$, $T = 100$ and $\epsilon = 0.3$.  See detailed description in Table  \ref{tab:40_50}}
		}
		\label{tab:cartpole_0.3}
		\resizebox{\textwidth}{!}{%
		\begin{tabular}{c |c |cccc}
			\hline
			methods  & MSE ($\times 10^6$) & MeSE ($\times 10^6$) & ECP & AL ($\times 10^3$) \\ 
			\hline
			{\proposed{}} & 0.14 (0.011) & 0.08 & 0.65 / 0.76 & 0.84 (0.031) / 1.01 (0.038)\\  
			{\textsc{\ql{}}}  &  0.13 (0.014) & 0.05 & 0.74 & 0.93 (0.124) \\
			{\textsc{\vl{}}}  & 16.08 (0.326) & 15.9 & 0 & 1.02 (0.001)\\
			{\textsc{FQE}} &3.99 (0.029)&3.98 & \_& \_\\
			{\textsc{minimax}} & 26.94 (0.962)&25.54 & \_& \_\\ 
			{\textsc{DR}}& 0.87 (0.053)&0.7& \_& \_ \\
			{\textsc{IS}} &59936.33 (7265.061)&30726.43& \_& \_\\
			\hline
		\end{tabular}}
	\end{table}

	\section{Proof}

	\subsection{{Proof in Section \ref{sec:basic}}}
	\label{sec:proof_sec2}
	{Here we show how to obtain \eqref{eq: Estimating Equation for ratio} from \eqref{eq: fwd bellman d}.}
	
	{Starting from (9), we multiply \(f(s,a)\) on both sides and get}
   { \begin{align}
      &d^\pi(s,a) f(s,a) = (1-\gamma) G(s) \pi(a\mid s) f(s,a)+ \gamma f(s,a)\int_{s'\in \mathcal{S}} \sum_{a'\in \mathcal{A}}d^{\pi}(s', a')p(s \mid s', a')\pi(a \mid s)ds' \label{eqn:wdef1}
    \end{align}
    Then we integrate with respect to \(s\) and \(a\) for both sides of \eqref{eqn:wdef1}. Next by the change of measure, we obtain
    \begin{align}
      \int_{s \in \mathcal{S}} \sum_{a \in \mathcal{A}} d^\pi(s,a) f(s,a) ds & = (1-\gamma) \int_{s \in \mathcal{S}, a \in \mathcal{A}} \mathbb G(s) \pi(a\mid s) f(s,a) \nonumber\\
    &  \qquad  + \gamma \int_{s \in \mathcal{S}} \sum_{a\in \mathcal{A}} f(s,a)\int_{s'\in \mathcal{S}} \sum_{a'\in \mathcal{A}}d^{\pi}(s', a')p(s \mid s', a')\pi(a \mid s)ds' ds \nonumber\\
    & =  (1-\gamma) \int_{s \in \mathcal{S}, a \in \mathcal{A}} \mathbb G(s) \pi(a\mid s) f(s,a) \nonumber\\
    &  \qquad  + \gamma \int_{s'\in \mathcal{S}}  
    \sum_{a'\in \mathcal{A}}d^{\pi}(s', a') \int_{s \in \mathcal{S}} \sum_{a \in \mathcal{A}} f(s,a)p(s \mid s', a')\pi(a \mid s) ds ds' \nonumber\\
    & = (1-\gamma) \EE_{S_0 \sim \mathbb{G}} \left[ \sum_{a \in \mathcal{A}} f(a, S0) \right] \nonumber\\
    &  \qquad  + \gamma \int_{s\in \mathcal{S}}  
    \sum_{a\in \mathcal{A}} d^{\pi}(s, a) \EE\left[\sum_{a'\in \mathcal{A}} \pi(a'\mid S') f(S', a)\mid S = s, A= a  \right] ds \label{eqn:wdef2}
    \end{align}}

	{Then, based on the definition of \(\bar{p}_T^b(S_t, A_t)\), we have
	\begin{multline}
	   \EE \left[\frac{1}{T}\sum_{t=0}^{T-1} \frac{d^{\pi}(S_t, A_t)}{\bar{p}_T^b(S_t, A_t)} \left\{f(S_t, A_t) - \gamma \EE\left[\sum_{a'\in \mathcal{A}} \pi(a'\mid S_{t+1}) f(S_{t+1}, a)\mid S_t, A_t  \right]\right\}\right]\\
	   =  \frac{1}{T} \sum_{t=0}^{T-1} \int_{s \in \mathcal{S}} \sum_{a \in \mathcal{A}} p_t(s, a)  \frac{d^\pi(s, a)}{\bar p^b_T(s, a)} \left\{f(s, a) - \gamma \EE\left[\sum_{a'\in \mathcal{A}} \pi(a'\mid s') f(s', a)\mid s, a  \right]\right\} ds \\
		= \int_{s \in \mathcal{S}} \sum_{a \in \mathcal{A}} {d^\pi(s, a)}\left\{f(s, a) - \gamma \EE\left[\sum_{a'\in \mathcal{A}} \pi(a'\mid s') f(s', a)\mid s, a  \right]\right\} ds. \label{eqn:wdef3}
	\end{multline}}

{Combining \eqref{eqn:wdef2} and \eqref{eqn:wdef3},  equation \eqref{eq: Estimating Equation for ratio} follows.}

	\subsection{Technical Proof in Section \ref{sec:method}}
	\label{sec:proof_dual}
	Here we show the proof of Lemma \ref{lem:expand_dim_example} and Theorem \ref{thm:dual}. For Lemma \ref{lem:expand_dim_example}, the proof is straightforward. For Theorem 2, we first remove the absolute value notation in \eqref{eqn:cons} by adding additional \(K\) constrains and therefore transform the optimization problem into a standard convex optimization problem with linear constraints.  Then we make use of the duality results in \cite{tseng1991relaxation} to show the corresponding statement in Lemma \ref{lem:expand_dim_example}.

	\begin{proof}[Proof of Lemma \ref{lem:expand_dim_example}]
		First, for any function $g\in\mathcal{G}'$, there exists $c\in \mathbb{R}$ (independent of $s,a,s'$) such that
		\[
		g(s,a,s') = f(s,a) - \gamma c \quad \mbox{where}\quad \sum_{a'\in\mathcal{A}} \pi(a'|s') f(s',a')\equiv c.
		\]
		Re-parametrizing $f(s,a)= m(s,a) + c$, we have
		\begin{equation*}
		\mathcal{G}' = \left\{ g(s,a,s') = m(s,a) + (1-\gamma)c: c\in \mathbb{R}, \sum_{a'\in\mathcal{A}} \pi(a'\mid s') m(s',a') \equiv 0 \right\}.
		\end{equation*}
		
		Let $\{s_1,\dots,s_{p_S}\}$ and $\{a_1,\dots, a_{p_A}\}$ be the possible values of $\mathcal{S}$ and $\mathcal{A}$ respectively.
		A real-valued function $m$ defined on $\mathcal{S}\times \mathcal{A}$ can be identified by a matrix $\bm{M}=[m(s_i,a_j)]_{i,j=1}^{p_S,p_A}\in\mathbb{R}^{p_S\times p_A}$.
		Similarly, we define $\bm{\Pi}=[\pi(a_j\mid s_i)]_{i,j=1}^{p_S,p_A}$.
		Denote
		by $\bm{1}_a\in \mathbb{R}^a$ a vector of ones for any positive integer $a$.
		To derive $\dim(\mathcal{G}')$, we study the constraint $\sum_{a'\in\mathcal{A}'} \pi(a'\mid s') m(s',a') \equiv 0$, which is equivalent to
		\begin{equation}
		\langle\bm{\pi}_i, \bm{m}_i\rangle = 0, \quad i=1,\dots, p_S, \label{eqn:expand_constraint}
		\end{equation}
		where $\bm{\pi}_i$ and $\bm{m}_i$ are the $i$-th row of $\bm{\Pi}$ and $\bm{M}$ respectively.
		Since the elements of $\bm{\pi}_i$ are non-negative and sum to 1 for all $i$, and each of these constraints are restricted on non-overlapping elements on $\bm{M}$, so \eqref{eqn:expand_constraint} has $p_S$ linearly independent constraints on $\bm{M}$.
		It is easy to see that $\dim(\{g(s,a,s') = m(s,a) : \sum_{a'\in\mathcal{A}} \pi(a'\mid s') m(s',a') \equiv 0\})= p_S p_A-p_S$.
		Together with the parameter $c$, we can
		show that $\dim(\mathcal{G}')= p_Sp_A-p_S+1$.
		
	\end{proof}

	\begin{proof}[Proof of Theorem \ref{thm:dual}]
		Let
		$$
		\hat{\bm L} = \frac{1}{nT}
		\left[
		\begin{array}{cccc c c}
		\hat{A}_1(S_{0,0},A_{0,0}) & \dots & \hat{A}_1(S_{0,T-1},A_{0,T-1})&  \hat{A}_1(S_{1,0},A_{1,0})& \dots& \hat{A}_1(S_{n,T-1},A_{n,T-1})\\
		\vdots  & \ddots  & \vdots  & \vdots & \ddots   &\vdots \\
		\hat{A}_K(S_{0,0},A_{0,0}) & \dots &   \hat{A}_K(S_{0,T-1},A_{0,T-1})&  \hat{A}_K(S_{1,0},A_{1,0})& \dots& \hat{A}_K(S_{n,T-1},A_{n,T-1})\\
		\end{array}
		\right], 
		$$
		$$ \bm w = [w_{i,t}]_{i=1,\dots,n, t = 0,\dots, T} \in \mathbb{R}^{nT}$$
		Then \eqref{eqn:cons} can be represented by 
		
		\begin{align}
		\mbox{min} &\qquad \frac{1}{nT}\sum_{i=1}^n \sum_{t=0}^{T-1} h(w_{i,t}) \nonumber\\
		\mbox{subject to} &\qquad 
		\left[
		\begin{array}{c}
		\hat{\bm L}\\
		-\hat{\bm L}
		\end{array}
		\right] \bm w \leq \left[ 
		\begin{array}{c}
		\bm \delta_K + \bm l_K\\
		\bm \delta_K - \bm l_K
		\end{array}
		\right] \label{eqn:cons2}
		\end{align}

		From  \cite{tseng1991relaxation}, the dual of \eqref{eqn:cons2} is 
		\begin{align*}
		\mbox{max} &\qquad g(\bm \lambda)\\
		\mbox{subject to}& \qquad \bm  \lambda \ge 0,
		\end{align*}
		where \(g(\bm \lambda) = -\frac{1}{nT}\sum_{j=1}^n\sum_{t=0}^{T-1} h^*(- \bm Q(S_{i,t}, A_{i,t})^\tp \bm \lambda) - \langle \bm \lambda, \bm d \rangle\), \(\bm \lambda \in \mathbb{R}^{2K}\) and \(\bm d = [(\bm \delta_K + \bm l_K)^\tp, (\bm \delta_K - \bm l_K)^\tp ]^\tp \in \mathbb{R}^{2K}\). In addition, \(\bm Q(S_{i,t}, A_{i,t}) = [\hat{\bm L}_K(S_{i,t},A_{i,t})^\tp, -\hat{\bm L}_K(S_{i,t},A_{i,t})^\tp ]^\tp \in \mathbb{R}^{2K}\) and \(h^*\) is the convex conjugate of \(h\) defined as
		\begin{align*}
		h^*(t) &= \sup_w\{tw - h(w)\}\\
		& = t \tilde w - h(\tilde w),
		\end{align*}
		where \(\tilde w\) satisfies the first order condition that \(t - h'(\tilde w) = 0\). Then we obtain that \(\tilde w = (h')^{-1}(t)\) and 
		\[h^*(t) = t(h')^{-1}(t) - h\{(h')^{-1}(t)\}. \]
		Take \(\rho(t)= h^*(t)\), and it can be verified that 
		\(\rho'(t) = (h')^{-1}(t) = \tilde{w}\).
		Then the dual form can be written as
		\begin{align}
		\label{eqn:dual_double}
		\mbox{min} & \qquad \ell(\bm\lambda) = \frac{1}{nT}\sum_{j=1}^n\sum_{t=0}^{T-1} \rho(- \bm Q(S_{i,t}, A_{i,t})^\tp \bm \lambda)  + \bm \lambda^\tp \bm d\\
		\mbox{subject to}& \qquad \bm  \lambda \ge 0  \nonumber
		\end{align}
		
		Suppose \( \tilde{ \bm \lambda }= [ \tilde{ \bm \lambda }_1^\tp, \tilde{ \bm \lambda }_2^\tp]^\tp\), \(\tilde{ \bm \lambda }_1, \tilde{ \bm \lambda }_2\in\mathbb{R}^{K}\), is an optimal solution of \eqref{eqn:dual_double}. Take \(\tilde{  \lambda }_{1,k} \), \(\tilde{  \lambda }_{2,k} \) and \(\delta_{k} \) as the \(k\)-th element of \(\tilde{\bm  \lambda }_{1},  \tilde{ \bm  \lambda }_{2}\) and \(\bm \delta_K\) respectively. Next, we show that \(\tilde{ \lambda }_{1,k} \tilde{  \lambda }_{2,k} = 0 \) for \(k = 1,\dots, K\). 
		
		Suppose that there exists \(k\) such that \(\tilde{  \lambda }_{1,k} > 0 \) and \(\tilde{  \lambda }_{2,k} > 0\), then we take 
		\(\tilde{ \bm \lambda' }_1 = \tilde{ \bm \lambda }_1- \min\{\tilde{  \lambda }_{1,k}, \tilde{  \lambda }_{2,k}\} \bm e_k\), \(\tilde{ \bm \lambda' }_2 = \tilde{ \bm \lambda }_2- \min\{\tilde{  \lambda }_{1,k}, \tilde{  \lambda }_{2,k}\} \bm e_k\), where \(\bm e_k\) is a vector where its \(k\)-th entry is 1 and all remaining entries are zero. Take \(\tilde{ \bm \lambda' } = [\tilde{ \bm \lambda' }_1^\tp, \tilde{ \bm \lambda' }_2^\tp]^\tp\).
		Then \[\ell(\tilde{ \bm \lambda' }) - \ell(\tilde{ \bm \lambda }) = -2\min\{\tilde{ \bm \lambda }_{1,k}, \tilde{ \bm \lambda }_{2,k}\} \delta_k <0,\]
		due to the fact that \(\delta_k > 0\). Then it contradicts with the assumption that \(\tilde{\bm \lambda }\) is the solution.
		
		Therefore, we can take the \(\bm \lambda^+ = \tilde{ \bm \lambda }_2 - \tilde{ \bm \lambda }_{1} \).  And it can be verified that \(|\bm \lambda^+ |= \tilde{ \bm \lambda }_2 + \tilde{ \bm \lambda }_{1} \).   Rewriting \eqref{eqn:dual_double} yields the result in Theorem \ref*{thm:dual}.
	\end{proof}
	
	\subsection{Technical Proof in Section \ref{sec:thm_projection}}
	\label{sec:proof_projection}
	In this section, we provide proofs of Theorem \ref{lem:approx_g} and Corollary \ref{cor:KRR}. 
	
	To prove Theorem \ref{lem:approx_g}, 
	a novel truncation argument of the Markov chain is presented in Lemma \ref{lem:g_truncation} to obtain the tight concentration bound in the scaling of $n$ and $T$. Specifically, we truncate each trajectory into two parts in the scaling of $T$. Informally, for the first part, the truncation threshold should be small enough so that we could borrow the proof techniques under standard i.i.d. settings developed in \cite{liao2018just} without losing too much for the upper bound. For the remaining part, the idea is that as the chain grows, the distribution becomes ``exponentially'' close to the stationary distribution according to the mixing condition in Assumption \ref{assum:approx_g}(a) (or \ref{assum:RKHS}(a)). We first develop a delicate peeling argument based on \cite{farahmand_regularized_20092} to bound the empirical process under the stationary distribution, from which we could achieve the desired order in the scaling of both \(n\) and \(T\). Then it remains to bound the difference between the stationary distribution and the distribution after truncation. By carefully choosing the truncation threshold, we are able to balance the upper bounds from two parts and obtain the desired rate of convergence. See the details in Lemma \ref{lem:g_truncation}.

	\label{sec:proof_approx_g}
	\begin{proof}[Proof of Theorem \ref{lem:approx_g}]
		Take \(\delta_t^\pi(B) = \sum_{a'}\pi(a'\mid S_{t+1}) B(S_{t+1},a')\). 
		In the following, we write \(g^\pi (\cdot, \cdot; B)\) and \(\g (\cdot, \cdot; B) \) as \(g^\pi(B)\) and \(\g ( B)\) in short.  We define $$
		\mathbb{P}_{\total} F(X) =(nT)^{-1} \sum_{i=1}^n \sum_{t=0}^{T-1} F(X_{i,t})$$
		as the empirical expectation of some function \(F\) with an input \(X\), which represents random sampled trajectories from some Markov chains. 

		To proceed our proof, we can decompose the error as
		\begin{align*}
		 &\frac{1}{T}\sum_{t=0}^{T-1}\EE\left[\left\{\hat g^\pi (S_t, A_t; B) - \g(S_t, A_t;B)\right\}^2\right] \\
		 = & \frac{1}{T} \sum_{t=0}^{T-1}\EE\left[ \left\{\hat g^\pi (S_t, A_t; B) - \delta_t^\pi(B) + \delta_t^\pi(B) -  \g(S_t, A_t;B)\right\}^2\right]\\
		 = &\frac{1}{T} \sum_{t=0}^{T-1}\EE\left[ \left\{\delta_t^\pi(B) - \hat g^\pi (S_t, A_t; B)\right\}^2\right] + \frac{1}{T} \sum_{t=0}^{T-1} \EE\left[ \left\{\delta_t^\pi(B) -  \g(S_t, A_t;B)\right\}^2\right] \\
		& \qquad \quad  +  \frac{2}{T} \sum_{t=0}^{T-1}\EE\left[ \left\{\hat g^\pi (S_t, A_t; B) - \delta_t^\pi(B)\right\} \left\{\delta_t^\pi(B) -  \g(S_t, A_t;B)\right\}\right].
		\end{align*}
		Since $\sum_{t=0}^{T-1}\EE\left[ \left\{\delta_t^\pi(B)-  \g(S_t, A_t;B)\right\} g(S_t, A_t)\right] = 0$ for all $g \in \G$ due to the optimizing property of $\g$, the last term above can be simplified as
		\begin{align*}
		&  \frac{2}{T} \sum_{t=0}^{T-1} \EE\Big[ \bigdkh{\hat g^\pi (S_t, A_t; B) - \g(S_t, A_t;B)  \\
			& \qquad \qquad +   \g(S_t, A_t;B)- \delta_t^\pi(B)} \bigdkh{\delta_t^\pi(B) -  \g(S_t, A_t;B)}\Big]\\
		& = \frac{2}{T}\sum_{t=0}^{T-1} \EE\left[ \bigdkh{\g(S_t, A_t;B)- \delta_t^\pi(B)} \bigdkh{\delta_t^\pi(B) -  \g(S_t, A_t;B)}\right]\\
		& = - \frac{2}{T}\sum_{t=0}^{T-1}\EE\left[ \left\{\delta_t^\pi(B) -  \g(S_t, A_t;B)\right\}^2\right].
		\end{align*}
		As a result, we have
		\begin{align*}
		& \frac{1}{T}\sum_{t=0}^{T-1}\EE\left[\left\{\hat g^\pi (S_t, A_t; B) - \g(S_t, A_t;B)\right\}^2\right] \\
		 = &  \EE\Big[ \frac{1}{T} \sum_{t=0}^{T-1} \bigdkh{\delta_t^\pi(B) -  \hat g^\pi (S_t, A_t; B)}^2  -  \bigdkh{\delta_t^\pi(B) - \g(S_t, A_t;B)}^2\Big].
		\end{align*}
		For $g_1, g_2 \in \G, B \in  \Q$, we define the following two functions:
		\begin{align*}
		& f_1^\pi(g_1, g_2, B): (S_t, A_t, S_{t+1}) \mapsto  \left\{\delta_t^\pi(B) - g_1(S_t, A_t)\right\}^2 -  \left\{\delta_t^\pi(B) -  g_2(S_t, A_t)\right\}^2\\
		& f_2^\pi(g_1, g_2, B): (S_t, A_t, S_{t+1}) \mapsto   \left\{\delta_t^\pi(B) - g_2(S_t, A_t)\right\}\left\{g_1(S_t, A_t) - g_2(S_t, A_t)\right\},
		\end{align*}
		With these notations, we know that
		\begin{align*}
		& 
		\frac{1}{T}\sum_{t=0}^{T-1}\EE\left[\left\{\hat g^\pi (S_t, A_t; B) - \g(S_t, A_t;B)\right\}^2\right] = \EE \left[\frac{1}{T} \sum_{t=0}^{T-1} f_1^\pi\left\{\hat g^\pi(B), \g (B), B\right\}(S_t, A_t, S_{t+1})\right],  \\
		& \norm{\hat g^\pi (B) - \g (B) }_{\total}^2 = \PN \left[f_1^\pi\left\{\hat g^\pi(B), \g (B), B\right\}(S, A, S')  + 2 f_2^\pi\left\{\gn(B), \g(B), B\right\}(S, A, S')\right].
		\end{align*}
		In the following, we decompose
		\begin{align*}
		&
		\frac{1}{T}\sum_{t=0}^{T-1}\EE\left[\left\{\hat g^\pi (S_t, A_t; B) - \g(S_t, A_t;B)\right\}^2\right] + \norm{\hat g^\pi (B) - \g (B) }_{\total}^2 + \mu J_{\mathcal{G}}^2\left\{\hat g^\pi(B)\right\} 
		= I_1(B) + I_2(B),
		\end{align*}
		where 
		\begin{align*}
		& I_1(B) = 3\PN f_1^\pi\left\{\hat g^\pi(B), \g (B), B\right\} +  \mu [3J_{\mathcal{G}}^2\left\{\hat g^\pi (B)\right\} + 2J_{\mathcal{G}}^2 \left\{\g(B)\right\} + 2J_{\mathcal{Q}}^2(B)],\\
		& I_2(B) =  \PN f_1^\pi\left\{\hat g^\pi(B), \g (B), B\right\} + \frac{1}{T}\EE\sum_{t=0}^{T-1}f_1^\pi\left\{\hat g^\pi(B), \g (B), B\right\}(S_t, A_t, S_{t+1})  + \mu J_{\mathcal{G}}^2\left\{\hat g^\pi(B)\right\} \\
		& \qquad \qquad \qquad + 2 \PN f_2^\pi\left\{\gn(B), \g(B), B\right\}- I_1(B).\\
		\end{align*}
		For the first term, the optimizing property of $\hat g^\pi(B)$ implies that 
		\begin{align*}
		\frac{1}{3}I_1(B)  
		& =  \PN \Big[  \left\{\delta_t^\pi(B) - \hat g^\pi (S, A; B)\right\}^2 -  \left\{\delta^\pi(B) -  \g(S, A;B)\right\}^2\Big] \\
		& \qquad + \mu J_{\mathcal{G}}^2\left\{\hat g^\pi (B)\right\} + \frac{2}{3}\mu J_{\mathcal{G}}^2 \left\{\g(B)\right\} + \frac{2}{3}\mu J_{\mathcal{Q}}^2(B)\\
		& = \PN  \Big[  \left\{\delta_t^\pi(B) - \hat g^\pi (S, A; B)\right\}^2\Big]+ \mu J_{\mathcal{G}}^2\left\{\hat g^\pi (B)\right\} \\
		&  \qquad - \PN \big[  \left\{\delta^\pi(B) -  \g(S, A;B)\right\}^2\big] + \frac{2}{3} \mu J_{\mathcal{G}}^2 \left\{\g(B)\right\}+ \frac{2}{3}  \mu J_{\mathcal{Q}}^2(B)  \\
		& \leq \frac{5}{3} \mu J_{\mathcal{G}}^2 \left\{\g(B)\right\}+ \frac{2}{3} \mu J_{\mathcal{Q}}^2(B).
		\end{align*}
		Thus, $I_1(B) \leq 5 \mu J_{\mathcal{G}}^2 \left\{\g(B)\right\}+ 2\mu J_{\mathcal{Q}}^2(B)$ holds for all $B$.

		Next we derive the uniform bound of $I_2(B)$ over all $B$. 
		For simplicity, take 
		\begin{align}
		\label{eqn:penalty}
		\mathbf{J}^2(g_1, g_2, B) = J_{\mathcal{G}}^2(g_1) + J_{\mathcal{G}}^2(g_2) + J_{\mathcal{Q}}^2(B),
		\end{align}
		for any $g_1, g_2 \in \G$ and $B \in \Q$. 
		
		And take $f^\pi = f_1^\pi - f_2^\pi$, where
		\begin{align}
		\label{eqn:fpi}
		f^\pi(g_1, g_2, B): (S_t, A_t, S_{t+1})\mapsto (g_2-g_1)(S_t, A_t) \cdot \left(3\delta^\pi(B) - 2g_2(S_t, A_t) - g_1(S_t, A_t)\right).
		\end{align}
		Then 
		\begin{align*}
		I_2(B) &= 2(\EE - \PN)f^\pi \left\{\hat g^\pi(B), \g (B), B\right\}- \EE f^\pi \left\{\hat g^\pi(B), \g (B), B\right\}  -2 \mu\mathbf{J}^2(\hat g^\pi, g^\pi_*, B).  
		\end{align*}
		By Lemma \ref{lem:g_truncation}, we are able to show that with probability at least \(1 - 2\delta - 1/(\Total)\), 
		\begin{align*}
		I_2(B) \lesssim \mu + \frac{1}{\Total \mu^{\frac{\alpha}{1- \tau(2+\alpha)}} }  + \frac{[\log(\max\{1/\delta, \Total\})]^{1/\tau}}{\Total},
		\end{align*} 
		where \(0 < \tau \leq \frac{1}{3}\), and the leading constant depends on \(Q_{\max}, G_{\max}, \alpha, \kappa, \oldu{C_s}\). Taking \(\tau = (1+\alpha)\log(\log (\Total))/(\alpha\log (\Total))\) and \(\mu \asymp  (\Total)^{-1/(1+\alpha)} (\log (\Total))\), the result of this theorem then follow.

	\end{proof}
	
	Corollary \ref{cor:KRR} is a direct application of Theorem \ref{lem:approx_g}.  It utilizes the uniformity of Theorem \ref{lem:approx_g} and the construction of kernel ridge regression estimators. Before presenting our proof, we state our assumption for Corollary \ref{cor:KRR} below.
	\setcounter{theorem}{4}
	\begin{assumption}
		The following conditions hold.
		\label{assum:RKHS}
		\begin{enumerate}
			\item The Markov chain \(\{S_{t}, A_{t}\}_{t\ge 0}\) has a unique stationary distribution \(\mathbb{G}^*\) with density \(p^*\) over $\mathcal{S}$ and $\mathcal{A}$ and is geometrically ergodic, i.e., there exists a function \(\phi(s,a)\) and constant \(\kappa\in (0,1)\) such that,
			for any $s\in\mathcal{S}$ and $a\in\mathcal{A}$,
			\[
			\left\|\mathbb{G}^b_t(\cdot \mid (s,a)) - \mathbb{G}^*(\cdot)\right\|_{\mathrm{TV}} \leq \phi(s,a) \kappa^t, \qquad\forall t \ge 0,
			\] 
			where \(\|\cdot\|_{\mathrm{TV}}\) denotes the total variation norm, and \(\mathbb{G}_t^b(\cdot \mid (s,a))\) is the distribution of $(S_t, A_t)$ conditioned on $S_0=s$ and $A_0=a$,
			under the behavior policy. 
			Also, there exists a constant $\newu\ltxlabel{C_s}>0$ such that \(\int_{(s,a)} \phi(s,a) d\mathbb{G}_0(s,a) \leq \oldu{C_s}\), where recall that \(\mathbb{G}_0\) is the initial distribution of \((S,A)\).
			\item 
			The function class \(\mathcal{Q}\) 
		satisfies that \(\max\{ \|B\|_\infty, J_{\mathcal{Q}}(B)\} \leq Q_{\max}\) for all \(B \in \mathcal{Q}\).
			\item The function class $\mathcal{G}$ in \eqref{eqn:approx_g} is a subset of an RKHS whose corresponding RKHS norm is \(\|\cdot\|_{\Hscr}\). In addition, \(\mathcal{G}\) is a star-shaped set with center 0 and it satisfies that 
			\(\|g\|_\infty \leq G_{\max}\) for all $g\in\mathcal{G}$ and
			{\(\| g^\pi_*(\cdot,\cdot; B)\|_{\Hscr} \leq G_{\max}\)} for all \(B \in \mathcal{Q}\).
			\item 
			The regularization functional \(J_{\mathcal{G}}\) in \eqref{eqn:approx_g} is taken as \(\|\cdot\|_{\Hscr}^2\).
			Let \(\mathcal{Q}_M = \{B: B\in \mathcal{Q},  J_{\mathcal{Q}}(B)\leq M\}\) and \(\mathcal{G}_M = \{g: g\in \mathcal{G}, \|g\|_{\Hscr}^2 \leq M\}\). There exist constants  \(\oldu{entropy}\) and \(\alpha \in (0,1)\), such that for any \(\epsilon, M >0\),
			\[\max \left\{ \log \mathcal{N}(\epsilon, \mathcal{G}_M, \|\cdot\|_\infty), \log \mathcal{N}(\epsilon, \mathcal{Q}_M, \|\cdot\|_\infty)  \leq \oldu{entropy} \left( \frac{M}{\epsilon} \right)^{2\alpha}\right\}\] 
		\end{enumerate}
	\end{assumption}
	
	\begin{proof}[Proof of Corollary \ref{cor:KRR}]
		With conditions stated in the Corollary \ref{cor:KRR}, by directly applying Theorem \ref{lem:approx_g} with $\delta = \frac{1}{\Total}$, we can show that  with probability at least \(1-3/(\Total)\), 
		\begin{align*}
		&{\EE \left\{ \frac{1}{T}\sum_{t=0}^{T-1} \left[  \hat g^\pi \left(S_t, A_t;  \sum_{k=1}^K {\upsilon_k} B_k\right) - \g \left(S_t, A_t;  \sum_{k=1}^K {\upsilon_k} B_k\right) \right]^2\right\}}\\
		& \quad + \left\|\hat g^\pi \left(\cdot, \cdot;  \sum_{k=1}^K {\upsilon_k} B_k\right) - \g \left(\cdot, \cdot;  \sum_{k=1}^K {\upsilon_k} B_k\right)\right\|_{\total}^2 \\
		&  \lesssim (\Total)^{-\frac{1}{1+\alpha}} (\log (\Total)) ^{\frac{2+\alpha}{ 1 + \alpha}}.
		\end{align*}
		
		Then it suffices to show that 
		\begin{align}
		\label{eqn:KRR_solution_gstar}
		g_*^\pi(\cdot,\cdot;  {\bm \upsilon}^\tp [B_k]_{k=1}^K) =  {\bm \upsilon}^\tp [  g^\pi_*(\cdot,\cdot; B_k)]_{k=1}^K,
		\end{align}
		and 
		\begin{align}
		\label{eqn:KRR_solution_ghat}
		\hat g^\pi(\cdot,\cdot;  {\bm \upsilon}^\tp [B_k]_{k=1}^K) =  {\bm \upsilon}^\tp [  \hat g^\pi(\cdot,\cdot; B_k)]_{k=1}^K.
		\end{align}
		Note that \eqref{eqn:KRR_solution_gstar} is due to the definition of \(g_*^\pi\). Next, we verify \eqref{eqn:KRR_solution_ghat}.
		
		Denote the reproducing kernel of \(\mathcal{G}\) by \(\kernel(\cdot,\cdot)\).
		Write  \(Y^k_{i,t} = \sum_{a'} \pi(a'\mid S_{i,t+1}) B_k(S_{i,t+1}, a')\), \(\bm Y_k = [Y^k_{i,t}]_{i=1,\dots,n ,t=0,\dots,T-1} \), 
		and \(\bm G\) as the Gram matrix.  Take \[\bm G_\kernel(s,a) = [\kernel\left\{ (s,a), (S_{i,t}, A_{i,t}) \right\})]_{i=1,\dots,n ,t=0,\dots,T-1} \in \mathbb{R}^{nT}.\]
		Then \[\hat g^\pi (s,a; B_k) = \bm Y_k^\tp (\bm G+ \mu \bm I)^{-1}\bm G_\kernel(s,a)\]
		
		If we keep the tuning parameter the same when approximating all \(g^\pi_*(\cdot,\cdot; B_k)\), then for any \( {\bm \upsilon}\),
		\begin{align*}
			{\bm \upsilon}^\tp [\hat g^\pi (s,a; B_k)]_{k=1}^K  & =    {\bm \upsilon}^\tp \left[ \bm Y_k^\tp (\bm G + \mu \bm I)^{-1} \bm G_\kernel(s,a) \right]_{k=1}^K\\
		& =  [ {\bm \upsilon}^\tp \tilde{\bm Y}_{i,t}]^\tp_{i=1\dots,n,t=0,\dots,T-1} (\bm G+ \mu \bm I)^{-1} \bm G_\kernel(s,a)
		\end{align*}
		where  \(\tilde{\bm Y}_{i,t} = [Y^k_{i,t}]_{k=1}^K \in\mathbb{R}^K\), \( {\bm \upsilon}^\tp \tilde{\bm Y}_{i,t}= \sum_{k=1}^K  {\upsilon_k} \{\sum_{a'} \pi(a'\mid S_{i,t+1}) B_k(S_{i,t+1}, a')]\} = \sum_{a'} \pi(a'\mid S_{i, t+1})\{\sum_{k=1}^K \alpha_k  B_k(S_{i,t+1}, a') \} \).
		Then \(  {\bm \upsilon}^\tp [\hat g^\pi (s,a; B_k)]_{k=1}^K\) is the solution for approximating \(g^\pi_* (s,a;  {\bm \upsilon}^\tp [B_k]_{k=1}^K)\) using the same tuning parameter \(\mu\).
		Therefore we have with probability at least \(1- 3/(\Total)\), 
		\begin{align*}
		&\EE \left\{ \sum_{k=1}^K \alpha_k  \hat g^\pi \left(\cdot, \cdot;  B_k\right) -  \sum_{k=1}^K \alpha_k\g \left(\cdot, \cdot;  B_k\right) \right\} ^2 + \left\|\sum_{k=1}^K \alpha_k \hat g^\pi \left(\cdot, \cdot;  B_k\right) -  \sum_{k=1}^K \alpha_k \g \left(\cdot, \cdot; B_k\right)\right\|_{\total}^2 \\
		&  \lesssim (\Total)^{-\frac{1}{1+\alpha}} (\log (\Total)) ^{\frac{2+\alpha}{ 1 + \alpha}}.
		\end{align*}
		
	\end{proof}

	\subsection{Technical Proof4 in Section \ref{sec:weights}}
	\label{sec:proof_weights}
	In this section, we present the proof of Theorem \ref{thm:weights}, which is generalized from proofs in \cite{fan2016improving} and \cite{wang2020minimal}. First, we show the  \(\vltwo\) convergence  of \(\blam^+\) to \(\blam_*\) 
	This requires a more delicate decomposition than those in \cite{fan2016improving} and \cite{wang2020minimal} due to the additional projection error from Theorem \ref{lem:approx_g}. In addition, by the similar truncation argument as we discussed in Section \ref{sec:proof_projection}, we develop a tight matrix concentration inequality for independent Markov chains.  After obtaining the convergence of \(\blam^+\) to \(\blam_*\), the convergence of weights can be derived based on the results in Lemma \ref{thm:dual1}
	combined with Assumption \ref{assum:weights}(a) and (c). 
	\begin{proof}
		Recall that \(\blam^+\) is the solution to \eqref{eqn:dual} and \(\blam_*\) is the coefficient that leads to the best approximation to the true ratio function \(\omega^\pi\) with basis \(\bm L_K\). By  Assumption \ref{assum:weights}(a), we can show that \(\sup_{s,a}|\omega^\pi - \rho'(\bm L_K(s,a)^\tp \blam_*)| \leq \oldu{weightsbound}(K^{-r_1})\). 
		In the following,  we study the convergence of \(\|\blam^+ - \blam_*\|_2\).  We aim to prove 
		\begin{align}
		\label{eqn:lambda_convergence}
		\|\blam^+ - \blam_*\|_2 = \bigOp\left[ \{\psi(K)\}^{-1} \left(  \frac{\sqrt{K}\log K}{\sqrt {\Total}} + K^{1/2-r_1}+ \sqrt{K}\zeta_{\total}\right)\right].
		\end{align}

		Define the objective function \(G(\blam) = \frac{1}{nT}\sum_{i=1}^n\sum_{t=0}^{T-1} \rho(\hat{\bm L}_K(S_{i,t}, A_{i,t})^\tp\bm \lambda) + |\bm \lambda|^\tp \bm \delta_K  - \bm \lambda^\tp \bm l_K\). In order to show the above bound in \eqref{eqn:lambda_convergence}, by the continuity and convexity of \(G(\cdot)\), it suffices to show that,  with high probability, 
		\[ \inf_{\Delta \in \mathcal{C} }\{G(\blam_* + \Delta) - G(\blam_*)\} > 0,\]
		where 
		\(\mathcal{C} = \{\Delta \in \mathbb{R}^K: \|\Delta\|_2 = \newu\ltxlabel{c_delta}\{\psi(K)\}^{-1}(  {\sqrt{K}\log (\Total)}/{\sqrt {\Total}} + K^{-r_1}+ \sqrt{K}\zeta_{\total})\)\}.
		Here \(\oldu{c_delta}\) is some appropriate constant. 
		
		First, we take the following decomposition:
		\begin{align}
		&G(\blam_* + \Delta) - G(\blam_*) \\
		&= \frac{1}{nT}\sum_{i=1}^n\sum_{t=0}^{T-1}\left[ \rho\{\hat{\bm L}_K(S_{i,t}, A_{i,t})^\tp(\blam_* + \Delta)\} - \rho\{\hat{\bm L}_K(S_{i,t}, A_{i,t})^\tp\bm \lambda_*\} \right] + \bm \delta_K^\tp (|\blam_*+ \Delta| - |\blam_*|) - \Delta^\tp \bm l_K \nonumber \\
		&\ge \frac{1}{nT}\sum_{i=1}^n\sum_{t=0}^{T-1} \left[\rho'\{\hat{\bm L}_K(S_{i,t}, A_{i,t})^\tp\bm \lambda_*\}\{ \hat{\bm L}_K(S_{i,t}, A_{i,t})^\tp\Delta\} + \frac{1}{2} \rho^{''}\{\hat v(S_{i,t}, A_{i,t})\}\{\hat{\bm L}_K(S_{i,t}, A_{i,t})^\tp \Delta\}^2\right] \nonumber\\
		&\quad- \Delta^\tp \bm l_K  - |\Delta|^\tp \bm \delta_K, \label{eqn:G_delta}
		\end{align}
		where 
		\(\hat v(S_{i,t}, A_{i,t})\) is a value between \(\hat{\bm L}_K(S_{i,t}, A_{i,t})^\tp(\blam_* + \Delta)\) and \(\hat{\bm L}_K(S_{i,t}, A_{i,t})^\tp\blam_* \).

		Define \(\hat{\bm \xi}(s,a) =  \hat{\bm L}_K(s, a)- \bm L_K(s,a) \). Now  we focus on the following term:
		\begin{align}
		&\frac{1}{nT}\sum_{i=1}^n\sum_{t=0}^{T-1} \rho'\{\hat{\bm L}_K(S_{i,t}, A_{i,t})^\tp\bm \lambda_*\}\{ \hat{\bm L}_K(S_{i,t}, A_{i,t})^\tp\Delta\} - \Delta^\tp \bm l_K \nonumber \\
		= & \frac{1}{nT}\sum_{i=1}^n\sum_{t=0}^{T-1} \rho'\left[ \bm L_K(S_{i,t}, A_{i,t})^\tp \bm \lambda_* + \blam_*^\tp \hat{\bm \xi}(S_{it}, A_{it}) \right]\{ \bm L_K(S_{i,t}, A_{i,t})^\tp \Delta+ \hat{\bm \xi}(S_{i,t}, A_{i,t})^\tp\Delta\} - \Delta^\tp \bm l_K \nonumber\\
		= & \frac{1}{nT}\sum_{i=1}^n\sum_{t=0}^{T-1} \left[ \rho'\left\{ \bm L_K(S_{i,t}, A_{i,t})^\tp \bm \lambda_*  \right\}\left\{ \bm L_K(S_{i,t}, A_{i,t})^\tp \Delta \right\} - \Delta^\tp\bm l_K \right] \label{eqn:t1}\\
		+ & \frac{1}{nT}\sum_{i=1}^n\sum_{t=0}^{T-1}   \rho'\left\{ \bm L_K(S_{i,t}, A_{i,t})^\tp \bm \lambda_*  \right\} \left\{ \hat{\bm \xi}(S_{i,t}, A_{i,t})^\tp \Delta \right\}\label{eqn:t2}\\
		+ & \frac{1}{nT}\sum_{i=1}^n\sum_{t=0}^{T-1}  \rho^{''}\{\tilde v(S_{i,t}, A_{i,t})\}\left\{ \hat{\bm \xi}(S_{i,t}, A_{i,t})^\tp \blam_* \right\}\left\{ \bm L_K(S_{i,t}, A_{i,t})^\tp \Delta \right\} \label{eqn:t3}\\
		+ &  \frac{1}{nT}\sum_{i=1}^n\sum_{t=0}^{T-1}  \rho^{''}\{\tilde v(S_{i,t}, A_{i,t})\}\left\{ \hat{\bm \xi}(S_{i,t}, A_{i,t})^\tp \blam_* \right\}\left\{\hat{\bm \xi}(S_{i,t}, A_{i,t})^\tp \Delta \right\} \label{eqn:t4},
		\end{align}
		where 
		\(\tilde v(S_{i,t}, A_{i,t})\) is a value between \(\bm L_K(S_{i,t}, A_{i,t})^\tp \bm \lambda_* \) and \(\hat{\bm L}_K(S_{i,t}, A_{i,t})^\tp \bm \lambda_*\). In the following, we show how to control \eqref{eqn:t1}-\eqref{eqn:t4} one by one.

		\begin{itemize}
			\item For \eqref{eqn:t1}, 
			\begin{align}
			\eqref{eqn:t1} & = \frac{1}{nT}\sum_{i=1}^n\sum_{t=0}^{T-1} \left[ \omega^\pi(S_{i,t}, A_{i,t})\left\{ \bm L_K(S_{i,t}, A_{i,t})^\tp \Delta \right\} - \Delta^\tp\bm l_K \right] \nonumber\\
			& \quad +   \frac{1}{nT}\sum_{i=1}^n\sum_{t=0}^{T-1} \left[  \rho'\left\{ \bm L_K(S_{i,t}, A_{i,t})^\tp \bm \lambda_*  \right\} -  \omega^\pi(S_{i,t}, A_{i,t}) \right]\left\{ \bm L_K(S_{i,t}, A_{i,t})^\tp \Delta \right\} \nonumber\\
			& \ge  -\left\| \frac{1}{nT}\sum_{i=1}^n\sum_{t=0}^{T-1} \omega^\pi(S_{i,t}, A_{i,t})\bm L_K(S_{i,t}, A_{i,t}) - \bm l_K\right\|_2 \left\| \Delta\right\|_2 \nonumber\\ 
			&\quad - \bigO(K^{-r_1}) \sqrt{ \Delta \left\{\frac{1}{nT}\sum_{i=1}^n\sum_{t=0}^{T-1} \bm L_K(S_{i,t}, A_{i,t})\bm L_K(S_{i,t}, A_{i,t})^\tp  \right\}\Delta} \nonumber\\
			& \ge  -\left\| \frac{1}{nT}\sum_{i=1}^n\sum_{t=0}^{T-1} \omega^\pi(S_{i,t}, A_{i,t})\bm L_K(S_{i,t}, A_{i,t}) - \bm l_K\right\|_2 \left\| \Delta\right\|_2 \nonumber\\ 
			&\quad - \bigO(K^{-r_1}) \left[ \lambda_{\max}\left\{ \frac{1}{nT}\sum_{i=1}^n\sum_{t=0}^{T-1} \bm L_K(S_{i,t}, A_{i,t})\bm L_K(S_{i,t}, A_{i,t})^\tp \right\} \right]^{1/2}\|\Delta\|_2 \nonumber\\
			& \ge  -\left\| \frac{1}{nT}\sum_{i=1}^n\sum_{t=0}^{T-1} \omega^\pi(S_{i,t}, A_{i,t})\bm L_K(S_{i,t}, A_{i,t}) - \bm l_K\right\|_2 \left\| \Delta\right\|_2  - \bigOp(K^{-r_1})\|\Delta\|_2\label{eqn:weights_part1}.
			\end{align}
			The first inequality is based on Cauchy–Schwarz inequality and Assumption \ref{assum:weights}(a). 
			For the last inequality, we apply Lemma \ref{lem:maxeigen} which yields
			\begin{align*}
			& \lambda_{\max}\left\{ \frac{1}{nT}\sum_{i=1}^n\sum_{t=0}^{T-1} \bm L_K(S_{i,t}, A_{i,t})\bm L_K(S_{i,t}, A_{i,t})^\tp   \right\}
			\leq \oldu{weightsC4} + \smallOp(1),
			\end{align*}
			due to Assumptions \ref{assum:RKHS}(a) and \ref{assum:weights}(c). 

			Next, we show how to bound \(\| (nT)^{-1}\sum_{i=1}^n\sum_{t=0}^{T-1} \omega^\pi(S_{i,t}, A_{i,t})\bm L_K(S_{i,t}, A_{i,t}) - \bm l_K\|_2\) in \eqref{eqn:weights_part1}.
			Note that 
			\[\bm l_K = \EE \left\{ \frac{1}{T}\sum_{t=0}^{T-1} \omega^\pi(S_{i,t}, A_{i,t})\bm L_K(S_{i,t}, A_{i,t}) \right\}.\]
			Then we apply the matrix concentration inequality developed in Lemma \ref{lem:weights_concen} to bound it.  Specifically,
			we take \(\bm F(s,a) = \omega^\pi(s,a) \bm L_K(s,a)\) in Lemma \ref{lem:weights_concen}. By Assumption \ref{assum:weights}(a) and (c),  we can show that \(\sup_{(s,a)}\|\bm F(s,a)\|_2 \leq \|\omega^\pi\|_\infty \left\{ \sup_{(s,a)}\|\bm L_K(s,a)\|_2 \right\} \leq \oldu{weightsC1}\oldu{weightsC2}(1+ \gamma)\sqrt{K}\). Therefore, Lemma \ref{lem:weights_concen} implies that
			\begin{align*}
			\left\| \frac{1}{nT}\sum_{i=1}^n\sum_{t=0}^{T-1} \omega^\pi(S_{i,t}, A_{i,t})\bm L_K(S_{i,t}, A_{i,t}) - \bm l_K\right\|_2  = \bigOp\left( \frac{\sqrt{K}\log (\Total)}{\sqrt{\Total}} \right).
			\end{align*}

			Thus we can show that
			\begin{align}
			\eqref{eqn:t1} \ge -\bigOp\left( \frac{\sqrt{K}\log (\Total)}{\sqrt{\Total}} \right)\|\Delta\|_2 - \bigO(K^{- r_1}) \|\Delta\|_2 \label{eqn:t1c}
			\end{align}
			
			\item \textbf{For \eqref{eqn:t2}},
			\begin{align}
			\eqref{eqn:t2} 
			& \ge -C\sup_{s,a}\left|\rho'\left\{ \bm L_K(s, a)^\tp \bm \lambda_*  \right\} \right| \sqrt{\frac{1}{nT}\sum_{i=1}^n\sum_{t=0}^{T-1} \left\{\hat{\bm \xi}(S_{i,t}, A_{i,t})^\tp \bm \Delta \right\}^2_2}  \nonumber\\
			& \ge -C(\oldu{weightsC1} + K^{-r_1}) \|\Delta\|_2 \gamma \left\| \hat g^\pi\left(\cdot, \cdot; \bm B_K^\tp \frac{\Delta}{\|\Delta\|_2}\right) - g^\pi_* \left(\cdot, \cdot; \bm B_K^\tp \frac{\Delta}{\|\Delta\|_2}\right)\right\|_{\total},  \nonumber
			\end{align} 
			where the first equality is given by Cauchy-Schwarz inequality and the second one is due to Assumption \ref{assum:weights}(a) and the definition of \(\hat {\bm \xi}\).
			As \(\sup_{s,a} \sup_{\|\bm \delta\|\leq 1 }\|\bm B_K^\tp \bm \delta\|_\infty \lesssim  \sqrt{K}\), one can show that \({K}^{-1/2} \bm B_K^\tp \Delta/\|\Delta\|_2\lesssim 1 \). Due to Theorem \ref{lem:approx_g} and Assumption \ref{assum:weights}(e), we have 
			\begin{align*}
			K^{-1/2}\left\| \hat g^\pi\left(\cdot, \cdot; \bm B_K^\tp \frac{\Delta}{\|\Delta\|_2}\right) - g^\pi_* \left(\cdot, \cdot; \bm B_K^\tp \frac{\Delta}{\|\Delta\|_2}\right)\right\|_{\total} = \bigOp\left( \zeta_{\total}  \right).
			\end{align*}

			Therefore, 
			\begin{align}
			\eqref{eqn:t2} & \ge -C\oldu{weightsC1}  \sqrt{K}\bigOp\left(\zeta_{\total} \right)\|\Delta\|_2,\label{eqn:t2c}
			\end{align}
			where \(C\) is some positive constant.

			\item \textbf{For \eqref{eqn:t3}}, 
			note that 
			\begin{align}
			\left|\bm L_K(S_{i,t}, A_{i,t})^\tp \bm \lambda_* \right| &= \left|\blam_*\bm B_K(S_{i,t}, A_{i,t})^\tp - \gamma g^\pi_*(S_{i,t}, A_{i,t}; \blam_*^\tp \bm\beta_K)\right|\nonumber \\
			& \leq \left| \bm \lambda_*\bm B_K(S_{i,t}, A_{i,t})^\tp\right| + \left| \gamma g^\pi_*(S_{i,t}, A_{i,t}; \blam_*^\tp \bm\beta_K)\right|\nonumber\\
			& \leq Q_{\max} + \gamma G_{\max}, \nonumber
			\end{align}
			due to Assumption \ref{assum:weights}(e).
			Similarly, we can show that 
			\[\left|\hat{\bm L}_K(S_{i,t}, A_{i,t})^\tp \bm \lambda_* \right| \leq Q_{\max} + \gamma G_{\max}. \]
			As \(\hat g^\pi (\cdot, \cdot; \blam_*^\tp \bm\beta_K) \in \mathcal{G}\) by assumption,
			\(|\tilde{v}(S_{i,t}, A_{i,t})| \leq Q_{\max} + \gamma G_{\max}\). 
			Since \(\rho{''}\) is a strictly positive and continuous function, there exists a constant \(\newl\ltxlabel{wc4} > 0\) such that \(\rho{''}\{\tilde{v}(S_{i,t}, A_{i,t})\} \leq \oldl{wc4}\).
			We then have
			\begin{align}
			\eqref{eqn:t3} &\ge -\sqrt{\frac{1}{nT}\sum_{i=1}^n\sum_{t=0}^{T-1} \rho^{''}\{\tilde{v}(S_{i,t}, A_{i,t})\}^2 \left\{ \hat{\bm \xi}(S_{i,t}, A_{i,t})^\tp \blam_* \right\}^2} \sqrt{\frac{1}{nT}\sum_{i=1}^n\sum_{t=0}^{T-1} \left\{ \bm L_K(S_{i,t}, A_{i,t}) ^\tp \Delta \right\}^2}\nonumber\\
			& \ge -\oldl{wc4} \bigOp(\zeta_{\total}) \left[ \lambda_{\max}\left\{ \frac{1}{nT} \sum_{i=1}^n\sum_{t=0}^{T-1}  \bm L_K(S_{i,t}, A_{i,t})  \bm L_K(S_{i,t}, A_{i,t})^\tp  \right\} \right]^{1/2} \|\Delta\| \nonumber\\
			& \ge -C\oldl{wc4}\bigOp(\zeta_{\total}) \left( \left[ \lambda_{\max}\left\{\EE \frac{1}{nT} \sum_{i=1}^n\sum_{t=0}^{T-1}  \bm L_K(S_{i,t}, A_{i,t})  \bm L_K(S_{i,t}, A_{i,t})^\tp  \right\} \right]^{1/2}  + \bigOp(1)\right) \|\Delta\|_2 \nonumber\\
			& \ge -C\oldl{wc4}\sqrt{\oldu{weightsC4}}\bigOp(\zeta_{\total})\|\Delta\|_2, \label{eqn:t3c}
			\end{align}
			where \(C\) is some positive constant.  The first inequality of the above arguments is given by Cauchy-Schwarz inequality. The second one uses results in Theorem \ref{lem:approx_g} based on Assumption \ref{assum:weights}(e) so that the upper bound for \((nT)^{-1}\sum_{i=1}^n\sum_{t=0}^{T-1}  \{ \hat{\bm \xi}(S_{i,t}, A_{i,t})^\tp \blam_* \}^2\) is obtained. We use Lemma \ref{lem:maxeigen} 
			in the third inequality, while the last inequality is obtained due to Assumption \ref{assum:weights}(c).
			
			\item \textbf{For \eqref{eqn:t4}}, by using the same argument showing \eqref{eqn:t3c} and \eqref{eqn:t2c}, we have 
			\begin{align}
			\eqref{eqn:t4} &\ge -\sqrt{\frac{1}{nT}\sum_{i=1}^n\sum_{t=0}^{T-1} \rho^{''}\{\tilde{v}(S_{i,t}, A_{i,t})\}^2 \left\{ \hat{\bm \xi}(S_{i,t}, A_{i,t})^\tp \blam_* \right\}^2} \sqrt{\frac{1}{nT}\sum_{i=1}^n\sum_{t=0}^{T-1} \left\{ \hat{\bm \xi}(S_{i,t}, A_{i,t}) ^\tp \Delta \right\}^2}\nonumber\\
			& \ge -\oldl{wc4} \bigOp(\zeta_{\total}) \sqrt{K} \bigOp(\zeta_{\total}) \|\Delta\|_2 \nonumber\\
			& \ge -\oldl{wc4}\sqrt{K}\bigOp(\zeta^2_{\total})\|\Delta\|_2 \label{eqn:t4c} 
			\end{align}

		\end{itemize}

		Substituting \eqref{eqn:t1c}, \eqref{eqn:t2c}, \eqref{eqn:t3c} and \eqref{eqn:t4c} into \eqref{eqn:t1}, \eqref{eqn:t2}, \eqref{eqn:t3} and \eqref{eqn:t4} respectively,  we have
		\begin{align}
		&\frac{1}{nT}\sum_{j=1}^n\sum_{t=0}^{T-1} \rho'\{\hat{\bm L}_K(S_{i,t}, A_{i,t})^\tp\bm \lambda_*\}\{ \hat{\bm L}_K(S_{i,t}, A_{i,t})^\tp\Delta\} - \Delta^\tp \bm l_K \nonumber \\
		\ge& -\bigOp\left( \frac{\sqrt{K}\log (\Total)}{\sqrt {\Total}} + K^{-r_1}+ \sqrt{K} \zeta_{\total} + \sqrt{K}\zeta_{\total}^2\right)\|\Delta\|_2.
		\end{align}
		
		In the following, we discuss the remaining term in  \eqref{eqn:G_delta}.
		Due to Assumption \ref{assum:weights}(b), \(\rho^{''}\{\hat v(S_{i,t}, A_{i,t})\} \ge \oldu{weightssecond}\).
		
		Next, we show how to bound 
		\[\lambda_{\min}\left[ \frac{1}{nT}\sum_{i=1}^n\sum_{t=0}^{T-1} \hat {\bm L}_K(S_{i,t}, A_{i,t}) \hat {\bm L}_K(S_{i,t}, A_{i,t})^\tp \right].\]
		First, note that for any vector \(\bm a \in \mathbb{R}^K\) such that \(\|\bm a\|_2 = 1\), we have 
		\begin{align}
		& \bm a ^\tp  \left[ \frac{1}{nT}\sum_{i=1}^n\sum_{t=0}^{T-1} \hat {\bm L}_K(S_{i,t}, A_{i,t}) \hat {\bm L}_K(S_{i,t}, A_{i,t})^\tp \right] \bm a = \frac{1}{nT}\sum_{i=1}^n\sum_{t=0}^{T-1}    \left\| \bm a ^\tp  \hat{\bm L}_K(S_{i,t}, A_{i,t})\right\|^2_2 \nonumber\\
		\ge & \frac{1}{nT}\sum_{i=1}^n\sum_{t=0}^{T-1}    \left\| \bm a ^\tp  {\bm L}_K(S_{i,t}, A_{i,t})\right\|^2_2  - \frac{1}{nT}\sum_{i=1}^n\sum_{t=0}^{T-1}   \left\| \bm a ^\tp  \left\{ \hat{\bm L}_K(S_{i,t}, A_{i,t})-{\bm L}_K(S_{i,t}, A_{i,t}) \right\}\right\|^2_2 \nonumber\\
		\ge &\bm a ^\tp  \left[ \frac{1}{nT}\sum_{i=1}^n\sum_{t=0}^{T-1} {\bm L}_K(S_{i,t}, A_{i,t})  {\bm L}_K(S_{i,t}, A_{i,t})^\tp \right] \bm a   - \bigOp(K\zeta^2_{\total}), \label{eqn:lambda_min1}
		\end{align}
		where the last inequality can be derived by the same arguments in proving \eqref{eqn:t2}.
		
		Then we have
		\begin{align*}
		&\lambda_{\min}\left[ \frac{1}{nT}\sum_{i=1}^n\sum_{t=0}^{T-1} \hat {\bm L}_K(S_{i,t}, A_{i,t}) \hat {\bm L}_K(S_{i,t}, A_{i,t})^\tp  \right] \\
		\ge &  \lambda_{\min}\left[ \frac{1}{nT}\sum_{i=1}^n\sum_{t=0}^{T-1} {\bm L}_K(S_{i,t}, A_{i,t}) {\bm L}_K(S_{i,t}, A_{i,t})^\tp  \right]- \bigOp(K\zeta^2_{\total}) \\
		\ge & \lambda_{\min}\left[ \frac{1}{T}\sum_{t=0}^{T-1} \EE {\bm L}_K(S_{t}, A_{t}) {\bm L}_K(S_{t}, A_{t})^\tp  \right] - \bigOp\left( \sqrt{\frac{K}{\Total}}\log (\Total) \right)- \bigOp(K\zeta^2_{\total}) \\
		\ge & \psi(K) - \smallOp(\psi(K)).
		\end{align*}
		The first inequality is due to \eqref{eqn:lambda_min1}. The second inequality is by Lemma \ref{lem:maxeigen}. 
		For the third inequality, due to the condition  that \(\{\psi(K)\}^{-1} \sqrt{K}\xi_{\total} = \smallO(1)\), then \(\sqrt{\frac{K}{\Total}}\log (\Total)  = \smallO(\psi(K))\) and \(K\xi_{\total}^2 = \smallO({\psi(K)}^2) = \smallO(\psi(K))\).

		Now, returning to \eqref{eqn:G_delta},  we can show that
		\begin{align}
		G(\blam_* + \Delta) - G(\blam_*) 
		&\ge -\bigOp\left( \frac{\sqrt{K}\log (\Total)}{\sqrt {\Total}} + K^{-r_1}+ \sqrt{K}\zeta_{\total}\right)\|\Delta\|_2 \nonumber\\
		+ &\frac{1}{2nT}\sum_{i=1}^n\sum_{t=0}^{T-1} \rho^{''}\{v'(S_{i,t}, A_{i,t})\}\{\hat{\bm L}_K(S_{i,t}, A_{i,t})^\tp \Delta\}^2 - \|\Delta\|_2\|\bm \delta_K\|_2\nonumber\\
		\ge& -\bigOp\left( \frac{\sqrt{K}\log (\Total)}{\sqrt {\Total}} + K^{-r_1}+ \sqrt{K}\zeta_{\total}\right)\|\Delta\|_2 
		+ \left\{\psi(K)  - \smallOp(\psi(K))\right\}\|\Delta\|_2^2,\nonumber
		\end{align} 
		where the last inequality is due to the condition of \(\bm \delta_K\) specified in Assumption \ref{assum:weights}(g). As long as \(\|\Delta\|_2 \ge \oldu{c_delta} \{\psi(K)\}^{-1}(\frac{\sqrt{K}\log (K\Total)}{\sqrt {\Total}} + K^{-r_1}+ \sqrt{K}\zeta_{\total})\) for some large enough constant \(\oldu{c_delta}\), with high probability, \[G(\blam_* + \Delta) - G(\blam_*) > 0. \] Therefore we have proved \eqref{eqn:lambda_convergence}.

		Finally, we are ready to show the convergence of \(\hat \omega\) given below. 
		\begin{align}
		&[\EE \left\{ \hat\omega^\pi(\cdot,\cdot) - \omega^\pi(\cdot,\cdot)\right\}^2]^{1/2}\nonumber\\
		\leq  & \left[\EE \left\{\rho'(\hat {\bm L}_K(\cdot,\cdot)^\tp \blam^+) - \rho'( {\bm L}_K(\cdot,\cdot)^\tp \blam_*)\right\}^2\right]^{1/2} + \sup_{s,a} \left|\omega^\pi(s,a) - \rho'( {\bm L}_K(s,a)^\tp \blam_*)\right|\nonumber\\
		\leq & \left[\EE \left\{\rho'({\bm L}_K(\cdot,\cdot)^\tp \blam^+) - \rho'( {\bm L}_K(\cdot,\cdot)^\tp \blam_*)\right\}^2\right]^{1/2}  + \left[\EE \left\{\rho^{''}\{\tilde v(\cdot,\cdot)\} \left\{ \hat {\bm L}_K(\cdot,\cdot) - {\bm L}_K(\cdot,\cdot) \right\}^\tp \blam^+\right\}^2\right]^{1/2}+ \bigO(K^{-r_1})\nonumber\\
		\leq & \oldl{wc4} \left[\EE \left\{{\bm L}_K(\cdot,\cdot)^\tp (\blam^+ - \blam_*)\right\}^2\right]^{1/2}  + \oldl{wc4} \bigOp(\zeta_{\total}) + \bigO(K^{-r_1}) \nonumber \\
		\leq & \oldl{wc4} \|(\blam^+ - \blam_*)\|_2 \left[ \lambda_{\max}\left\{  \EE  \left[\frac{1}{T}\sum_{t=0}^{T-1}\bm L_K(S_t, A_t) \bm L_K(S_t, A_t)^\tp \right] \right\}  \right]^{1/2} + \bigOp(\zeta_{\total}) + \bigO(K^{-r_1}) \nonumber \\
		\leq & \oldl{wc4} \sqrt{\oldu{weightsC4}} \|(\blam^+ - \blam_*)\|_2 + \bigOp(\zeta_{\total}) + \bigO(K^{-r_1}) \nonumber \\
		\leq & \bigOp\left[ \{\psi(K)\}^{-1}\left(  \frac{\sqrt{K}\log (\Total)}{\sqrt {\Total}} + K^{-r_1}+ \sqrt{K}\zeta_{\total}\right) \right].\nonumber
		\end{align}
		The second inequality is based on Assumption \ref{assum:weights}(a) and the mean value theorem.  For the third inequality, we adopt the mean value theorem and Assumption \ref{assum:weights}(b) again, using the similar arguments for proving \eqref{eqn:t3} to show the boundedness of \(\rho^{''}\). Assumption \ref{assum:weights}(e) is used to obtain the desired \(\bigOp(\zeta_{\total})\) order. The fifth inequality results from Assumption \ref{assum:weights}(c) and the last inequality is due to \eqref{eqn:lambda_convergence}.
		
		Noting that
		\[\lambda_{\max}\left\{ \frac{1}{nT}\sum_{i=1}^n \sum_{t=0}^{T-1}\bm L_K(S_{i,t}, A_{i,t}) \bm L_K(S_{i,t}, A_{i,t})^\tp  \right\} = \lambda_{\max}\left\{ \EE  \left[\frac{1}{T}\sum_{t=0}^{T-1}\bm L_K(S_t, A_t) \bm L_K(S_t, A_t)^\tp \right] \right\}  + \smallOp(1)\]
		by Lemma \ref{lem:maxeigen}, we then could use the similar argument to bound \(\| \hat \omega^\pi(\cdot,\cdot) - \omega^\pi(\cdot,\cdot)\|_{\total}\), which completes the proof.
	\end{proof}

	\subsection{Technical Proof in Section \ref{sec:efficiency}}
	\label{sec:proof_asymp}
	In this section, we provide the proof of Theorem \ref{thm:asymp}, which shows the efficiency of our weighted estimator. In the proof, we first decompose the estimation error incurred by \(\hat \calV(\pi)\) to the true $\calV(\pi)$. Based on the decomposition, we show that with the help of convergence results regarding to the projection step and the final weights, together with the approximation error to  \(Q^\pi\), the desired \(\sqrt{nT}\) convergence rate of $\hat \calV(\pi)$ can be established. Finally, the efficiency of our estimator can be achieved.
	\begin{proof}[Proof of Theorem \ref{thm:asymp}]
		We mainly prove Result (ii). Result (i) can be derived through similar arguments but under a different condition for \(K\). Thus, we omit that for brevity.
		To start with, we derive the following decomposition
		\begin{align*}
		& \hat \calV(\pi) - \calV(\pi) \\
		= &\frac{1}{nT}\sum_{i=1}^n\sum_{t=0}^{T-1}\hat \omega^\pi_{i,t} \left[ Q^\pi(S_{i,t},A_{i,t}){ - }\gamma g_*^\pi(S_{i,t}, A_{i,t}; Q^\pi) \right] 
		{-} (1{-}\gamma) \EE_{S_0\sim \mathbb{G}} \left[ \sum_{a'\in \Ascr} \pi(a'\mid S_{0}) Q^\pi(S_{0},a') \right]\nonumber\\
		& + \frac{1}{nT}\sum_{i=1}^n\sum_{t=0}^{T-1} \hat \omega^\pi_{i,t}\epsilon_{i,t}\\
		= &\frac{1}{nT}\sum_{i=1}^n\sum_{t=0}^{T-1}\hat \omega^\pi_{i,t} \left[ \bm B_K(S_{i,t},A_{i,t})^\tp - \gamma g_*^\pi(S_{i,t}, A_{i,t}; \bm B_K)^\tp  \right]\bm \beta  - (1-\gamma)\EE_{S_0\sim \mathbb{G}} \left[ \sum_{a'\in \Ascr} \pi(a'\mid S_{0}) \bm B_K(S_{0},a')^\tp \bm \beta \right]\\
		&+ \frac{1}{nT}\sum_{i=1}^n\sum_{t=0}^{T-1} \hat \omega^\pi_{i,t}\left[ \Delta_Q(S_{i,t}, A_{i,t}) - \gamma g^\pi_*(S_{i,t}, A_{i,t}, \Delta_Q) \right] \\
		&- (1-\gamma)\EE_{S_0\sim \mathbb{G}} \left[ \sum_{a'\in \Ascr} \pi(a'\mid S_{0}) \Delta(S_{0},a') \right] + \frac{1}{nT}\sum_{i=1}^n\sum_{t=0}^{T-1}\hat \omega^\pi_{i,t}\epsilon_{i,t}\\
		= &\frac{1}{nT}\sum_{i=1}^n\sum_{t=0}^{T-1} \hat \omega^\pi_{i} \left[ \bm B_K(S_i,A_i)^\tp - \gamma \hat g^\pi(S_i, A_i; \bm B_K) ^\tp \right]\bm \beta  - (1-\gamma)\EE_{S_0\sim \mathbb{G}} \left[ \sum_{a'\in \Ascr} \pi(a'\mid S_{0}) \bm B_K(S_{0},a')^\tp \bm \beta \right] \cdots \cdots (\mathrm{\Rnum{1}}) \\
		& + \frac{1}{nT}\sum_{i=1}^n\sum_{t=0}^{T-1} \omega^\pi_{i,t} \gamma \left[ \hat g^\pi(S_{i,t}, A_{i,t}; \bm B_K) - g_*^\pi(S_{i,t}, A_{i,t}; \bm B_K) \right]^\tp\bm \beta  \cdots \cdots (\mathrm{\Rnum{2}})\\
		&+  \frac{1}{nT}\sum_{i=1}^n\sum_{t=0}^{T-1} (\hat \omega^\pi_{i,t}  - \omega^\pi_{i,t})\gamma \left[ \hat g^\pi(S_{i,t}, A_{i,t}; \bm B_K) - g_*^\pi(S_{i,t}, A_{i,t}; \bm B_K) \right]^\tp\bm \beta  \cdots \cdots (\mathrm{\Rnum{3}})\\
		&+ \frac{1}{nT}\sum_{i=1}^n\sum_{t=0}^{T-1}\omega^\pi_{i,t}\left[ \Delta_Q(S_{i,t}, A_{i,t}) - \gamma g^\pi_*(S_{i,t}, A_{i,t}, \Delta_Q) \right] - (1-\gamma)\EE_{S_0\sim \mathbb{G}} \left[ \sum_{a'\in \Ascr} \pi(a'\mid S_{0}) \Delta(S_{0},a') \right] \cdots \cdots (\mathrm{\Rnum{4}})\\
		& + \frac{1}{nT}\sum_{i=1}^n\sum_{t=0}^{T-1} (\hat \omega^\pi_{i,t} - \omega^\pi_{i,t})\left[ \Delta_Q(S_{i,t}, A_{i,t}) - \gamma g^\pi_*(S_{i,t}, A_{i,t}, \Delta_Q) \right]\cdots \cdots (\mathrm{\Rnum{5}}) \\
		&  + \frac{1}{nT}\sum_{i=1}^n\sum_{t=0}^{T-1}\hat \omega^\pi_{i,t}\epsilon_{i,t}\cdots \cdots (\mathrm{\Rnum{6}})
		\end{align*}
		In the following, we analyze (I)-(VI) components separately. 
		\begin{itemize}
			\item \textbf{For (\Rnum{1})}, from the optimization constraint \eqref{eqn:cons},
			\[(\mathrm{\Rnum{1}}) \leq  {\|\bm \beta \circ \bm \delta_{K}\|_1}
			.\]
			By the condition that 
			in Assumption \ref{assum:additional}(a), we can show \((\mathrm{\Rnum{1}})   = \smallO((\Total)^{-1/2})\). 
			
			\item \textbf{For (\Rnum{3})}, we control it by utilizing the convergence of weights and function \(\hat g^\pi\). Note that
			\begin{align*}
			(\mathrm{\Rnum{3}})  &\leq  \gamma  \|\hat \omega^\pi - \omega^\pi\|_{\total} \|\{g_*^\pi (\cdot, \cdot; \bm B_K) - \hat g(\cdot, \cdot; \bm B_K)\}^\tp \bm \beta \|_{\total} \\
			& = \bigOp \left(  \frac{\sqrt{K}\log (\Total)}{\sqrt {\Total}} + K^{-r_1}+ \sqrt{K}\zeta_{\total}  \right) \bigOp\left( \zeta_{\total} \right) = \smallOp((\Total)^{-1/2}),
			\end{align*}
			due to the conditions in Corollary \ref{cor:KRR} and conditions for \(K\) in Theorem \ref{thm:asymp} Result (ii). 
			
			\item \textbf{For (\Rnum{4})}, it can be seen that the mean is zero. In addition, \(\sup_{s,a}|\Delta_Q(s,a)| = \bigO(K^{-r_2})\) by Assumption \ref{assum:additional}(a),  and \(\omega^\pi(s,a)\) is bounded above by Assumption \ref{assum:weights}(a).  Applying similar arguments in the proof of Theorem \ref{thm:weights}, we can show that  \[(\mathrm{\Rnum{4}}) = \bigOp((\Total)^{-1/2}K^{-r_2}) = \smallOp((\Total)^{-1/2}).\]
			\item \textbf{For (\Rnum{5})}, we will control it by the convergence of weights and also the magnitude of \(\Delta_Q\).
			\begin{align*}
			(\mathrm{\Rnum{5}})&\leq \|\hat \omega^\pi - \omega^\pi\|_{\total} \|\Delta_Q - \gamma  g^\pi_*(\cdot, \cdot; \Delta_Q)\|_{\total}\\
			& \leq \bigOp\left( \frac{\sqrt{K}\log (\Total)}{\sqrt {\Total}} + K^{-r_1}+  \sqrt{K}\zeta_{\total} \right)  \bigO\left( K^{-r_2} \right) = \smallOp((\Total)^{-1/2}),
			\end{align*}
			due to conditions for \(K\) in Theorem \ref{thm:asymp}(ii). 
			
			\item \textbf{For (\Rnum{6})}, by
			using the convergence of \(\hat \omega^\pi\) to \(\omega^\pi\) and the independence of \(\epsilon_{i,t}\) for every $1 \leq i \leq n$ and $0 \leq t \leq T-1$, we are able to prove that 
			\begin{align}
			&\EE \left\{  \frac{1}{nT}\sum_{i=1}^n \sum_{t=0}^{T-1} (\hat \omega^\pi_{i,t}- \omega^\pi_{i,t})\epsilon_{i,t}  \right\}^2 \nonumber\\
			 = &\EE\left[ \EE \left\{  \frac{1}{nT}\sum_{i=1}^n \sum_{t=0}^{T-1} (\hat \omega^\pi_{i,t}- \omega^\pi_{i,t})\epsilon_{i,t} \right\}^2\mid \{S_{i,t}, A_{i,t}, i = 1,\dots,n, t= 0, \dots, T-1\} \right]\nonumber\\
			\leq &\EE\left[ \frac{\sup_{i,t}\EE \epsilon_{i,t}^2 }{(\Total)^2} \sum_{i=1}^n \sum_{t=0}^{T-1}(\hat \omega^\pi_{i,t}- \omega^\pi_{i,t})^2 \right] \lesssim \frac{1}{\Total} \EE \left[ \|\hat \omega^\pi - \omega^\pi\|_{\total}^2\right] = \smallO\left( \frac{1}{\Total} \right), \nonumber
			\end{align}
			The first inequality is derived based on  Assumption \ref{assum:additional}(b) and the second inequality comes from Theorem \ref{thm:weights}. Then we can show that 
			\begin{align}
			\frac{1}{nT}\sum_{i=1}^n \sum_{t=0}^{T-1} (\hat \omega^\pi_{i,t}- \omega^\pi_{i,t})\epsilon_{i,t}   = \smallOp((\Total)^{-1/2}).\nonumber
			\end{align}
			Therefore,
			\begin{align}
			(\mathrm{\Rnum{6}}) = \frac{1}{nT}\sum_{i=1}^n \sum_{t=0}^{T-1} \omega^\pi_{i,t} \epsilon_{i,t} + \frac{1}{nT}\sum_{i=1}^n \sum_{t=0}^{T-1} (\hat \omega^\pi_{i,t}- \omega^\pi_{i,t})\epsilon_{i,t} 
			= \frac{1}{nT}\sum_{i=1}^n \sum_{t=0}^{T-1} \omega^\pi_{i,t} \epsilon_{i,t}  + \smallOp((\Total)^{-1/2}) .\nonumber
			\end{align}

			\item \textbf{For (\Rnum{2})}, recall that the true weight function \(\omega^\pi(s,a) 
			\in \mathcal{G}\) by the assumption in Theorem \ref{thm:asymp}(ii). This is the key to bound (\Rnum{2}). 
			Take a function \(B =  \sum_{k=1}^K \beta_k B_k\) in optimization problem \eqref{eqn:approx_g}.  In addition, it can be seen that \(\hat g^\pi_\beta(\cdot,\cdot) := \sum_{k=1}^K \beta_k \hat g^\pi (\cdot, \cdot; B_k)\) by the structure of the optimization problem. Let \({Y}_{i,t} =  \sum_{a'}\pi(a'\mid S_{i,t+1}) B(S_{i,t+1}, a')\).

			Due to the optimization condition for \(\hat g^\pi_\beta\), we have
			\[\frac{d \left[  \frac{1}{nT}\sum_{i=1}^n \sum_{t=0}^{T-1} \left\{ {Y}_{i,t} - \hat g^\pi_\beta(S_{i,t}, A_{i,t}) - u \omega^\pi(S_{i,t}, A_{i,t}) \right\}^2 + \mu J_{\mathcal{G}}(\hat g^\pi_\beta + u \omega^\pi) \right] }{du } \Bigg|_{u=0}= 0 \]
			
			Recall that \(J_{\mathcal{G}}(g) = \|g\|_{\Hscr_2}^2\). Here, we abuse the notation slightly and denote by \(\langle\cdot, \cdot \rangle_{\Hscr_2}\), i.e., the inner product with respect to the RKHS specified in Assumption \ref{assum:RKHS}(b). Then we have 
			\begin{align*}
			0 & =  \frac{-2}{nT}\sum_{i=1}^n \sum_{t=0}^{T-1}  \omega^\pi(S_{i,t}, A_{i,t}) \{{Y}_{i,t} - g^\pi_*(S_{i,t}, A_{i,t}; B)\}  \\
			&  +  \frac{2}{nT}\sum_{i=1}^n \sum_{t=0}^{T-1} \omega^\pi(S_i, A_i) \{\hat g^\pi_\beta(S_{i,t}, A_{i,t}) - g^\pi_*(S_{i,t}, A_{i,t}; B)\} 
			+ 2 \mu \langle \hat g^\pi_\beta, \omega^\pi  \rangle_{\Hscr_2}\\
			& = (\mathrm{i}) + (\mathrm{ii}) + (\mathrm{iii}),
			\end{align*}

			Now, we can see that the term \(( \mathrm{ii})\) multiplied by $\gamma$ is exactly term (\Rnum{2}) that  we aim to bound. Therefore, it is sufficient to bound \((\mathrm{i})\) and \((\mathrm{iii})\). Since we require that \(\mu = \smallO((\Total)^{-1/2})\), \(|(\mathrm{iii}) |= \smallOp((\Total)^{-1/2})\). As for term \((i)\), recall that the  error \(e_{i,t} := \sum_{a'} \pi(a'\mid S_{i,t+1}) Q^\pi(S_{i, t+1}, a') - g^\pi_*(S_{i,t}, A_{i,t}; Q^\pi)\). 
			Then we have
			\begin{align*}
			\mathrm{(i)}  &=  \frac{-2}{nT}\sum_{i=1}^n \sum_{t=0}^{T-1} \omega^\pi(S_{i,t}, A_{i,t})e_{i,t} \\
			&  +\frac{-2}{nT}\sum_{i=1}^n \sum_{t=0}^{T-1} \omega^\pi(S_{i,t}, A_{i,t}) \left\{\sum_{a'}\pi(a'\mid S_{i,t+1}) \Delta_Q(S_{i,t+1},a) - g_*^\pi(S_{i,t}, A_{i,t}; \Delta_Q)\right\} 
			\end{align*}
			Next, we use Freedman's inequality \citep[][]{freedman1975tail} to bound the second term in (i). 
			For any integer \(1\leq g \leq nT\), let \(i(g)\) and \(t(g)\) be the quotient and the remainder of \(g + T -1\) divided by \(T\) satisfy \(g = \{i(g) - 1\} T + t(g) + 1 \) and \(0 \leq t(g)< T\). Let \(\mathcal{F}^{(0)} = \{S_{1,0}, A_{1,0}\}\). Then we iteratively define \(\{\mathcal{F}^{g}\}_{1\leq g \leq nT}\) as follows:
			\begin{align*}
			\mathcal{F}^{(g)} = \mathcal{F}^{(g-1)} \cup \{S_{i(g),t(g)+1}, A_{i(g), t(g)+1}\}, \mbox{ if } t(g) < T-1,\\
			\mathcal{F}^{(g)} = \mathcal{F}^{(g-1)} \cup \{S_{i(g),T}, S_{i(g)+1,0},  A_{i(g)+1, 0}\}, \mbox{ otherwise }.
			\end{align*}
			Take \(e^{(g)}:= \omega^\pi(S_{i(g), t(g)}, A_{i(g), t(g)})e^\Delta_{i(g), t(g)}\). From the definition of \(e_{i,t}^\Delta\), we can show that
			\begin{align*}
			\EE \{e^{(g)}\mid \mathcal{F}^{(g-1)}\} = \EE \{e^{(g)}\mid S_{i(g), t(g)}, A_{i(g), t(g)}\} = 0.
			\end{align*}
			Then \(\{\sum_{g=1}^G e^{(g)}\}\), \(G = 1,\dots, nT\) forms a martingale with respect to the filtration \(\{\sigma(\mathcal{F}^{(g)})\}_{g\ge 0}\), where \(\sigma(\mathcal{F}^{(g)})\) stands for the \(\sigma\)-algebra generated by \(\mathcal{F}^{(g)}\).

			Note that errors \( e_{i,t}^\Delta : = \sum_{a'}\pi(a'\mid S_{i,t+1}) \Delta_Q(S_{i,t+1},a) - g_*^\pi(S_{i,t}, A_{i,t}; \Delta_Q)\), \(i = 1,\dots, n\), \(t= 1,\dots, T\) 
			are bounded by \(\|\Delta_Q\|_\infty \leq \oldu{Qconst1} K^{-r_2}\). Then we can verify that \(\EE( e^{(g)})^2 \leq \oldu{weightsC1}^2 \oldu{Qconst1}^2 K^{-2r_2}\) by Assumptions \ref{assum:weights}(a) and \ref{assum:additional}(a). Now, we are able to apply Theorem 1.6 in \cite{freedman1975tail}. For all \(x > 0\), we have 
			\begin{align*}
			\pr\left(  \sum_{g=1}^{\Total} e^{(g)} > x \right) \leq \exp\left( -\frac{x^2/2}{\Total \oldu{weightsC1}^2 \oldu{Qconst1}^2 K^{-2r_2} + \oldu{weightsC1} \oldu{Qconst1} K^{-r_2}x/3} \right).
			\end{align*}
			
			As such, we can derive
			\begin{align*}
			\frac{-2}{nT}\sum_{i=1}^n \sum_{t=0}^{T-1}  \omega^\pi(S_{i,t}, A_{i,t}) \left\{\sum_{a'}\pi(a'\mid S_{i,t+1}) \Delta_Q(S_{i,t+1},a) - g_*^\pi(S_{i,t}, A_{i,t}; \Delta_Q)\right\} \\
			= - \frac{2}{\Total}\sum_{g=1}^{\Total} e^{(g)} = \bigOp(
			(\Total)^{-1/2} K^{-r_2}\log (\Total) ) = \smallOp((\Total)^{-1/2}).
			\end{align*}
			Then we can show that 
			\begin{align}
			(\mathrm{\Rnum{2}}) = \gamma (\mathrm{ii}) =  \gamma \left\{ \frac{1}{nT}\sum_{i=1}^n \sum_{t=0}^{T-1}    \omega^\pi_{i,t} e_{i,t} + \smallOp((\Total)^{-1/2}) \right\}
			\end{align}
			
		\end{itemize}

		Combining all the bounds from \((\Rnum{1})\) to \((\Rnum{6})\), we have
		\begin{align}
		\hat \calV(\pi) - \calV(\pi)& =  
		(\Rnum{1}) + (\Rnum{2}) + (\Rnum{3})  + (\Rnum{4})  + (\Rnum{5}) + (\Rnum{6})\nonumber\\
		& = \frac{1}{nT}\sum_{i=1}^n \sum_{t=0}^{T-1}  \omega^\pi_{i,t}\left\{ \epsilon_{i,t} + \gamma e_{i,t} \right\} + \smallOp((\Total)^{-1/2}).\nonumber
		\end{align}

		Taking \[\sigma^2 = \frac{1}{T} \sum_{t=0}^{T-1}\EE \left\{\omega^\pi(S_{t},A_{t})\left(R_t + \gamma\sum_{a' \in \calA}\pi(a' | S_{t+1}) Q^\pi(S_{t+1}, a') - Q^\pi(S_t, A_t)\right)\right\}^2,\]
		we see that 
		\[\sum_{i=1}^n \sum_{t=0}^{T-1}  \frac{\omega^\pi_{i,t}\left\{ \epsilon_{i,t} + \gamma e_{i,t} \right\}}{\sqrt{nT}\sigma}\] 
		forms a mean zero martingale with respect to the filtration \(\{\sigma(\mathcal{F}^{(g)})\}_{g\ge 0}\).  
		
		Then we can verify that 
		\begin{align}
		\max_{i,t} \frac{|\omega^\pi_{i,t}(\epsilon_{i,t} + \gamma e_{i,t})|}{\sigma\sqrt{nT}}
		\leq \oldu{weightsC1}\left\{\frac{R_{\max}}{\sigma \sqrt{nT}} + \gamma \frac{Q_{\max} + G_{\max} + \oldu{Qconst1}K^{-r_2}}{\sigma \sqrt{nT}}  \right\} \xrightarrow[]{P} 0\nonumber
		\end{align}
		due to  Assumption \ref{ass: reward}, \ref{assum:weights}(a) and \ref{assum:additional}(a) and   
		\begin{align}
		\sum_{i=1}^n\sum_{t=0}^{T-1}\frac{|\omega^\pi_{i,t}(\epsilon_{i,t} + \gamma e_{i,t})|^2}{\sigma^2nT} \xrightarrow[]{P} 1\nonumber
		\end{align}
		due to  the definition of \(\sigma\).
		
		Then by martingale central limit theorem \citep{mcleish1974dependent}, we can show the asymptotic normality,
		i.e., as long as either \(n\rightarrow \infty\) or \(T \rightarrow \infty\), we have
		\begin{align}
		\frac{\sqrt{nT}}{\sigma} \left\{ \hat \calV(\pi) - \calV(\pi)\right\}  \xrightarrow[]{d} N(0,1). 
		\end{align}
		
	\end{proof}

	\subsection{Technical Lemmas}
	In this section, we provide two technical lemmas (Lemma \ref{lem:g_truncation} and \ref{lem:weights_concen}) that involve the truncation arguments for Markov chains. They are developed to obtain sharp convergence rates that depend on both sample size \(n\) and the length of trajectories \(T\). In particular, Lemma \ref{lem:g_truncation} uses  empirical process arguments and is used for Theorem \ref{lem:approx_g}. Lemma \ref{lem:weights_concen} generalizes the standard matrix Bernstein’s inequality for independent Markov chains and is used for Theorem \ref{thm:weights}.

	\begin{lemma}
		\label{lem:g_truncation}
		Suppose Assumptions \ref{assum:approx_g} hold. For 
		\begin{align}
		I_2(B) &= 2(\EE - \PN)f^\pi \left\{\hat g^\pi(B), \g (B), B\right\}- \EE f^\pi \left\{\hat g^\pi(B), \g (B), B\right\}  -2 \mu\mathbf{J}^2(\hat g^\pi, g^\pi_*, B). \nonumber
		\end{align}
		where \(f^\pi\) and \(\mathbf{J}^2\) are defined in\eqref{eqn:fpi} and      \eqref{eqn:penalty} respectively. Then with probability at least \(1- 2\delta - 1/(\Total)\), for any \(B \in \mathcal{Q}\), we have 
		\begin{align*}
		I_2(B) \lesssim \mu + \frac{1}{\Total \mu^{\frac{\alpha}{1- \tau(2+\alpha)}} }  + \frac{[\log(\max\{1/\delta, \Total\})]^{1/\tau}}{\Total},
		\end{align*} 
		
		where \(0 < \tau \leq \frac{1}{3}\), and the leading constant depends on \(Q_{max}, G_{\max}, \alpha, \kappa, \oldu{C_s}\).
	\end{lemma}
	
	\begin{proof}[Proof of Lemma \ref{lem:g_truncation}]
		
		To deal with the scenario that the initial distribution \(\mathbb{G}\) is possibly not the stationary distribution \(\mathbb{G}^*\), we decompose \(I_2(B)\) into three components. 
		
		Take \(T' = \min\{\ceil{K_3\log(nT)},T\}\), where \(K_3\) is a constant to be specified later
		,  \(Z^T = [(S_{t}, A_{t})]_{t=0}^{T-1}\), \(f_T^\pi(g_1, g_2,B)(Z^T) = \frac{1}{T-1}\sum_{t=0}^T f^\pi(g_1, g_2,B)(S_t, A_t, S_{t+1})\). Take \(\mathbb{P}_n f_T^\pi(g_1, g_2,B)(Z^T)  =  n^{-1}\sum_{i=1}^n f_T^\pi(g_1, g_2,B)(Z^T_i )\) as the empirical evaluation for \(f_T^\pi\).  And we denote by \(\EE^*\), \({\pr}^*\) the expectation and probability under stationary distribution respectively. 
		\begin{align}
		I_2(B) &= 2(\EE - \PN)f^\pi \left\{\hat g^\pi(B), \g (B), B\right\}- \EE f^\pi \left\{\hat g^\pi(B), \g (B), B\right\}  -2 \mu\mathbf{J}^2(\hat g^\pi, g^\pi_*, B). \nonumber\\
		&= \frac{T'}{T} \mathrm{\Rnum{2}}(B) + \frac{T-T'}{T}(\mathrm{\Rnum{3}}(B) + \mathrm{\Rnum{4}}(B)), \nonumber
		\end{align}
		where 
		\begin{align*}
		\mathrm{\Rnum{2}}(B) = 2(\EE- \mathbb{P}_{n})f_{T'}^\pi \left\{\hat g^\pi(B), \g (B), B\right\}- \EE f^\pi_{T'} \left\{\hat g^\pi(B), \g (B), B\right\} 
		-2 \mu\mathbf{J}^2(\hat g^\pi, g^\pi_*, B),
		\end{align*}
		\begin{align}
		\mathrm{\Rnum{3}}(B) =& \left(\frac{1}{n(T-T')}\sum_{i=1}^n\sum_{t=T'}^{T-1}  -2f^\pi \left\{\hat g^\pi(B), \g (B), B\right\}(S_{i,t}, A_{i,t}, S_{i,t+1}) \right) +  \EE^* f^\pi \left\{\hat g^\pi(B), \g (B), B\right\}  \nonumber\\
		&-2 \mu\mathbf{J}^2(\hat g^\pi, g^\pi_*, B), \nonumber
		\end{align}
		and 
		\begin{align}
		\mathrm{\Rnum{4}}(B) = \EE \left[ \frac{1}{T-T'}\sum_{t=T'}^{T-1}  f^\pi \left\{\hat g^\pi(B), \g (B), B\right\}(S_{t}, A_{t}, S_{t+1}) \right] - \EE^* f^\pi \left\{\hat g^\pi(B), \g (B), B\right\} .\nonumber
		\end{align}

		For some fixed \(t >0\), 
		\begin{align*}
		&\pr\left( \exists B \in  \mathcal{Q}, I_2(B) > 3t \right)\\
		\leq & \pr\left( \exists B \in  \mathcal{Q},  \mathrm{\Rnum{2}}(B)>\frac{T}{T'}t \right) + \pr\left( \exists B \in  \mathcal{Q},  \mathrm{\Rnum{3}}(B)>\frac{T}{T-T'}t \right) +  \pr\left( \exists B \in  \mathcal{Q},  \mathrm{\Rnum{4}}(B)>\frac{T}{T-T'}t \right)
		\end{align*}
		Then we bound these three probabilities one by one. Note that when \(T' = T\), we do not need to take into account components \Rnum{3}(B) and \Rnum{4}(B). So in the following, we analyze \Rnum{3}(B) and \Rnum{4}(B) under the condition that \(T \gg \log(nT)\).
		
		\begin{itemize}
			\item \textbf{For \Rnum{2}(B):}
		\end{itemize}
		
		Take $\F^{T'}_l = \{ f_{T'}^\pi\left\{g, \g (B), B\right\}: J_{\mathcal{G}}^2(g) \leq \frac{2^{l}}{\mu}\frac{T}{T'}t, J_{\mathcal{G}}^2(\g(B)) \leq \frac{2^{l}}{\mu}\frac{T}{T'}t,  J_{\mathcal{Q}}^2(B) \leq \frac{2^{l}}{\mu}\frac{T}{T'}t,  B \in \Q  \}.$ 
		
		\begin{align*}
		&\pr\left( \exists B \in  \mathcal{Q}, \mathrm{\Rnum{2}}(B) > \frac{T}{T'}t \right)\\
		\leq& \sum_{l=0}^\infty \pr\left( \sup_{h \in \mathcal{F}^{T'}_l} \frac{(\EE - \mathbb{P}_n)h}{\EE(h) + 2^l\frac{T}{T'}t} > \frac{1}{2}\right),
		\end{align*}
		
		In the following, we verify the conditions (A1-A4) in Theorem 19.3 in \cite{gyorfi2006distribution} with \(\mathcal{F} = \mathcal{F}^T_l\), \(\epsilon = 1/2\) and \(\eta = 2^l t\). 
		
		It's easy to verify for (A1). For any \(h \in \mathcal{F}\), 
		\begin{align}
		\label{eqn:A1}
		\|f^\pi\{g, g^\pi_*(B), B\}\|_\infty \leq 6G_{\max}(3Q_{\max} + 3G_{\max})  \triangleq K_1
		\end{align}
		and therefore 
		\begin{align*}
		\label{}
		\|f_{T'}^\pi\{g, g^\pi_*(B), B\}\|_\infty \leq \frac{1}{T'} \sum_{t=0}^{T'-1}\|f^\pi\{g, g^\pi_*(B), B\}\|_\infty  \leq K_1
		\end{align*}
		
		For (A2),
		recall  $f^\pi  = f_1^\pi - f_2^\pi$ and thus
		\begin{align*}
		\EE [f^\pi\left\{g, \g (B), B\right\}^2] \leq 2\EE [f_1^\pi\left\{g, \g (B), B\right\}(S, A, S')^2] + 2\EE [f_2^\pi\left\{g, \g (B), B\right\}(S, A, S')^2].
		\end{align*}
		For the first term of RHS above:
		\begin{align*}
		& \EE [f_1^\pi\left\{g, \g (B), B\right\}(S, A, S')^2]  \\
		& = \EE\left[\left\{ \left\{\delta_t^\pi(B) - g(S, A)\right\}^2 -  \left\{\delta_t^\pi(B) - \g(S, A; Q)\right\}^2\right\}^2\right] \\
		& = \EE \left[ \left\{2\delta^\pi(B) - g(S, A)  - \g(S, A; B)\right\}^2 \left\{\g(S, A; B) - g(S, A)\right\}^2\right] \\
		& \leq \left\{2(Q_{\max}) + 2G_{\max}\right\}^2   \EE [\left(\g(S, A; B) - g(S, A)\right)^2] \\
		& = 4\left( Q_{\max} + G_{\max}\right)^2 \EE \left[f^\pi\left\{g, \g (B), B\right\}(S, A, S')\right],
		\end{align*}
		and the second term:
		\begin{align*}
		&  \EE [f_2^\pi\left\{g, \g (B), B\right\}(S, A, S')^2] \\
		& = \EE \left[\Bigdkh{ \dkh{\delta_t^\pi(B) - \g(S, A; B)}\dkh{g(S, A) - \g(S, A; B)}}^2\right] \\
		& \leq  \EE \left[ \dkh{\delta^\pi(B) - \g(S, A; B)}^2\dkh{g(S, A) - \g(S, A; B)}^2\right] \\
		& \leq \left(Q_{\max} + G_{\max}\right)^2 \EE [\left(\g(S, A; B) - g(S, A)\right)^2] \\
		&  = \left( Q_{\max} + G_{\max}\right)^2 \EE  [f^\pi\left\{g, \g (B), B\right\}(S, A, S')],
		\end{align*}
		where we use the fact that $\EE [f_2^\pi\left\{\gn(B), \g(B), B\right\}(S, A, S')] = 0$. Putting together, we can show that
		\begin{align}
		\label{eqn:A2}
		\EE [f^\pi\left\{g, \g (B), B\right\}(S, A, S')^2]
		& \leq K_2 \EE [f^\pi\left\{g, \g (Q), Q\right\}(S, A, S')],
		\end{align}
		where $K_2 = \left(Q_{\max} + G_{\max}\right)^2$. 
		Therefore,
		\begin{align*}
		\EE f_T^2 &= \EE \left( \frac{1}{T}\sum_{t=0}^{T-1} f^\pi\left\{g, \g (Q), Q\right\}(S_{t}, A_{t}, S_{t+1}) \right)^2 \\
		&\leq \frac{1}{T}\sum_{t=0}^{T-1} \EE [f^\pi\left\{g, \g (B), B\right\}(S_{t}, A_{t}, S_{t+1}) ^2]\\
		&\leq \frac{1}{T}\sum_{t=0}^{T-1} K_2  \EE [f^\pi\left\{g, \g (B), B\right\}(S_{t}, A_{t}, S_{t+1})] = K_2 \EE f^{\pi}_T
		\end{align*}
		
		To verify that Condition (A3) holds for every \(l\). It suffices to ensure the inequality holds when \(l=0\), i.e.,
		\begin{align}
		\label{eqn:A3}
		\sqrt{n}(1/2)^{3/2}\sqrt{\frac{T}{T'}t} \ge 288\max(K_1, \sqrt{2K_2})\\
		\frac{T}{T'}t \ge 8\{288\max(K_1, \sqrt{2K_2})\}^2 n^{-1}\nonumber
		\end{align}
		
		Next we verify (A4). We first consider the function class $\F_l = \{ f^\pi\left\{g, \g (B), B\right\}: J_{\mathcal{G}}^2(g) \leq \frac{2^{l}}{\mu} \frac{T}{T'}t, J_{\mathcal{G}}^2(\g(B)) \leq \frac{2^{l}}{\mu} \frac{T}{T'}t,  J_{\mathcal{Q}}^2(B) \leq \frac{2^{l}}{\mu} \frac{T}{T'}t,  B \in \Q  \}.$ 
		It is not hard to verify that with \(M = \sqrt{2^l \frac{T}{T'}t/\mu}\), 
		
		\begin{align*}
		&\log \left(\mathcal{N}(\epsilon , \F^T_l, \norm{\cdot}_n) \right)\\
		\leqconst  &  \log\left(\mathcal{N}(\epsilon , \F_l, \norm{\cdot}_{\total}) \right)\\
		\leqconst &  \log\left(\cal \mathcal{N}(\epsilon, \Q_{M}, \norm{\cdot}_\infty) \mathcal{N} (\epsilon, \G_{M}, \norm{\cdot}_\infty) \right).\\
		\end{align*}
		As a result of the entropy condition in Assumption \ref{assum:approx_g}
		\begin{align}
		&   \log \left(\mathcal{N}(\epsilon , \F^T_l, \norm{\cdot}_n) \right)\nonumber\\
		\leqconst&  \log  \mathcal{N}(\epsilon, \F_l, \norm{\cdot}_{\total}) \nonumber\\
		\leqconst& \log \mathcal{N}(\epsilon, \Q_{M}, \norm{\cdot}_\infty) + 2 \log \mathcal{N}(\epsilon, \G_{M}, \norm{\cdot}_\infty)\nonumber \\
		\leqconst &  \left(\frac{2^l}{\mu}\frac{T}{T'}t\right)^{\alpha} \epsilon^{-2\alpha}. \label{eqn:A4}
		\end{align}
		Now the Condition (A4) is satisfied if 
		\begin{align*}
		x^{\frac{1+\alpha}{2}} \ge 4*96 \sqrt{2}\max(K_1, 2K_2) \sqrt{\newl\ltxlabel{case1c1}}\left( \frac{2^l}{\mu} \frac{T}{T'}t\right)^{\alpha/2} n^{-1/2},
		\end{align*}
		for all \(x \ge 2^lTt/(8T')\) and any \(l\), where \(\oldl{case1c1}>0\) is a constant depending on \(\alpha, G_{\max}, Q_{\max}\).
		And the above inequality holds as long as 
		$$ \frac{T}{T'}t \ge  \newl\ltxlabel{case1c2} (\mu^\alpha n)^{-1}$$
		for some constant \(\oldl{case1c2}>0\) depending on \(\alpha,  G_{\max}, Q_{\max}\).
		
		Now, we are able to show that when \(\frac{T}{T'}t > \mu\), \(\frac{T}{T'}t \ge \oldl{case1c2}(\mu^\alpha n)^{-1}\), \( \frac{T}{T'}t \ge 8\{288\max(K_1, \sqrt{2K_2})\}^2 n^{-1}\nonumber\),
		\begin{align*}
		&\Pr\left( \exists B \in  \mathcal{Q}, \mathrm{\Rnum{2}}(B) > \frac{T}{T'}t \right)\\
		\leq& \sum_{l=0}^\infty \Pr\left( \sup_{h \in \mathcal{F}^T_l} \frac{(\EE - \mathbb{P}_n)h}{\EE h + 2^l\frac{T}{T'}t} > \frac{1}{2}\right)\\
		\leq & \sum_{l=0}^\infty\exp \left\{-\frac{n2^l \frac{T}{T'}t(1/2)^3}{128*2304\max(K_1^2, K_2)}  \right\}\\
		\leq & \newl\ltxlabel{case1c3} \exp\left( -\newl\ltxlabel{case1c4}n\frac{T}{T'}t \right),
		\end{align*}
		where \(\oldl{case1c3}, \oldl{case1c4} >0\) are some constants depending on \(K_1, K_2\). 
		Then, fixing some \(\delta > 0\), we can verify that there exists a constant \(\oldl{case1c5}>0\), such that by taking \(t \ge \newl\ltxlabel{case1c5} \log(1/\delta) T'/(nT)\) ,\(\Pr\left( \exists B \in  \mathcal{Q}, \mathrm{\Rnum{2}}(B) > \frac{T}{T'}t \right)\leq \oldl{case1c3}\exp( -\oldl{case1c4}n Tt / T') \leq \delta.\)

		Combining all the conditions on \(t\), there exists a constant \(\newl\ltxlabel{case1c6}\), such that  
		\begin{align*}
		\Pr\left( \exists B \in  \mathcal{Q}, \frac{T'}{T} \mathrm{\Rnum{2}}(B) > \oldl{case1c6}\left\{\frac{T'}{T}\mu + \frac{T'}{T}(\mu^\alpha n)^{-1} + (\log(1/\delta) + 1) \frac{T'}{nT}\right\}\right) \leq \delta.
		\end{align*}
		
		\begin{itemize}
			\item \textbf{For \Rnum{3}(B):}
		\end{itemize}

		By the  geometrically ergodic property (Assumption \ref{assum:approx_g}(a)), 
		\begin{align*}
		&\EE\left[  \Pr\left(\exists B,  \mathrm{\Rnum{3}}(B) \ge \frac{T}{T-T'}t \mid \{(S_{i,0}, A_{i,0})\}\right) -{\pr}^*\left( \exists B,  \mathrm{\Rnum{3}}(B) \ge  \frac{T}{T-T'}t  \right)\right]\\
		\leq &  \EE \left\{\sum_{i=1}^n \left\| \mathbb{G}^b_{T'}(\cdot \mid (S_{i,0}, A_{i,0})) - \mathbb{G}^*(\cdot)\right\|_{\mathrm{TV}} \right\}\\
		\leq & \EE \left\{\sum_{i=1}^n \phi(S_{i,0}, A_{i,0})\kappa^{T'}\right\} \\
		\leq & n \kappa^{T'} \int_{(s,a)} \phi(s,a) d\mathbb{G}(s,a) \leq  n \kappa^{T'} \oldu{C_s}.
		\end{align*}
		Thus we have 
		\begin{align}
		\label{eqn:stationary_decompose}
		\Pr\left(\exists B,  \mathrm{\Rnum{3}}(B) \ge  \frac{T}{T-T'}t  \right) \leq {\Pr}^*\left( \exists B,  \mathrm{\Rnum{3}}(B) \ge  \frac{T}{T-T'}t  \right)  +  n \kappa^{T'} \oldu{C_s} 
		\end{align}
		
		Next, we will focus on the term \({\Pr}^*( \exists B,  \mathrm{\Rnum{3}}(B) \ge  \frac{T}{T-T'}t)\), which is easier to deal with, given the stationary condition.

		Take \(  \tilde{N} = n(T-T')\). Due to that Markov chains are stationary, we can write 
		\begin{align*}
		\mathrm{\Rnum{3}}(B) 
		= &2(\EE^* - \mathbb{P}_{\tilde{N}})f^\pi \left\{\hat g^\pi(B), \g (B), B\right\}({S}, {A},{S}' ) - \EE^* f^\pi \left\{\hat g^\pi(B), \g (B), B\right\}({S}, {A},{S}' )\\
		&-2 \mu\mathbf{J}^2(\hat g^\pi, g^\pi_*, B).
		\end{align*}
		Similarly as how we handle \Rnum{2}(B), we will bound 
		\begin{align*}
		&{\Pr}^*\left( \exists B \in  \mathcal{Q},  \mathrm{\Rnum{3}} >  \frac{T}{T-T'}t  \right)
		\leq \sum_{l=0}^\infty {\Pr}^*\left( \sup_{h \in \mathcal{F}_l} \frac{(\EE^* - \mathbb{P}_{\tilde{N}})h}{\EE^* h + 2^l \frac{T}{T-T'}t } > \frac{1}{2}\right),
		\end{align*}
		Here $\F_l = \{ f^\pi\left\{g, \g (B), B\right\}: J_{\mathcal{G}}^2(g) \leq \frac{2^{l} }{\mu}\frac{T}{T-T'}t, J_{\mathcal{G}}^2(\g(B)) \leq \frac{2^{l}}{\mu}\frac{T}{T-T'}t,  J_{\mathcal{Q}}^2(B) \leq \frac{2^{l}}{\mu}\frac{T}{T-T'}t,  B \in \Q  \}.$ 
		
		We use the independent block techniques \citep{yu1994rates} and the peeling device with the exponential inequality for the relative deviation of the empirical process developed in \citep{farahmand2012regularized}. 
		
		Next, we bound each term of the above probabilities by using the independent block technique. We define a partition by dividing the index $\{1, \cdots, \tilde{N}\}$ into $2v_{\tilde{N}}$ blocks, where each block has an equal length $x_{\tilde{N}}$. The residual block is denoted by $R_{\tilde{N}}$, i.e., $\{(j-1)x_{\tilde{N}}+1, \cdots, (j-1)x_{\tilde{N}} + x_{\tilde{N}} \}_{j= 1}^{2v_{\tilde{N}}}$ and $R_{\tilde{N}} = \left\{2v_{\tilde{N}}x_{\tilde{N}} +1, \cdots,\tilde{N} \right\}$. Then it can be seen that  ${\tilde{N}}- 2x_{\tilde{N}} < 2 v_{\tilde{N}} x_{\tilde{N}} \leq {\tilde{N}}$ and the cardinality $|R_{\tilde{N}}|< 2x_{\tilde{N}}$.

		For each $l \geq 0$, we will use a different independent block sequence denoted by $(x_{\tilde{N}, l}, v_{\tilde{N}, l})$ with the residual $R_{l}$ and then optimize the probability bound by properly choosing $(x_{\tilde{N}, l}, v_{\tilde{N}, l})$ and $R_{l}$. More specifically, we choose
		$$
		x_{\tilde{N}, l} = \lfloor x'_{\tilde{N}, l} \rfloor \epc \mbox{and } \epc  v_{N, l} = \lfloor \frac{\tilde{N}}{2x_{\tilde{N},l}} \rfloor,
		$$
		where $x'_{\tilde{N}, l} = (\tilde{N}\frac{T}{T-T'}t)^\tau (2^l)^p$ and $v'_{\tilde{N}, l} = \frac{\tilde{N}}{2x'_{\tilde{N}, l}}$ with some positive constants $\tau$ and $p$  determined later. We require $\tau \leq p \leq \frac{1}{2 + \alpha} \leq \frac{1}{2}$. We also need $\frac{T}{T-T'}t \geq \frac{1}{\tilde{N}}$ so that $x'_{\tilde{N}, l}\geq 1$. Suppose $\tilde{N}$ is sufficiently large such that
		\begin{align}\label{sample size constraint for lemma of value function}
		\tilde{N} \geq \newl\ltxlabel{case2c0} \triangleq 4 \times 8^2 \times K_1. 
		\end{align}
		
		In the following, we consider two cases. The first case considers any $l$ such that $x'_{\tilde{N},l} \geq \frac{\tilde{N}}{8}$. In this case, since $\tau \leq p$, we can show that $x'_{\tilde{N}, l} \leq (\tilde{N} \frac{T}{T-T'}t 2^l)^p$. Combining with the sample size requirement,  we can obtain that  $(\tilde{N} \frac{T}{T-T'}t 2^l) \geq (\frac{\tilde{N}}{8})^\frac{1}{p} \geq 4\tilde{N}K_1$. Then we can show that in this case,
		$$
		\frac{(\EE^*-\mathbb{P}_{\tilde{N}}) \left\{h(Z)\right\}}{\EE^* \left\{h(Z)\right\} + 2^l  \frac{T}{T-T'}t } \leq  \frac{2K_1}{2^l \frac{T}{T-T'}t } \leq  \frac{1}{2}.
		$$
		Therefore, when $ \frac{T}{T-T'}t  \geq \frac{1}{\tilde{N}}$ and $x'_{\tilde{N}, l} \geq \frac{\tilde{N}}{8}$, 
		$$
		{\Pr}^*\left( \sup_{h \in \F_l}  \frac{(\EE^*-\mathbb{P}_{\tilde{N}}) \left\{h(Z)\right\}}{\EE^*  \left\{h(Z)\right\} + 2^l  \frac{T}{T-T'}t } > \frac{1}{2} \right) = 0.
		$$
		
		The second case that we consider is when $x'_{\tilde{N}, l} < \frac{\tilde{N}}{8}$. 
		Under the geometric ergodicity assumption, it follows from Theorem 3.7 \cite{bradley2005basic} that the stationary Markov chain is exponentially \(\boldsymbol{\beta}\)-mixing. The $\boldsymbol{\beta}$-mixing coefficient at time lag $k$ satisfies that $\beta_k \leq \beta_0 \exp(-\beta_1 k)$ for $\beta_0 \geq 0$ and $\beta_1 > 0$. 
		We apply the relative deviation concentration inequality for the exponential $\boldsymbol{\beta}$-mixing stationary process given in Theorem 4 of \cite{farahmand2012regularized}, which combined results in \cite{yu1994rates} and Theorem 19.3 in \cite{gyorfi2006distribution}. To use their results, it then suffices to verify Conditions (C1)-(C5) in Theorem 4 of \cite{farahmand2012regularized} with $\F = \F_l$, $\epsilon = 1/2$ and $\eta = 2^{l} t$ to get an exponential inequality for each term in the summation. 
		First of all, (C1) and (C2) have been verified in \eqref{eqn:A1} and \eqref{eqn:A2}.

		To verify Condition  (C3), without loss of generality, we assume $K_1 \geq 1$. Otherwise, let $K_1 = \max(1, K_1)$. Then we know that $2K_1x_{\tilde{N}, l} \geq \sqrt{2K_1x_{\tilde{N}, l}}$ since $x_{\tilde{N}, l} \geq 1$. We need to verify $\sqrt{\tilde{N}} \epsilon\sqrt{1-\epsilon} \sqrt{\eta} \geq 1152K_1x_{\tilde{N}, l}$, or it suffices to have $\sqrt{\tilde{N}} \epsilon\sqrt{1-\epsilon} \sqrt{\eta} \geq 1152K_1x'_{\tilde{N}, l}$ since $x'_{\tilde{N}, l} \geq x_{\tilde{N}, l}$ by definition. Recall that $\epsilon = 1/2$ and $\eta = 2^{l} \frac{T}{T-T'}t$. To show this, it is enough to verify that
		$$
		\sqrt{\tilde{N}}\frac{\sqrt{2}}{4}  \sqrt{2^l \frac{T}{T-T'}t} \geq 1152K_1(\tilde{N} \frac{T}{T-T'}t2^l)^p,
		$$
		since $(\tilde{N} \frac{T}{T-T'}t2^l)^p \geq x'_{\tilde{N}, l}$. Recall that $p \leq \frac{1}{2+\alpha}$, then it is sufficient to 
		let $ \frac{T}{T-T'}t \geq \frac{2304\sqrt{2}K_1}{\tilde{N}}\triangleq \frac{c_1'}{\tilde{N}}$ so that the above inequality holds for every $l \geq 0$.
		
		Next we verify (C4) that $\frac{|R_l|}{\tilde{N}} \leq \frac{\epsilon\eta}{6K_1}$.  Recall that $|R_l| < 2 x_{\tilde{N}, l} \leq 2x'_{\tilde{N}, l} = 2(\tilde{N} \frac{T}{T-T'}t)^\tau(2^l)^p$. So if $ \frac{T}{T-T'}t \geq \frac{\newl\ltxlabel{case2c1}}{\tilde{N}}$ for some positive constant $\oldl{case2c1}$ depending on $K_1$, we can have
		$$
		\frac{\epsilon\eta}{6K_1} = \frac{2^l}{12K_1} \frac{T}{T-T'}t \geq \frac{2(\tilde{N} \frac{T}{T-T'}t)^\tau (2^l)^p}{\tilde{N}} = \frac{2x'_{\tilde{N},l}}{\tilde{N}}> \frac{|R_l|}{\tilde{N}}.
		$$
		In addition, $|R_l| \leq 2 x'_{\tilde{N}, l} < \frac{\tilde{N}}{2}$.
		
		Lastly we verify condition (C5). Using the similar arguments in verifying \eqref{eqn:A4}, we can show that  
		\begin{align*}
		& \log  \mathcal{N}(\epsilon, \F_l, \norm{\cdot}_{\total}) 
		\leqconst   \left(\frac{2^l  \frac{T}{T-T'}t}{\mu}\right)^{\alpha} \epsilon^{-2\alpha}.
		\end{align*}
		Then Condition (C5) is satisfied if the following inequality holds for all $x \geq (2^l \frac{T}{T-T'}tx_{\tilde{N}, l})/8$, 
		\begin{align*}
		\frac{\sqrt{v_{\tilde{N}, l}} (1/2)^2 x}{96x_{\tilde{N}, l}\sqrt{2} \max(K_1, 2K_2)} 
		&\geq \int_{0}^{\sqrt{x}} \sqrt{\newl\ltxlabel{case2c2}} \left(\frac{2^l  \frac{T}{T-T'}t}{\mu}\right)^{\alpha/2} \left(\frac{u}{2x_{\tilde{N}, l}}\right)^{-\alpha}  du \\
		&= x_{\tilde{N}, l}^{\alpha} x^{\frac{1-\alpha}{2}} \sqrt{2^\alpha \oldl{case2c2} } \left(\frac{2^l  \frac{T}{T-T'}}{\mu}\right)^{\alpha/2},
		\end{align*}
		where \(\oldl{case2c2}>0\) is a constant depending on \(\alpha, G_{\max}, Q_{\max}\).
		
		It is enough to guarantee that
		$$
		\frac{\sqrt{v_{\tilde{N}, l}} (1/2)^2 x}{96x_{\tilde{N}, l}\sqrt{2} \max(K_1, 2K_2)} \geq x_{\tilde{N}, l}^{\alpha} x^{\frac{1-\alpha}{2}} \sqrt{2^\alpha \oldl{case2c2}} \left(\frac{2^l  \frac{T}{T-T'}t}{\mu}\right)^{\alpha/2}.
		$$
		After some algebra, we can check that the above inequality holds if for some constant $\newl\ltxlabel{case2c3}>0$,
		\begin{align*}
		\frac{T}{T-T'}t \geq \oldl{case2c3} \frac{(x_{\tilde{N}, l})^{1+\alpha}}{v'_{\tilde{N}, l}2^l \mu^\alpha},
		\end{align*}
		or equivalently, for some constant \(\oldl{case2c4}>0\),
		\begin{align*}
		\frac{T}{T-T'}t \geq \newl\ltxlabel{case2c4} \frac{1}{\tilde{N} \mu^{\frac{\alpha}{1 - \tau(2+ \alpha)}}\left(2^l\right)^{\frac{1 - p(2+\alpha)}{1 - \tau(2 + \alpha)}}},
		\end{align*}
		by the definition that $x_{\tilde{N}, l} \leq x'_{\tilde{N}, l}$ and $v'_{\tilde{N}, l} \leq v_{\tilde{N}, l}$. To summarize, if for any $l\geq 0$,
		$$
		\frac{T}{T-T'}t \geq \mu + \oldl{case2c4} \frac{1}{\tilde{N} \mu^{\frac{\alpha}{1 - \tau(2+ \alpha)}}\left(2^l\right)^{\frac{1 - p(2+\alpha)}{1 - \tau(2 + \alpha)}}},
		$$
		then the entropy inequality in Condition (C5) above holds. Since $0 < \tau \leq p \leq \frac{1}{1 + 2\alpha}$, the right hand side is a non-increasing function of $l$. Then
		as long as, 
		$$
		\frac{T}{T-T'}t \geq \mu + \oldl{case2c4} \frac{1}{\tilde{N} \mu^{\frac{\alpha}{1 - \tau(2+ \alpha)}}},
		$$
		Condition (C5) holds .
		
		To summarize, Conditions (C1)--(C5) in Theorem 4 in \cite{farahmand2012regularized} with $\F = \F_l$, $\epsilon = 1/2$ and $\eta = 2^{l} t$ hold for every $l \geq 0$ when $ \frac{T}{T-T'}t \geq \oldl{case2c1}\tilde{N}^{-1}$ for some constant $\oldl{case2c1}\geq1$ and $ \frac{T}{T-T'}t \geq \mu + \oldl{case2c4} \frac{1}{\tilde{N} \mu^{\frac{\alpha}{1 - \tau(2+ \alpha)}}}$. Thus when $\tilde{N}\geq \oldl{case2c0}$, there exists a constant \(\newl\ltxlabel{case2c5}>0\), such that
		\begin{align*}
		& {\Pr}^*\left\{\exists B \in  \Q,  \mathrm{\Rnum{3}}(B) >  \frac{T}{T-T'}t\right\} \\
		& \leq \sum_{l=0}^\infty {\Pr}^* \left[\sup_{h \in \F_l}  \frac{(\EE^*-\mathbb{P}_{\tilde{N}}) \left\{h(Z)\right\}}{\EE^* \left\{h(Z)\right\} + 2^l  \frac{T}{T-T'}t} > \frac{1}{2} \right] \\
		& \leq \sum_{l=0}^\infty 120 \exp\left\{-\oldl{case2c5} \frac{v_{\tilde{N},l }^{'2} \frac{T}{T-T'}t2^l}{\tilde{N}}\right\} + 2\beta_{x_{\tilde{N},l}}v_{\tilde{N}, l} \\
		& \leq \sum_{l=0}^\infty 120 \exp\left\{-\oldl{case2c5}\frac{v_{\tilde{N},l }^{'2} \frac{T}{T-T'}t2^l}{\tilde{N}}\right\} + 2\beta_0\exp\left(-\beta_1 x_{\tilde{N}, l} + \log v_{\tilde{N}, l}\right) ,
		\end{align*}
		where the last inequality is due to  that the Markov chain is stationary and exponentially \(\beta\)-mixing.  When $t \geq \frac{\left(4/\beta_1\log(\tilde{N})\right)^{1/\tau}}{\tilde{N}}$, we have $\log v_{\tilde{N}, l} \leq \frac{1}{2}\beta_1 x_{\tilde{N}, l}$ by using $x_{\tilde{N}, l}' \leq 2x_{\tilde{N}, l}$ and $v_{\tilde{N}, l} \leq \tilde{N}$. This will further imply that $2\beta_{x_{\tilde{N},l}}v_{\tilde{N}, l} \leq 2\beta_0\exp\left(-\beta_1 x_{\tilde{N}, l}/2 \right)$. 
		
		Then we will have 
		\begin{align*}
		&{\Pr}^*\left\{\exists B \in  \Q,  \mathrm{\Rnum{3}}(B) >  \frac{T}{T-T'}t\right\} \\
		& \leqconst \sum_{l=0}^\infty 120 \exp\left(-\oldl{case2c5}(\tilde{N} \frac{T}{T-T'}t)^{1-2\tau}(2l)^{1-2p} \right)+ 2\beta_0\exp\left(-\beta_1 (\tilde{N} \frac{T}{T-T'}t)^{\tau}(2^l)^p\right)\\
		& \leq \newl\ltxlabel{case2c6} \exp\left(-\newl\ltxlabel{case2c7}(\tilde{N} \frac{T}{T-T'}t)^{1-2\tau}\right) + \newl\ltxlabel{case2c8}\exp\left(-\newl\ltxlabel{case2c9} (\tilde{N} \frac{T}{T-T'}t)^{\tau}\right),
		\end{align*}
		where \(\oldl{case2c6}\), \(\oldl{case2c7}\), \(\oldl{case2c7}\), \(\oldl{case2c8}\) are some positive constants.
		
		For some fixed \(\delta > 0\), we can verify that there exists  a  constant \(\newl\ltxlabel{case2c10} > 0\), such that \(\oldl{case2c6} \exp(-\oldl{case2c7}(\tilde{N}t)^{1-2\tau}) + \oldl{case2c8}\exp(-\oldl{case2c9} (\tilde{N} t)^{\tau}) \leq \delta\) for  \(t \ge \oldl{case2c10} \max\{ [\log(1/\delta)\tilde{N}^{-1}]^{1/(1-2\tau)}, [\log(1/\delta) \tilde{N}^{-1}]^{1/\tau} \}\).
		
		Substituting this probability into \eqref{eqn:stationary_decompose} and combining all the conditions over \(t\), we can derive that there exists a constant \(\newl\ltxlabel{case2c11}>0\), such that 
		\begin{align*}
		&\Pr\left(\exists B,  \frac{T-T'}{T}\mathrm{\Rnum{3}}(B) \ge \oldl{case2c11} \left\{\mu + \frac{1}{\tilde{N} \mu^{\frac{\alpha}{1- \tau(2+\alpha)}} } + \frac{(\log \tilde{N})^{1/\tau} + 1 + (\log(1/\delta))^{1/(1-2\tau)} +(\log(1/\delta))^{1/\tau} }{\tilde{N}} \right\}  \right) \\
		\leq & \delta + n\kappa^{T'}\oldu{C_s}
		\end{align*}
		\begin{itemize}
			\item \textbf{For \Rnum{4}(B):}
		\end{itemize}
		
		\begin{align*}
		| \mathrm{\Rnum{4}}(B)| &\leq \EE \frac{1}{T-T'} \sum_{t=T'}^{T-1}K_1 \left\|\mathbb{G}^b_t(\cdot \mid (s_0,a_0)) - G_s(\cdot) \right\|_{\mathrm{TV}}\\
		& \leq K_1 \oldu{C_s}\frac{1}{T-T'} \kappa^{T'}/(1-\kappa).
		\end{align*}

		Recall that \(T' = \ceil{K_3 \log (nT)}\) given the condition that \(T \gg \log(nT)\). Then as long as we take \(K_3 = -3(\log kappa)^{-1} \), then for sufficiently largest \(\Total\), we have  
		\[ n\kappa^{T'}\oldu{C_s}= C_1n^{-2}T^{-3} \leq  n^{-1}T^{-1},\] and 
		\[\frac{\kappa^{T'}}{(T-T')(1-\kappa)} = \bigO\left(  \frac{1}{n^3 T^{4}(1-\kappa)}\right)  = \smallO(t),\]
		due to the conditions of \(t\).

		Combining all the results we have derived for \Rnum{2}(B), \Rnum{3}(B) and \Rnum{4}(B), we have that   with probability at least \(1 - 2\delta - 1/(\Total)\), for any \(B \in \mathcal{Q}\), 
		\begin{align*}
		I_2(B) \lesssim  \frac{\log (\Total)}{T}\left\{ \mu + (\mu^\alpha n)^{-1}\}  + \frac{\log (\Total) (1 + \log(1/\delta))}{\Total}
		\right\} \\
		+  \mu + \frac{1}{\Total \mu^{\frac{\alpha}{1- \tau(2+\alpha)}} } + \frac{(\log (\Total))^{1/\tau} + 1 + (\log(1/\delta))^{1/(1-2\tau)} +(\log(1/\delta))^{1/\tau} }{\Total} .
		\end{align*}

		Recall that we require \(0 < \tau \leq p \leq 
		1/(2+ \alpha)\), consider any \(\tau \leq 1/3\) and pick \(p\) any value between \(\tau\) and \(1/(1+ 2\alpha)\).  Then we can simplify the bound into 
		
		\begin{align*}
		I_2(B) \lesssim  \mu + \frac{1}{\Total \mu^{\frac{\alpha}{1- \tau(2+\alpha)}} }  + \frac{[\log(\max\{1/\delta, \Total\})]^{1/\tau}}{\Total}.
		\end{align*}
		
	\end{proof}

	\begin{lemma}
		\label{lem:maxeigen}
		Suppose Assumption  \ref{assum:approx_g}(a) (\ref{assum:RKHS}(a)) and  Assumption \ref{assum:weights}(c) hold. Then we have  
		\begin{multline*}
		\lambda_{\max}\left\{\frac{1}{nT}\sum_{i=1}^n\sum_{t=0}^{T-1} \bm L_K(S_{i,t}, A_{i,t}) \bm L_K(S_{i,t}, A_{i,t})^\tp   - \EE \frac{1}{T}\sum_{t=0}^{T-1} \bm L_K(S_{t}, A_{t}) \bm L_K(S_{t}, A_{t})^\tp\right\}\\
		= \bigOp\left( \sqrt{\frac{K}{\Total}}\log (\Total) \right)=\smallOp(1).\nonumber
		\end{multline*}
	\end{lemma}
	
	\begin{proof}
		Recall that \(\overline{p}_{T}^b = 1/T \sum_{t=0}^T p_t^b\).  
		By Assumption \ref{assum:weights}(c), we can show that \\
		\(\sup_{(s,a)} \lambda_{\max}\{\bm L_K(s,a) \bm L_K(s,a)^\tp\} \leq \sup_{(s,a)} \|\bm L_K(s,a)\|_2^2 \leq (\oldu{weightsC2} + \gamma \oldu{weightsC2})^2K\).
		Next, let us consider two cases. 
		\begin{itemize}
			\item \textbf{Case 1: \(T\) is fixed, i.e., \(T \asymp 1\). }
			
			Take \(\bm Z_i = \frac{1}{T} \sum_{t=0}^T \bm L_K(S_{i,t},A_{i,t}) \bm L_K(S_{i,t},A_{i,t})^\tp\). From above, we can see that \(\lambda_{\max} \left\{ \bm Z_i \right\} \leq (\oldu{weightsC2} + \gamma \oldu{weightsC2})^2 K \) for any \(i\).
			
			\begin{align}
			\lambda_{\max}\left\{ \sum_{i=1}^n \EE \bm Z_i \bm Z_i^\tp \right\}& \leq n  \lambda_{\max}\left\{  \EE \bm Z_i \bm Z_i^\tp \right\}\nonumber\\
			& = n \sup_{i}\lambda_{\max}\{\bm Z_i\} \lambda_{\max}\left\{  \EE \bm Z_i  \right\}\nonumber\\
			& \leq n (\oldu{weightsC2} + \gamma \oldu{weightsC2})^2 K \oldu{weightsC4}. \nonumber
			\end{align}
			Then, applying  Matrix Bernstein inequality \citep[][]{tropp2015introduction}, we have 
			\begin{align}
			& \pr\left(\lambda_{\max} \left\{ \frac{1}{nT}\sum_{i=1}^n\sum_{t=0}^{T-1} \bm L_K(S_{i,t}, A_{i,t}) \bm L_K(S_{i,t}, A_{i,t})^\tp  - \EE \left(\frac{1}{T}\sum_{t=0}^{T-1} \bm L_K(S_{t}, A_{t}) \bm L_K(S_{t}, A_{t})^\tp  \right)\right\} \ge x\right)\nonumber\\
			& \leq \pr\left(\lambda_{\max} \left\{ \frac{1}{n}\sum_{i=1}^n \bm Z_i - \EE \bm Z_i \right\}\ge x\right) \nonumber\\
			& \leq (K + K) \exp\left( -\frac{-nx^2/2}{(\oldu{weightsC2} + \gamma \oldu{weightsC2})^2 K \oldu{weightsC4} + (\oldu{weightsC2} + \gamma \oldu{weightsC2})^2 K x/3} \right)\nonumber
			\end{align}
			Due to the condition that \(K = \smallO(\sqrt{\Total/\log (\Total)})  = \smallO(\sqrt{n/\log n})\), we have 
			\begin{align*}
			&  \lambda_{\max}\left\{\frac{1}{nT}\sum_{i=1}^n\sum_{t=0}^{T-1} \bm L_K(S_{i,t}, A_{i,t}) \bm L_K(S_{i,t}, A_{i,t})^\tp   - \EE \frac{1}{T}\sum_{t=0}^{T-1} \bm L_K(S_{t}, A_{t}) \bm L_K(S_{t}, A_{t})^\tp\right\} \\
			=& \bigOp\left( \sqrt{\frac{K\log K}{n}} \right) = \smallOp(1).
			\end{align*}
			
			\item  \textbf{Case 2: \(T \rightarrow \infty\). }
			
			Take \(\pr^*\) as the stationary distribution and \(\EE^*\) as the expectation taken over the stationary distribution. 
			Take \(\bm Z_{i,t} = \bm L_K(S_{i,t}, A_{i,t})\bm L_K(S_{i,t}, A_{i,t})^\tp\). {From Assumption \ref{assum:weights}(c), we know that \(\lambda_{\max}\left\{ {\EE}^* \bm Z_{i,t} \right\}  \leq C_7\).}  It can be also verified that 
			\(\lambda_{\max}\{\bm Z_{i,t}\} \leq (\oldu{weightsC2} + \gamma \oldu{weightsC2})^2 K \) for any \(i,t\), and 
			\begin{align*}
			\lambda_{\max}\left\{ {\EE}^* \bm Z_{i,t} \bm Z_{i,t} \right\} \leq \sup_{i,t} \lambda_{\max} \left\{ \bm Z_{i,t} \right\}  \lambda_{\max}\left\{ {\EE}^* \bm Z_{i,t} \right\} \lesssim K.
			\end{align*}
			
			For any fixed \(i\),
			by Theorem 4.2 of \cite{chen2015optimal}, there exists some constant \(\newl\ltxlabel{bernc1} > 0\), such that for any \(\tau \ge 0\) and integer \(1< q< T\),  
			\begin{align*}
			& {\pr}^*\left(  \lambda_{\max}\left\{ \sum_{t=0}^{T-1} \left( \bm Z_{i,t} - \EE^* \bm Z_{i,t}  \right)\right\} \ge 6\tau \right) \\
			\leq & \frac{T}{q}\beta(q) + {\pr}^*\left(  \lambda_{\max}\left\{ \sum_{t\in \mathcal{I}_r} \left( \bm Z_{i,t}  - \EE^* \bm Z_{i,t}   \right)\right\} \ge 6\tau \right) \\
			+ & \oldl{bernc1} K \exp \left( -\frac{\tau^2/2}{T q K +  q (\oldu{weightsC2} + \gamma \oldu{weightsC2})^2 K \tau /3} \right),
			\end{align*}
			where \(\mathcal{I}_r = \{q\floor{(T+1)/q}{}, q\floor{(T+1)/q}{}+ 1, \dots, T-1\}\). Suppose \(\tau \ge  q (\oldu{weightsC2} + \gamma \oldu{weightsC2})^2 K\). Notice that \(|\mathcal{I}_r|\leq q\). It follows that 
			\begin{align*}
			{\pr}^*\left(  \lambda_{\max}\left\{ \sum_{t\in \mathcal{I}_r} \left( \bm Z_{i,t}  - \EE^*\bm Z_{i,t}   \right)\right\} \ge 6\tau \right) = 0.
			\end{align*}
			Since \(\beta(q) = \bigO(\kappa^q)\), set 
			{\(q = -6\log (nT)/\log \kappa\)},
			we obtain  
			{\(T\beta(q)/q = \bigO(n^{-6}T^{-5})\)}. 
			Set \(\tau = \newl \ltxlabel{bernc2} \max \{ \sqrt{T q K \log (nT)}, q (\oldu{weightsC2} + \gamma \oldu{weightsC2})^2 K \}\) with some appropriate constant \(\oldl{bernc2}>0\). From the condition that \(K = \smallO(nT)\),  the following event occurs with probability at least 
			{\(1- \bigO(n^{-6}T^{-5})\)},
			\begin{align}
			\lambda_{\max}\left\{ \sum_{t=0}^{T-1} \left(  \bm Z_{i,t}  - \EE^* \bm Z_{i,t}  \right)\right\}  \lesssim \max\{\sqrt{TK} \log (nT), K \log^2 (Tn)\}, \quad \mathrm{where} \quad \bm Z_{i,t} \sim \mathbb{G}^*. \nonumber
			\end{align}
			
			Note that
			{\begin{align}
			&\pr\left(  \lambda_{\max}\left\{ \frac{1}{T}\sum_{t=0}^{T-1} \bm Z_{i,t} - \EE^* \bm Z_{i,t} \right\} \ge x\right)= \EE \mathbbm{1}_{\lambda_{\max}\left\{ \frac{1}{T}\sum_{t=0}^{T-1} \bm Z_{i,t} - \EE^* \bm Z_{i,t} \right\} \ge x} \nonumber\\
			 = &\EE_{(S_{i,t},A_{i,t}) \sim G*} \frac{\bar p_T^b(S, A)}{p^*(S,A)}\mathbbm{1}_{\lambda_{\max}\left\{\frac{1}{T}\sum_{t=0}^{T-1} \bm Z_{i,t} - \EE^* \bm Z_{i,t} \right\} \ge x}\nonumber\\
			\leq & \sqrt{\EE^*  \left( \frac{\bar p_T^b(S, A)}{p^*(S,A)} \right)^2}\sqrt{{\pr}^*\left(  \lambda_{\max}\left\{ \sum_{t=0}^{T-1} \left( \bm Z_{i,t} - \EE^* \bm Z_{i,t}  \right)\right\} \ge x\right)}  \nonumber \\
			\leq &  \sqrt{C_7} \sqrt{{\pr}^*\left(  \lambda_{\max}\left\{ \sum_{t=0}^{T-1} \left( \bm Z_{i,t} - \EE^* \bm Z_{i,t}  \right)\right\} \ge x\right)}, \nonumber
			\end{align}}
			{where the first inequality is due to Cauchy Schwarz inequality and the second inequality is due to Assumption 5(c).}

			Then we have with probability at least {\(1- \bigO(n^{-3}T^{-5/2})\)},
			\begin{align}
			\lambda_{\max}\left\{ \sum_{t=0}^{T-1} \left(  \bm Z_{i,t}  - \EE^* \bm Z_{i,t}  \right)\right\}  \lesssim \max\{\sqrt{TK} \log (nT), K \log^2 (Tn)\}, \quad \mathrm{where} \quad \bm Z_{i,t} \sim \overline{p}_T^b. \nonumber
			\end{align}
			Then with probability at least {\(1- \bigO(n^{-2}T^{-5/2})\)}, 
			\begin{align}
			\label{eqn:uniform}
			\max_{i=1,\dots, n} \lambda_{\max}\left\{ \sum_{t=0}^{T-1} \left(\bm Z_{i,t}  - \EE^* \bm Z_{i,t}   \right)\right\} 
			\lesssim \max\{\sqrt{TK} \log (nT), K \log^2 (Tn)\}
			\end{align}
			Take 
			\begin{align*}
			\mathcal{A}_i= \left\{ \lambda_{\max}\left\{ \sum_{t=0}^{T-1} \left(\bm Z_{i,t} - \EE^* \bm Z_{i,t}   \right)\right\}  \lesssim \max\{\sqrt{TK} \log (nT), K \log^2 (Tn) \} \right\}. 
			\end{align*}
			Then based on the condition of $K$ (\(K = \smallO(nT/\log^2(nT))\)),  the matrix Bernstein inequality \cite{tropp2012user} yields that 
			\begin{align*}
			\lambda_{\max}\left\{\sum_{i=1}^n \sum_{t=0}^{T-1} \left(\bm Z_{i,t}  - \EE^* \bm Z_{i,t}   \right) \mathbbm{1}(\mathcal{A}_i) - n \EE \sum_{t=0}^{T-1} \left(\bm Z_{0,t}  - \EE^* \bm Z_{0,t}   \right) \mathbbm{1}(\mathcal{A}_0)\right\}  = \bigOp\left( \sqrt{nTK} \log (nT) \right).
			\end{align*}
			Based on the probability for \eqref{eqn:uniform}, we have 
			\begin{align*}
			\lambda_{\max}\left\{\sum_{i=1}^n \sum_{t=0}^{T-1} \left(\bm Z_{i,t}  - \EE^* \bm Z_{i,t}   \right)  - n \EE \sum_{t=0}^{T-1} \left(\bm Z_{0,t}  - \EE^* \bm Z_{0,t}   \right) \mathbbm{1}(\mathcal{A}_0)\right\}  = \bigOp\left( \sqrt{nTK} \log (nT) \right).
			\end{align*}
			Also, we can verify that for any \(\bm a \in \mathbb{R}^K\) such that \(\|\bm a\|_2 = 1\),
			\begin{align*}
			& \bm a^\tp \left\{\EE \sum_{t=0}^{T-1} \left(\bm Z_{0,t}  - \EE^* \bm Z_{0,t}   \right) \mathbbm{1}(\mathcal{A}^c_0)\right\}  \bm a \leq \sqrt{ \EE \left\{\bm a^\tp  \sum_{t=0}^{T-1} \left(\bm Z_{0,t}  - \EE^* \bm Z_{0,t}   \right) \bm a\right\}^2 } \sqrt{\pr(\mathcal{A}_0^c)}\nonumber\\
			&{ \leq \sqrt{T K} \log (nT) n^{-3/2}T^{-5/4} = \bigO(n^{-1}).}\nonumber
			\end{align*}
			
			Overall, we have 
			\begin{align*}
			\lambda_{\max} \left\{ \sum_{i=1}^n \sum_{t=0}^{T-1} \bm Z_{i,t} - n \EE \left(\sum_{t=1}^{T-1} \bm Z_{0,t} \right) \right\} = \bigOp(\sqrt{nTK} \log (nT))
			\end{align*}
			And therefore, 
			\begin{align*}
			&  \lambda_{\max}\left\{\frac{1}{nT}\sum_{i=1}^n\sum_{t=0}^{T-1} \bm L_K(S_{i,t}, A_{i,t}) \bm L_K(S_{i,t}, A_{i,t})^\tp   - \EE \frac{1}{T}\sum_{t=0}^{T-1} \bm L_K(S_{t}, A_{t}) \bm L_K(S_{t}, A_{t})^\tp\right\} \\
			&= \bigOp\left( \frac{\sqrt{K}}{\sqrt{nT}}\log (nT) \right) = \smallOp(1)
			\end{align*}
			
		\end{itemize}

	\end{proof}

	\begin{lemma}
		\label{lem:weights_concen}
		Suppose the Markov chain \(\{S_t, A_t\}_{t\ge 0 }\) satisfy Assumption \ref{assum:approx_g}(a) (\ref{assum:RKHS}(a)). Take \(\bm F(S, A) = [f_{k}(S,A)]_{k1,\dots,K} \in \mathbb{R}^K\). Suppose that \( \sup_{(s,a)} \|\bm F(s,a)\|_2 \leq R \) for some quantity \(R > 0\),  Then we have 
		\begin{align*}
		\left\| \frac{1}{nT} \sum_{i=1}^n \sum_{t=0}^{T-1} \bm F(S_{i,t}, A_{i,t}) - \EE \left\{ \frac{1}{T} \sum_{t=0}^{T-1} \bm F(S_{t}, A_{t})\right\}\right\|_2 
		=   \bigOp\left( R\sqrt{\frac{\log(K)\log(nT)}{nT}}\right) 
		\end{align*}
	\end{lemma}
	\begin{proof}
		Take \(T' = \min \{T, 2\log(1/\kappa)\log(\Total)\}\),  
		\(\bm F^{\mathrm{\Rnum{1}}}(S_i,A_i) =  \sum_{t=0}^{T'-1}\bm F(S_{i,t},A_{i,t})/T'\), \(\bm F^{\mathrm{\Rnum{2}}}(S_i,A_i) =  \sum_{t = T'}^{T-1}\bm F (S_{i,t},A_{i,t})/(T-T')\). 
		Take \(\EE^*\) and \({\pr}^*\) as the expectation and probability under the stationary distribution of \(\{S_t, A_t\}\).
		\begin{align*}
		&\left\|\frac{1}{nT} \sum_{i=1}^n \sum_{t=0}^{T-1} \bm F (S_{i,t}, A_{i,t}) - \EE \left\{ \frac{1}{T} \sum_{t=0}^{T-1} \bm F(S_{t}, A_{t})\right\}\right\|_2\\
		= & \frac{T'}{T}\left\| \frac{1}{n}\sum_{i=1}^n \bm F^{\mathrm{\Rnum{1}}}(S_i,A_i) - \EE \bm F^{\mathrm{\Rnum{1}}}(S,A) \right\|_2 + \frac{T - T'}{T}\left\| \frac{1}{n}\sum_{i=1}^n \bm F^{\mathrm{\Rnum{2}}}(S_i,A_i) - \EE^* \bm F^{\mathrm{\Rnum{2}}}(S_i,A_i)\right\|_2 \\
		& + \frac{T - T'}{T}\left\|  \EE^* \bm F^{\mathrm{\Rnum{2}}}(S,A) -  \EE \bm F^{\mathrm{\Rnum{2}}}(S,A)\right\|_2\\
		& = \mathrm{(i)} + \mathrm{(ii)} + \mathrm{(iii)}
		\end{align*}

		\begin{itemize}
			\item \textbf{For (i)},
		\end{itemize}
		We apply the Matrix Bernstein inequality \citep[][]{tropp2015introduction} to bound it.  Take \(\bm Z(S_i, A_i) =\bm F^{\mathrm{\Rnum{1}}}(S_i,A_i) - \EE \bm F^{\mathrm{\Rnum{1}}}(S_i,A_i) \).
		We can verify that 
		\begin{align*}
		&\left\|  \sum_{i=1}^n \EE \bm Z(S_i,A_i) \{\bm Z(S_i,A_i)\}^\tp\right\|_2 \leq \EE \sum_{i=1}^n  \left\|_2  \bm Z(S_i,A_i) \right\|^2_2  \\
		\leq &2 n \sup_{s,a} \| \bm F^{\mathrm{\Rnum{1}}}(s,a)\|^2_2  \leq \frac{2n}{T'}\sum_{t=0}^{T'-1}  \sup_{s,a} \| \bm F(s,a)\|^2_2   \leq 2 n R^2.
		\end{align*}
		The same bound can be derived for \(\left\|  \sum_{i=1}^n \EE  \{\bm Z(S_i,A_i)\}^\tp \bm Z(S_i,A_i)\right\|_2\).
		Then for all \(x > 0\), 
		\begin{align*}
		\pr\left( \left\| \frac{1}{n}\sum_{i=1}^n \bm F^{\mathrm{\Rnum{1}}}(S_i,A_i) - \EE \bm F^{\mathrm{\Rnum{1}}}(S,A) \right\|_2 \ge x\right) \leq K \exp \left( \frac{-n x^2/2}{2 R^2 + 2x R/3} \right).
		\end{align*}
		Given the condition that \(K = \smallO(\Total)\) and \(T' = \min\{T, 2\log(1/\kappa)\log (\Total)\}\), we have 
		\begin{align}
		(i)=\bigOp\left( \frac{T'}{T} \frac{R \sqrt{\log K }}{\sqrt{n}}\right) = \bigOp\left( \frac{R \log (\Total)}{\sqrt{\Total}} \right). \nonumber
		\end{align}
		
		If \(T  = \bigO(\log (\Total)) \), then  we do not need to analyze the remaining two components. In the following, we assume that  \(T^{-1} = \smallO(1/\log(\Total))\).
		
		\begin{itemize}
			\item \textbf{For (ii)},
		\end{itemize}
		
		\begin{align*}
		&\pr\left( \left\| \frac{1}{n}\sum_{i=1}^n \bm F^{\mathrm{\Rnum{2}}}(S_i,A_i)  - \EE^* \bm F^{\mathrm{\Rnum{2}}}(S,A) \right\|_2> x\mid (S_{i,0}, A_{i,0}), i = 1,\dots, n\right)  \\
		&\qquad -  {\pr}^*\left( \left\| \frac{1}{n}\sum_{i=1}^n \bm F^{\mathrm{\Rnum{2}}}(S_i,A_i) - \EE^* \bm F^{\mathrm{\Rnum{2}}}(S,A) \right\|_2 > x \mid (S_{i,0}, A_{i,0}), i = 1,\dots, n\right) \\
		\leq & \sum_{i=1}^n \|\mathbb{G}_{T'}^b(\cdot\mid (S_{i,0}, A_{i,0})) - \mathbb{G}^*\|_{\mathrm{TV}} \leq  \sum_{i=1}^n \phi(S_{i,0}, A_{i,0})\kappa^{T'}
		\end{align*}
		Then
		\begin{align*}
		&\pr\left( \left\|\frac{1}{n}\sum_{i=1}^n \bm F^{\mathrm{\Rnum{2}}}(S_i,A_i) - \EE^* \bm F^{\mathrm{\Rnum{2}}}(S_i,A_i)\right\|_2>x \right)  \\
		&\leq   {\pr}^*\left( \left\| \frac{1}{n}\sum_{i=1}^n \bm F^{\mathrm{\Rnum{2}}}(S,A) - \EE^* \bm F^{\mathrm{\Rnum{2}}}(S_i,A_i)\right\|_2>x\right) + n\kappa^{T'}\int_{(s,a)}\phi(s,a)d\mathbb{G}(s,a) \\
		& \leq  {\pr}^*\left( \left\| \frac{1}{n(T-T')}\sum_{i=1}^n \sum_{t=T'}^{T-1} \bm F(S_{i,t},A_{i,t}) - \EE^* \bm F(S, A)\right\|_2>x \right) + \oldu{C_s}n\kappa^{T'}, 
		\end{align*}

		Next, we apply Theorem 4.2 in \cite{chen2015optimal} to bound \[{\pr}^*\left( \left\| \frac{1}{n(T-T')}\sum_{i=1}^n \sum_{t=T'}^{T-1} \bm F(S_{i,t},A_{i,t}) - \EE^* \bm F(S, A)\right\|_2>x \right).\]
		To begin with, take \(\bm Z^*(S_{i,t}, A_{i,t}) = \bm F(S_{i,t}, A_{i,t}) - \EE^* \bm F(S,A)\). We can verify that 
		\begin{align*}
		\left\|  \sum_{i=1}^n \EE \bm Z^{s}(S_i,A_i) \{\bm Z^{s}(S_i,A_i)\}^\tp\right\|_2 \leq \EE  \left\|  \sum_{i=1}^n \bm Z^{s}(S_i,A_i) \right\|^2_2  
		\leq 2 n \sup_{s,a} \| \bm F(s,a)\|^2_2   \leq 2 n R^2.
		\end{align*}
		The same bound can be derived for \( \left\|  \sum_{i=1}^n \EE \{\bm Z^{s}(S_i,A_i)\}^\tp \bm Z^{s}(S_i,A_i) \right\|_2\).
		Then we adopt the similar arguments in the proof of Case 2 in Lemma \ref{lem:maxeigen}. 
		Take \(q = -3 \log(nT)/\log \kappa\) in Theorem 4.2 of \cite{chen2015optimal}. Then  by Corollary 4.2 in \cite{chen2015optimal}, we can verify that \(\beta(q) \Total/q = \smallO(1)\), and \(R\sqrt{q\log K } = \smallO(R \sqrt{\Total})\), and we have 
		\begin{align*}
		\left\| \frac{1}{n(T-T')}\sum_{i=1}^n \sum_{t=T'}^{T-1} \bm F(S_{i,t},A_{i,t}) - \EE^* \bm F(S, A)\right\|_2 = \bigOp\left( \frac{R \sqrt{\log (\Total) \log K } }{\sqrt{n(T-T')}} \right)
		\end{align*}

		\begin{itemize}
			\item \textbf{For (iii)},
		\end{itemize}
		\begin{align*}
		\left\| \EE^* \bm F^{\mathrm{\Rnum{2}}}_{k}(S,A) -  \EE \bm F^{\mathrm{\Rnum{2}}}_{k}(S,A) \right\|^2_2 &= \sum_{k} \left[ \frac{1}{T-T'} \sum_{t=T'}^{T-1}\left\{ \EE^* f_{k}(S, A) - \EE f_{k}(S_t, A_t)  \right\} \right]^2 \\
		& \leq \sum_{k} \frac{1}{T-T'} \sum_{t=T'}^{T-1} \left\{ \EE^* f_{k}(S, A) - \EE f_{k}(S_t, A_t)  \right\} ^2\\
		& \leq  \frac{1}{T-T'} \sum_{t=T'}^{T-1} \kappa^{2t}\int_{(s,a)} R^2  d \mathbb{G}_0(s,a)\\
		& \leq \frac{1}{T-T'} \sum_{t=T'}^{T-1}\kappa^{2t} R^2 \leq \frac{\newl\ltxlabel{concen_c2}}{T-T'} \frac{ R^2 }{n^4T^4},
		\end{align*}
		where \(\oldl{concen_c2}>0\) is a constant depending on  \(\oldu{C_s}\) and \(\kappa\).

		Combining all the bounds from (i), (ii) and (iii), given the condition that of \(a, b\) and \(R\), we can derive that 
		
		\begin{align*}
		&\left\| \frac{1}{nT} \sum_{i=1}^n \sum_{t=0}^{T-1} \bm F(S_{i,t}, A_{i,t}) - \EE \left\{ \frac{1}{T} \sum_{t=0}^{T-1} \bm F(S_{t}, A_{t})\right\}\right\|_2 \\
		=  &  \bigOp\left( \frac{R \sqrt{\log K \log (\Total)}}{\sqrt{\Total}}\right) 
		\end{align*}
	\end{proof}

	\section{Additional Proof and Lemmas}
	\label{sec:additional_proof}
	\begin{lemma}
		\label{lem:eigenvalues}
		If we take \(\bm B_K(\cdot,a)\) as either tensor-product B-spline basis or tensor-product wavelet basis for every \(a \in \mathcal{A}\). And we assume that the average visitation probability (density) \(\bar{p}_{T}^b(s,a)\) is upper bounded by a constant \(p_{\max}\) and
		lower bounded by a constant \(p_{\min}\). Then there exists constants \(c_*\) and \(C_*\) such that 
		\begin{align}
		c_* \leq \lambda_{\min}\left\{  \EE\left[ \frac{1}{T}\sum_{t=0}^{T-1} \bm B_K(S_t,A_t) \bm B_K(S_t,A_t)^\tp \right] \right\}\leq \lambda_{\max} \left\{ \EE\left[ \frac{1}{T}\sum_{t=0}^{T-1} \bm B_K(S_t,A_t) \bm B_K(S_t,A_t)^\tp \right] \right\} \leq C_*.
		\end{align}
	{	In addition, when \(T \rightarrow \infty\), we have 
		\begin{align}
			\frac{c_*}{2}\leq \lambda_{\min}\left\{  \EE_{(S,A) \sim G^*}\left[  \bm B_K(S,A) \bm B_K(S,A)^\tp \right] \right\}\leq \lambda_{\max} \left\{ \EE_{(S,A) \sim G^*}\left[  \bm B_K(S,A) \bm B_K(S,A)^\tp \right] \right\} \leq 2C_*.\nonumber
			\end{align} }
	\end{lemma}

	\begin{proof}[Proof of Lemma \ref{lem:eigenvalues}]
		Under the conditions in Lemma \ref{lem:eigenvalues}, by Lemma 2 in \cite{shi2020statistical}, for any \(a \in \mathcal{A}\), there exists positive constants \(\tilde{C}\) and \(\tilde{c}\) such that 
		\begin{align}
		\tilde{c} \leq \lambda_{\min}\left\{ \int_{s \in \mathcal{S}} \bm B_K(s,a) \bm B_K(s,a)^\tp ds \right\} \leq \lambda_{\max}\left\{ \int_{s \in \mathcal{S}} \bm B_K(s,a) \bm B_K(s,a)^\tp ds \right\} \leq \tilde{C}\label{eqn:L2_eigenvalues}.
		\end{align}
		Based on these results, we have
		\begin{align}
		&\lambda_{\max} \left\{ \EE\left[ \frac{1}{T}\sum_{t=0}^{T-1} \bm B_K(S_t,A_t) \bm B_K(S_t,A_t)^\tp \right]\right\}\nonumber\\
		= & \lambda_{\max} \left\{  \overline{\EE} \left\{ \bm B_K(S,A) \bm B_K(S,A)^\tp\right\}\right\}\nonumber\\
		\leq & p_{\max}  \max_{a \in \mathcal{A}} \left[  \lambda_{\max}\left\{ \int_{s \in \mathcal{S}} \bm B_K(s,a) \bm B_K(s,a)^\tp ds \right\} \right]\nonumber\\
		\leq & p_{\max}\tilde{C}.\nonumber
		\end{align}
		Using the similar arguments, we are able to show that 
		\begin{align}
		&\lambda_{\min} \left\{ \EE\left[ \frac{1}{T}\sum_{t=0}^{T-1} \bm B_K(S_t,A_t) \bm B_K(S_t,A_t)^\tp \right]\right\}\nonumber\\
		\ge  & p_{\min} \tilde{c}.\nonumber
		\end{align}
		Take \(c_* = p_{\min}\tilde{c}\) and \(C_* = p_{\max}\tilde{C}\), the conclusion follows.

		{When \(T \rightarrow \infty\),
			 Due to Assumption \ref{assum:RKHS}(a), 
			\begin{align}
			\left\|\overline{p}_{T}^b - \pr^*\right\|_{\mathrm{TV}} = & \int_{(s,a)}  \left\|\frac{1}{T}\sum_{t=0}^{T-1}\mathbb{G}_t^b(\cdot\mid (s,a)) - \mathbb{G}^*(\cdot)\right\|_{\mathrm{TV}} dG_0(s,a)\nonumber\\
			\leq & \int_{(s,a)}  \frac{1}{T}\sum_{t=0}^{T-1}  \phi(s,a) \kappa^t dG_0(s,a)
			\leq  \oldu{C_s} \frac{1}{T} \frac{1-\kappa^T}{1-\kappa} = \smallO(1) \nonumber.
			\end{align}
			As \(\overline{p}_T^b\) is lower bounded by \(p_{\min}\) and upper bounded by \(p_{\max}\). The density of  the stationary distribution is lower bounded by \(p_{\min}/2\) and upper bounded by \(2p_{\max}\) for sufficiently large \(T\).
			Then follow the same argument, we are able to show the bounds for $\lambda_{\max} \{ {\EE}^* \left\{ \bm B_K(S,A) \bm B_K(S,A)^\tp \right\}\}$ and $\lambda_{\min} \{ {\EE}^* \left\{ \bm B_K(S,A) \bm B_K(S,A) ^\tp \right\}\}$.}
	\end{proof}

	\begin{proof}[Proof of Theorem \ref{thm:mineign}]
	\revise{We start with showing the ``if'' direction of the statement. Suppose \(\chi \leq \newl\ltxlabel{cnt:chi1}\) for some \(\oldl{cnt:chi1} >0 \).}
	
	\revise{First, we show that following equality holds for any function \(f \in \mathbb{R}^{\mathcal{S}\times \mathcal{A}}\):
	\begin{align}
		f(s,a) = \frac{1}{1-\gamma}  \EE_{(S',A') \sim d^\pi(\cdot,\cdot \mid s, a)} [f(S',A') - \gamma (\mathcal{P}^\pi f) (S',A') \mid S=s, A=a], \quad  \forall s \in \mathcal{S}, a \in \mathcal{A}.
		\label{eqn:bellequality}
	\end{align}
	To see this, note the definition of \(d^\pi(\cdot,\cdot \mid S, A)\), and we have 
	\begin{align}
		& (1-\gamma)^{-1}\EE_{(S',A') \sim d^\pi(\cdot,\cdot \mid s, a)} [f(S',A') - \gamma (\mathcal{P}^\pi f) (S',A') \mid S=s, A=a] \nonumber \\
		= &\int_{s'\in \mathcal{S}} \sum_{a'\in \mathcal{A}} \pi(a'\mid s')\left( \sum_{t=0}^{\infty} \left\{ f(s', a') - \gamma \EE_{\tilde{S}\sim P(\cdot\mid s', a'), \tilde{a} \sim \pi(\cdot\mid \tilde{S})}\left[ f(\tilde{S},\tilde{a}) \mid s', a' \right]\right\} \gamma^t p_t^\pi(s', a' \mid S_0 = s, A_0 = a)  \right)ds'\nonumber\\
		= &\int_{s'\in \mathcal{S}} \sum_{a'\in \mathcal{A}} \pi(a'\mid s')\left( \sum_{t=0}^{\infty}  f(s', a') \gamma^t  p_t^\pi(s', a' \mid S_0 = s, A_0 = a)  \right) ds'\nonumber\\
		& \qquad \qquad \qquad  -  \int_{\tilde{s}\in \mathcal{S}} \sum_{\tilde{a}\in \mathcal{A}}  \pi(\tilde{a}\mid \tilde{s})\left( \sum_{t=1}^{\infty}  f(\tilde{s}, \tilde{a}) \gamma^t  p_t^\pi(\tilde{s}, \tilde{a} \mid S_0 = s, A_0 = a)  \right) d\tilde{s} \nonumber\\
		= &  \int_{s'\in \mathcal{S}} \sum_{a'\in \mathcal{A}} \pi(a'\mid s')\left(  f(s', a')  p_0^\pi(s', a' \mid S_0 = s, A_0 = a)  \right) ds' = f(s, a).\nonumber
	\end{align}
	Based on \eqref{eqn:bellequality}, for any \(f \in \mathbb{R}^{\mathcal{S}\times \mathcal{A}}\), we have 
	\begin{align}
		\bar{\EE} f^2(S,A) &= \frac{1}{(1-\gamma)^2} \bar{\EE}\left\{  \EE_{(S',A') \sim d^\pi(\cdot,\cdot \mid s, a)} [f(S',A') - \gamma (\mathcal{P}^\pi f) (S',A') \mid S, A] \right\}^2 \nonumber \\
		& \leq  \frac{1}{(1-\gamma)^2} \chi \bar{\EE}\left\{ (I-\gamma\mathcal{P}^\pi)f(S,A)  \right\}^2, \nonumber
	\end{align}
	where the  inequality comes from the definition of \(\chi\).
	Then we prove 
	\begin{align*}
	\Upsilon = \inf_{\bar{\EE} f^2(S,A) \ge 1}	\left[{\bar{\EE}\left\{ (I-\gamma\mathcal{P}^\pi)f(S,A)  \right\}^2}\right]
	\ge \frac{(1-\gamma)^2}{{\chi}} > \frac{(1-\gamma)^2}{{\oldl{cnt:chi1}}}> 0.
	\end{align*}
	}

	\revise{Next, we derive the ``only if'' direction of the statement. Suppose  \(\Upsilon \ge \newl\ltxlabel{cnt:chi2}\) for some  positive constant \(\oldl{cnt:chi2}\).}

\revise{Then, the null space of \(I- \gamma\mathcal{P}^\pi\) is empty. For any \(f \in \mathbb{R}^{\mathcal{S}\times \mathcal{A}}\), there exists a \(\tilde{f} \in \mathbb{R}^{\mathcal{S}\times \mathcal{A}}\), such that 
\begin{align}
	f= (I- \gamma\mathcal{P}^\pi)\tilde{f}.
	\label{eqn:existence}
\end{align}
Now we can derive 
\begin{align*}
	& \bar{\EE} \left[ \EE_{(S',A') \sim d^\pi(\cdot,\cdot \mid S, A)}\left\{ f(S',A') \mid S, A \right\} \right]^2 \\
	= &  \bar{\EE} \left[ \EE_{(S',A') \sim d^\pi(\cdot,\cdot \mid S, A)}\left\{  (I - \gamma\mathcal{P}^\pi) \tilde{f} (S', A')\mid S, A \right\} \right]^2\\
	= & (1-\gamma)^2 \bar{\EE} \left[ \tilde{f}(S, A) \right]^2\\
	\leq & (1-\gamma)^2  \left[\inf_{\bar{\EE}g^2(S,A)=1}	{\bar{\EE}\left\{ (I-\gamma\mathcal{P}^\pi)g(S,A)  \right\}^2}  \right]^{-1} \bar{\EE}\left\{ (I-\gamma\mathcal{P}^\pi)\tilde{f}(S,A)  \right\}^2\\
	\leq & (1-\gamma)^2  \oldl{cnt:chi2}^{-1} \bar{\EE} \left[ f(S,A) \right]^2,
\end{align*}
where the first equality is from \eqref{eqn:existence}, the second equality is due to \eqref{eqn:bellequality} and the last inequality is from the assumption for the minimal eigenvalue and \eqref{eqn:existence}. 
Then 
\begin{align*}
	\chi \leq (1-\gamma)^2\oldl{cnt:chi2}^{-1}.
\end{align*}
}

	\end{proof}

	\begin{proof}[Proof of Corollary \ref{cor:mineigen}]
\revise{By Jensen's inequality and the definition of \(\bar{d}^\pi\), we have 
\begin{align*}
	& \bar{\EE} \left[ \EE_{(S',A') \sim d^\pi(\cdot,\cdot \mid S, A)}\left\{ f(S',A') \mid S, A \right\} \right]^2 \\
	\leq & \bar{\EE} \left[ \EE_{(S',A') \sim d^\pi(\cdot,\cdot \mid S, A)}\left\{ f^2(S',A') \mid S, A \right\} \right]\\
	= & {\EE}_{(S,A) \sim \bar{d}^\pi} f^2(S, A).
\end{align*}
Then under the condition \eqref{eqn:denbound},  we obtain 
\begin{align*}
	\chi \leq \sup_{f\in \mathbb{R}^{\mathcal{S}\times \mathcal{A}}}\frac{{\EE}_{(S,A) \sim \bar{d}^\pi} f^2(S, A)}{\bar{\EE} f^2(S,A)} \leq \frac{\bar{\EE} \frac{\bar{d}^\pi(S,A) }{\bar{p}_T^b (S,A)} f^2(S,A)}{\bar{\EE} f^2(S,A)} \leq \oldu{cnt:denbound}\frac{\bar{\EE} f^2(S,A)}{\bar{\EE} f^2(S,A)} = \oldu{cnt:denbound}.
\end{align*}
Under  Assumption \ref{assum:mineign} (a) and (b),   we can show that 
\(\bar{d}^\pi (s,a) \leq p_{\max}\) for every \((s,a) \in \mathcal{S}\times \mathcal{A}\). 
Then we have
\begin{align*}
\EE_{(S,A)\sim \bar{d}^\pi}  \left[ f(S,A) \right]^2 = \bar{\EE}\left[  \frac{\bar{d}^\pi(S,A)}{\bar{p}_T^b(S,A)} f^2(S,A) \right] \leq \frac{p_{\max}}{p_{\min}} \bar{\EE} f^2(S,A).
\end{align*}
And 
\begin{align*}
	\chi \leq \frac{p_{\max}}{p_{\min}}.
\end{align*}
The bounds for \(\Upsilon\) can be obtained by applying Theorem \ref{thm:mineign}.}

\revise{
 For \(\psi(K)\), 
by taking \(f(S,A) = \bm B_K(S,A)^\tp \bm \alpha\) for any \(\|\bm \alpha\|_2 = 1\), we have
\begin{align*}
\bar{\EE} \left[ \bm L^\tp_K(S,A) \bm \alpha \right]^2 \ge  \frac{(1-\gamma)^2}{\oldu{cnt:denbound}
 } \bar{\EE} \left[ \bm B^\tp_K(S,A) \bm \alpha \right]^2 \ge \frac{(1-\gamma)^2}{\oldu{cnt:denbound}
 }  \lambda_{\min} \bar{\EE} \left\{ \bm B^\tp_K(S,A) \bm B_K(S,A) \right\} \|\bm \alpha\|^2_2.
\end{align*}
The conclusion follows.
}
		
	\end{proof}

\bibliographystyle{chicago} %
\bibliography{refer.bib}       %


\end{document}